\documentclass[final]{colt2025} % Anonymized submission
\title[Towards Fundamental Limits for Active Multi-distribution Learning]{ Towards Fundamental Limits for Active Multi-distribution Learning}
\usepackage{times}
\usepackage{algorithm}
\usepackage[noend]{algorithmic}
\usepackage{ulem}

\coltauthor{%
 \Name{Chicheng Zhang} \Email{chichengz@cs.arizona.edu}\\
 \addr The University of Arizona
 \AND
 \Name{Yihan Zhou} \Email{joeyzhou@cs.utexas.edu}\\
 \addr The University of Texas at Austin%
}

\newif\iffinal
\finalfalse
\iffinal
    \newcommand{\edit}[2]{}
    \newcommand{\authnote}[1]{}
    \newcommand{\chicheng}[1]{}
    \newcommand{\joey}[1]{}
\else
\newcommand{\authnote}[2]{{#1: #2}}
\newcommand{\chicheng}[1]{{\color{red}\authnote{Chicheng}{#1}}}
\newcommand{\joey}[1]{{\color{red}\authnote{Joey}{#1}}}
\newcommand{\edit}[2]{{\color{blue}{\dashuline{#1}}}{\color{red}{#2}}}
\fi

\newcommand{\blue}[1]{{\color{blue}{#1}}}

\newcommand\inner[2]{\langle #1, #2 \rangle}

\newcommand\DIS{\mathrm{DIS}}
\let\Pr\relax
\DeclareMathOperator{\Pr}{Pr} 

\usepackage{amsmath}
\usepackage{mathtools}
\usepackage{comment}
\usepackage{caption}

\newcommand{\remove}[1]{}%

\newcommand{\Var}{\mathrm{Var}}

\newcommand{\brc}[1]{\left\{ {#1} \right\}}

\newcommand{\cardin}[1]{\left\lvert {#1} \right\rvert}%

\newcommand{\norm}[1]{\| {#1} \|}%

\newcommand{\Acal}{\mathcal{A}}
\newcommand{\Bcal}{\mathcal{B}}
\newcommand{\Ccal}{\mathcal{C}}
\newcommand{\Dcal}{\mathcal{D}}

\newcommand{\Fcal}{\mathcal{F}}

\newcommand{\Hcal}{\mathcal{H}}
\newcommand{\Ical}{\mathcal{I}}

\newcommand{\Ocal}{\mathcal{O}}

\newcommand{\Qcal}{\mathcal{Q}}

\newcommand{\Scal}{{\mathcal{S}}}

\newcommand{\Vcal}{\mathcal{V}}
\newcommand{\Wcal}{\mathcal{W}}
\newcommand{\Xcal}{\mathcal{X}}
\newcommand{\Ycal}{\mathcal{Y}}

\newcommand{\rbr}[1]{\left(#1\right)}
\newcommand{\sbr}[1]{\left[#1\right]}
\newcommand{\cbr}[1]{\left\{#1\right\}}

\newcommand{\bra}[1]{\left(#1\right)}

\newcommand{\abs}[1]{\left|#1\right|}

\newcommand{\E}{\mathbb{E}}
\newcommand{\de}{\mathrm{d}}
\newcommand{\PP}{\mathbb{P}}
\newcommand{\squr}[1]{\left[#1\right]}

\DeclareMathOperator{\polylog}{polylog}

\DeclareMathOperator{\final}{final}

\DeclareMathOperator{\dist}{dist}

\DeclareMathOperator{\KL}{KL}

\DeclareMathOperator{\Unif}{Unif}
\DeclareMathOperator{\AGE}{AGR}
\DeclareMathOperator{\AGR}{AGR}

\newcommand{\EX}{\textsc{EX}}
\newcommand{\PMDL}{\textsc{Passive-MDL}\xspace}
\newcommand{\PRMDL}{\textsc{Passive-RPU-MDL}\xspace}
\newcommand{\PRSDL}{\textsc{Robust-RPU-Learn}\xspace}

\newcommand{\eps}{\varepsilon}
\newcommand{\DKL}{D_{\KL}}

\DeclareMathOperator{\argmin}{arg\,min}
\DeclareMathOperator{\argmax}{arg\,max}
\DeclareMathOperator{\VC}{VC-dim}
\DeclareMathOperator{\sign}{sign}

\newcommand{\s}{\mathfrak{s}}

\newcommand{\N}{\mathbb{N}}
\newcommand{\B}{\mathbf{B}}

\renewcommand{\emph}[1]{\textit{#1}}

\newtheorem{assumption}{Assumption}
\newtheorem{claim}{Claim}

\begin{document}

\maketitle

% \chicheng{
% Near-optimal active multi-distribution learning (may be too overclaiming)

% Active multi-distribution learning: near-optimal in realizable and improved in agnostic

% Active multi-distribution learning: towards fundamental statistical limits

% Towards fundamental statistical limits of active multi-distribution learning

% Novel upper and lower bounds for active multi-distribution learning: towards fundamental limits
% }

\begin{abstract}%
Multi-distribution learning extends agnostic Probably Approximately Correct (PAC) learning to the setting in which a family of $k$ distributions, $\cbr{D_i}_{i\in[k]}$, is considered and a classifier's performance is measured by its error under the worst distribution. This problem has attracted a lot of recent interests due to its  applications in collaborative learning, fairness, and robustness. Despite a rather complete picture of sample complexity of passive multi-distribution learning, research on active multi-distribution learning remains scarce, 
with algorithms whose optimality remaining unknown.

In this paper, we develop new algorithms for active multi-distribution learning and establish improved label complexity upper and lower bounds, in distribution-dependent and distribution-free settings. Specifically, in the near-realizable setting we prove an upper bound of 
$\widetilde{O}\Bigl(\theta_{\max}(d+k)\ln\frac{1}{\varepsilon}\Bigr)$
and $\widetilde{O}\Bigl(\theta_{\max}(d+k)\Bigl(\ln\frac{1}{\varepsilon}+\frac{\nu^2}{\varepsilon^2}\Bigr)+\frac{k\nu}{\varepsilon^2}\Bigr)$ in the realizable and agnostic settings respectively,
where \(\theta_{\max}\) is the maximum disagreement coefficient among the \(k\) distributions, \(d\) is the VC dimension of the hypothesis class, \(\nu\) is the multi-distribution error of the best hypothesis, and \(\varepsilon\) is the target excess error. Moreover, we show that the bound in the realizable setting is information-theoretically optimal and that the \(k\nu/\varepsilon^2\) term in the agnostic setting is fundamental for proper learners. We also establish instance-dependent sample complexity bound for passive multidistribution learning that smoothly interpolates between realizable and agnostic regimes~\citep{blum2017collaborative,zhang2024optimal}, which may be of independent interest. 
\end{abstract}

%We also derive novel distribution-free label complexity upper bounds. 

%Along the way, we also obtain an instance-dependent sample complexity upper bound for passive multi-distribution learning, which can be of independent interest. 
%In addition, we derive novel distribution-free label complexity upper bounds.

%%Additionally, in the near-realizable case ($\nu \leq O(\varepsilon)$), our upper bound can be sharpened to , which we show is information-theoretically optimal. 

%a matching lower bound and

%Notably, our result improves the previous multiplicative dependence on $kd$ to an additive dependence. 

%\chicheng{1. let's define the symbols if we use them. 2. perhaps distribution-free and distribution-dependent result can be mentioned in parallel.}


\section{Introduction}
\label{sec:intro}

Multi-Distribution Learning (MDL)~\citep{blum2017collaborative,haghtalab2022demand} is an emerging machine learning paradigm that has gained popularity in recent years. It naturally generalizes the classical  PAC~\citep{valiant1984theory,kearns1992toward} learning framework. In traditional PAC learning, the objective is to approximately identify the optimal hypothesis $h^*$ from a hypothesis class $\Hcal$ within error tolerance $\varepsilon$, with probability at least $1-\delta$, under a single unknown distribution $D$. MDL extends this framework by considering $k$ unknown distributions $\{D_i\}_{i \in [k]}$ and evaluating the performance of a hypothesis $h$ based on its worst-case error across these $k$ distributions, $\max_{i \in [k]} L(h, D_i)$. 
The goal is to output a classifier $\hat{h}$, such that $\max_{i \in [k]} L(\hat{h}, D_i) \leq \nu + \varepsilon$, where $\nu = \min_{h \in \Hcal} \max_{i \in [k]} L(h, D_i)$ is the optimal worst-case error in class $\Hcal$.
This setting has found diverse applications, including collaborative and federated learning~\citep{blum2017collaborative,mohri2019agnostic}, fairness~\citep{rothblum2021multi,du2021fairness}, and robustness~\citep{wang2023distributionally,deng2020distributionally}, highlighting its significance in addressing complex learning scenarios.

%\chicheng{Is this example compatible with MDL setting? Maybe yes, in that we'd like to ensure error under multiple subpopulations are all small?}

In certain real-world applications, such as cancer detection~\citep{gal2017deep}, unlabeled data is significantly more abundant and less costly to obtain than labeled data. Consequently, a more practical objective in these scenarios is to minimize the number of labels required to achieve the PAC multi-distribution learning goal, a concept known as \emph{label complexity}. This learning paradigm, termed active learning, has been extensively studied over the past three decades~\citep{cohn1994improving,dasgupta2005coarse,hanneke2014theory}. An active learner can access an unlimited number of unlabeled data points and selectively query labels for certain instances, whereas a passive learner relies solely on randomly sampled feature-label pairs from the underlying distribution. \citet{cohn1994improving,freund1997selective,dasgupta2004analysis} demonstrated that, in the realizable setting, active learning algorithms can achieve exponentially lower label complexity than passive learners when learning geometric concepts such as 1-dimensional threshold functions and $d$-dimensional linear separators. 
%\chicheng{Actually, I am not sure if this is the first paper that shows this.. For PAC result, the first work I am aware of is~\cite{freund1997selective} (but only in Bayesian setting) and~\cite{dasgupta2004analysis}. 
Subsequently, many follow-up works \citep[e.g.][]{balcan2006agnostic,hanneke2007bound,dasgupta2007general} showed that such exponential improvements in label complexity are also attainable in the agnostic setting. Notably, recent work of \citet{rittler2024agnostic} studied active learning in the MDL setting and provided algorithms and upper bounds on label complexity, demonstrating that an improvement over its passive counterpart is possible (see below for more details).

\begin{table}[H]
\centering
\begin{tabular}{|c|c|c|}
\hline
& Passive & Active \\
\hline
PAC ($k=1$) & $\widetilde{O}\bra{\frac{d (\nu + \eps)}{\varepsilon^2}}$ & $\widetilde{O}\bra{d\theta\bra{\ln \frac{1}{\varepsilon}+\frac{\nu^2}{\varepsilon^2}}}$ \\
\hline
MDL & $\widetilde{O}\bra{\frac{k+d}{\varepsilon^2}}$ & $\widetilde{O}\bra{\frac{\nu^2}{\varepsilon^2}kd\theta_{\max}^2+\frac{k}{\varepsilon^2}}$\\
\hline
\end{tabular}
\caption{Label complexity upper bounds for different settings prior to our work}
\label{table:upperboundcomparison}
\end{table}

% \chicheng{Since we mention $\nu$, shall we actually mention that the refined upper bound of $\frac{d (\nu + \varepsilon)}{\varepsilon^2}$?}
%use it to construct surrogates of $L(h, D_i)$ and
%to run the classic PAC learning algorithm independently for each of the $k$ distributions \chicheng{Need to aggregate the $k$ outputs somehow? Or actually, we just sample $S_i$ of size $\frac{d}{\varepsilon^2}$ from $D_i$, and return $\hat{h} = \argmin_{h \in \Hcal} \max_{i} L(h, S_i)$?
%}
%, and $\nu$ represents the error of the best hypothesis

In the classical PAC ($k=1$) setting, it is well known that the sample complexity upper bound in the agnostic passive setting is $\widetilde{O}\bra{\frac{d (\nu + \varepsilon)}{\varepsilon^2}}$ (where $\widetilde{O}$ hides logarithmic dependencies on problem parameters)~\citep{vapnik1982estimation}, where $d$ is the Vapnik-Chervonenkis (VC) dimension of $\Hcal$. A straightforward approach for the MDL setting is to sample $S_i$, a set of  $\widetilde{O}\bra{\frac{d (\nu + \varepsilon)}{\varepsilon^2}}$ iid examples from each $D_i$ and find $\hat{h}$ that minimizes the worst-case empirical error $\hat{h} = \argmin_{h \in \Hcal} \max_{i} L(h, S_i)$. This approach yields a sample complexity upper bound of $\widetilde{O}\bra{\frac{kd (\nu + \varepsilon)}{\varepsilon^2}}$. Remarkably, a series of works have developed more sample-efficient algorithms and analyses, improving the sample requirement to $\widetilde{O}\bra{\frac{k+d}{\varepsilon}}$ or  $\widetilde{O}\bra{\frac{k+d}{\varepsilon^2}}$ in realizable and agnostic settings, respectively~\citep{blum2017collaborative,nguyen2018improved,chen2018tight,zhang2024optimal,peng2024sample}. This shift from a multiplicative to an additive dependence on $k$ and $d$ is a significant improvement, requiring substantial technical innovation. On the other hand, one of the state-of-the-art label complexity upper bounds for agnostic single-distribution active learning is $\widetilde{O}\bra{d\theta\bra{\ln\frac{1}{\varepsilon}+\frac{\nu^2}{\varepsilon^2}}}$~\citep{dasgupta2007general,hanneke2007bound,ailon2012active}, where $\theta$ is a distribution-dependent parameter known as the \textit{disagreement coefficient}~\citep{hanneke2007bound}. 
The disagreement coefficient can be shown to be small under many favorable assumptions of the data distributions, leading to improved label efficiency of  active learning than passive learning; see e.g.~\cite{hanneke2014theory,friedman2009active,wang2011smoothness} for example distributions for which the disagreement coefficients are small.

% for the active MDL problem

For active MDL problems, \citet{rittler2024agnostic} established label complexity upper bounds of
$\widetilde{O}\bra{ kd\theta_{\max} \ln\frac1\varepsilon}$
and 
$\widetilde{O}\bra{\frac{\nu^2}{\varepsilon^2}kd\theta_{\max}^2+\frac{k}{\varepsilon^2}}$ in near-realizable and nonrealizable settings respectively,
where $\theta_{\max}$ is the maximum disagreement coefficient among all $k$ distributions. While this result provides a novel upper bound for active MDL, the dependence on $kd$ and the quadratic factor $\theta_{\max}^2$ may be undesirable. It is worth noticing that in some cases where the disagreement coefficient is large, the label complexity of the active MDL algorithm could be worse than its passive counterpart. In addition, distinct from single-distribution active learning label complexity results, an extra  
$O\rbr{ \frac{k}{\varepsilon^2} }$ additive factor appears in the agnostic label complexity bound, and it is unclear if is information-theoretically necessary. A comparison of the best known label complexity upper bounds across different settings is presented in Table~\ref{table:upperboundcomparison}. Given these, we ask the following questions:
\begin{center}
    \textit{In active MDL, is it possible to design algorithms with label complexity bounds that are never worse than passive learners? Is it possible to improve the multiplicative $kd$ factor in the label complexity bounds to an additive $k+d$?
    }
\end{center}

%\joey{revisit}

\begin{comment}
\chicheng{One drawback: in the worst case when $\theta(r) = \frac{1}{r}$, their bound can be worse than SOTA passive MDL sample complexity. -- Can we design algs that perform no worse than passive MDL?}
\chicheng{Another angle: when $k=1$, their agnostic upper bound is much higher than SoTA single-distribution DBAL algorithms. -- Can we design active MDL algorithms that perform no worse than single-distn DBAL for small $k$ (e.g. $k=2$)? (Negative answer)
}
in near-realizable and nonrealizable settings respectively,
where $\theta_{\max}$ is the maximum disagreement coefficient among all $k$ distributions. While this result provides a novel upper bound for active MDL, the dependence on $kd$ and the quadratic factor $\theta_{\max}^2$ may be undesirable; in addition, distinct from single-distribution active learning label complexity results, an extra  
$O\rbr{ \frac{k}{\varepsilon^2} }$ additive factor appears in the agnostic label complexity bound, and it is unclear if is information-theoretically necessary. A comparison of the best known label complexity upper bounds across different settings is presented in Table~\ref{table:upperboundcomparison}. Thus, we ask the following question:
\begin{center}
    \textit{In active MDL, is it possible to design algorithms that obtain a label complexity upper bound that depends additively on $k$ and $d$? Is a $\Omega\rbr{ \frac{k}{\varepsilon^2}}$ label complexity necessary?
    }
\end{center}
\end{comment}

\subsection{Our Contributions}
In this paper, we establish novel and tighter label complexity bounds for active MDL, providing an affirmative answer to those open questions. Our main contributions are as follows:
\begin{enumerate}
\item For the large $\varepsilon$ regime ($\varepsilon\ge100\nu$), we develop an algorithm with a distribution-dependent sample complexity bound of $\widetilde{O}\bra{\theta_{\max}\bra{d+k}\ln\frac{1}{\varepsilon}}$, matching the bound shown in Table~\ref{table:upperboundcomparison} for the PAC setting (Section~\ref{sec:large-eps}). We prove that this bound is information-theoretically optimal by establishing a matching lower bound (Section~\ref{sec:lower-bound}).
\item For the small $\varepsilon$ regime ($\varepsilon<100\nu$), we propose a two-stage algorithm that achieves a distribution-dependent label complexity of $\widetilde{O}\bra{\theta_{\max}(d+k) \bra{\ln \frac{1}{\varepsilon} + \frac{\nu^2}{\varepsilon^2}} + \frac{k \nu}{\varepsilon^2}}$ (Section~\ref{sec:small-eps}), accompanied by a lower bound of $\Omega(k \frac{\nu}{\eps^2})$ for proper learners (Section~\ref{sec:lower-bound}).
This lower bound reveals an interesting ``phase transition'' behavior in label complexity unique to active multi-distribution learning, since it does not appear in the large $\eps$ regime.
As part of our analysis, we strengthen the existing passive MDL sample complexity bound to $\widetilde{O}\bra{\frac{(k+d) (\nu + \eps)}{\varepsilon^2}}$, which 
smoothly interpolates between realizable and agnostic regimes~\citep{blum2017collaborative,zhang2024optimal}
—a result that may be of independent interest beyond active learning.

%%Interestingly, 
%our lower bound $\Omega(k \frac{\nu}{\varepsilon^2})$ 
%the passive setting

\item In the large $\varepsilon$ regime ($\varepsilon\ge 100 \nu$), we establish a distribution-free upper bound of $\widetilde{O}\bra{\mathfrak{s}(d+k) \ln \frac1\varepsilon}$, where $\mathfrak{s}$ denotes the star number of the hypothesis class $\Hcal$~\citep{hanneke2015minimax}. We further show that when $\varepsilon\ge 100 \bra{d+k}\nu$, this bound tightens to $\widetilde{O}\bra{\mathfrak{s}\ln \frac{1}{\varepsilon}}$—a novel result even in the context of single-distribution PAC active learning (Section~\ref{sec:distn-free}).
\end{enumerate}

Table~\ref{table:improvedupperbound} summarizes our new label complexity upper bounds.

\begin{table}[h]
    \centering        
    \begin{tabular}{|c|c|c|}
        \hline
        & $\varepsilon\ge100\nu$ & $\varepsilon<100\nu$ \\ 
        \hline
        Distribution-dependent & $\widetilde{O}\bra{\theta_{\max}\bra{d+k}\ln \frac{1}{\varepsilon}}$ & $\widetilde{O}\bra{\theta_{\max}(d+k) (\ln \frac{1}{\varepsilon} + \frac{\nu^2}{\varepsilon^2}) + \frac{k \nu}{\varepsilon^2}}$ \\ \hline
         & $\varepsilon\ge 100 \nu$ & $\varepsilon \ge 100 (d+k)\nu$ \\
         \hline & &\\
         [-1em]
         Distribution-free & $\widetilde{O}\bra{\mathfrak{s}(d+k)}$ & $\widetilde{O}\bra{\mathfrak{s}\ln \frac{1}{\varepsilon}}$\\
         \hline
    \end{tabular}
    \caption{Summary of our improved label complexity upper bounds.
    % \chicheng{The lower right corner seems incorrect?
    % It is possible to split the lower left corner to two parts, one $\eps \geq \nu$, the other $\eps \geq (d+k) \nu$?
    % }
    }
    \label{table:improvedupperbound}
\end{table}

 % The remainder of the paper is organized as follows: Section~\ref{sec:relwork} reviews related work in active learning and MDL. Section 3 introduces formal problem definitions and notation. Section 4 presents our algorithms and analysis in the large $\varepsilon$ regime, for achieving distribution-dependent guarantees. Section 5 details our algorithms and analysis for the small $\varepsilon$ regime. Section 6 establishes the lower bounds. Section 7 develops algorithms with distribution-free label complexity guarantees. 

\begin{comment}
Would you like me to explain any specific changes I made or focus on particular aspects for further refinement?

\begin{table}[H]
\centering
\begin{tabular}{ccc}
& Upper bound & Lower Bound \\
\hline
Distribution-dependent & $\theta(d+k) \ln\frac1\eps $ & $\theta(d+k)$ \\
Distribution-free & $s \ln\frac{k}{\eps}$ & $s$
\end{tabular}
\caption{Our results in the realizable setting.}
\label{tab:realizable}
\end{table}

\begin{table}[H]
\centering
\begin{tabular}{ccc}
& Upper bound & Lower Bound \\
\hline
Distribution-dependent & $\theta(\nu + \varepsilon) (d+k) (1 + \frac{\nu^2}{\eps^2}) + \frac{k \nu}{\eps^2} $ & $\theta(\nu + \varepsilon) (d+k)+ \frac{d \nu^2}{\eps^2} + k\frac{\nu}{\eps^2}$ \\
Distribution-free & $s (d+k) (1 + \frac{\nu^2}{\eps^2}) + \frac{k \nu}{\eps^2}$ & $s + d \frac{\nu^2}{\eps^2} + k\frac{\nu}{\eps^2}$
\end{tabular}
\caption{Our results in the agnostic setting.}
\label{tab:agnostic}
\end{table}
\end{comment}

\section{Related Work}
\label{sec:relwork}

%\chicheng{I think the dependence on $\eps$ is only $1/\eps$?}

\paragraph{Multi-Distribution Learning} \citet{blum2017collaborative} first established the sample complexity upper bound of $\widetilde{O}\bra{\frac{d+k}{\varepsilon}}$ 
for MDL in the realizable and large $\varepsilon$ regime. Their analysis was later refined to remove extra logarithmic factors by \citet{nguyen2018improved,chen2018tight}. Their method iteratively learns the best hypothesis under the average of a carefully-maintained subset of distributions. If a hypothesis performs well on the average, it must also perform well on at least a constant fraction of the distributions, allowing the elimination of those where good performance has already been achieved. By the end of the process, a well-performing hypothesis is obtained for each distribution. For the nonrealizable setting,  \citet{haghtalab2022demand} introduced a new approach that reduces the problem to solving a two-player zero-sum game, where one player learns the best hypothesis while an adversary simultaneously selects the hardest distribution. However, their algorithm is only applicable when the hypothesis class is finite. \citet{awasthi2023open} modified the game dynamics algorithm of~\citet{haghtalab2022demand} to extend it to the infinite hypothesis class setting; however, their algorithm suffers from an undesirable \(O\bra{\frac{1}{\varepsilon^4}}\) multiplicative factor. 
\cite{hanashiro2023distribution} designed algorithms inspired by best-arm identification in bandit literature that achieve instance-dependent rates; such rates depend on suboptimality gaps of hypotheses which may be arbitrarily small. 
The optimal sample complexity bound in the infinite hypothesis class case was later established independently by \citet{zhang2024optimal} and \citet{peng2024sample}. While \citet{zhang2024optimal} employed the same game dynamics algorithm with a more careful control of the Hedge algorithm's trajectory and a delicate analysis, \citet{peng2024sample} solved the problem using recursive width reduction. 
Our algorithms for the large $\eps$ regime incorporates \citet{blum2017collaborative}'s algorithm as a subroutine, while our distribution-dependent algorithm for the  small $\eps$ regime builds on \citet{zhang2024optimal} as a subroutine. 
%\chicheng{Not exactly? The realizable distribution-dependent result also uses \citet{blum2017collaborative} as a subroutine}

%This framework extended the sample complexity bound to the small \(\varepsilon\) regime. 
%\chicheng{Perhaps we should also discuss~\cite{awasthi2023open}}
% \chicheng{like I mentioned before, I don't think the CAL paper provided an analysis. 
% \href{https://citeseerx.ist.psu.edu/document?repid=rep1&type=pdf&doi=db4e431f42bb9cf3556459acbc7e94d818331f90}{Steve Hanneke's thesis} Section 1.4 also mentioned this I think..
% }

\paragraph{Active Learning} \citet{cohn1994improving} designed  
the first active learning algorithm in the realizable setting, and was first analyzed by~\citet{hanneke2007bound}. In a phased variant of the algorithm~\citep[][Chapter 2]{hsu2010algorithms}, the algorithm begins with a constant error tolerance and learns a good hypothesis from a passive learner. It then eliminates all apparently bad hypotheses and reduces the error tolerance by half in each iteration. Since the algorithm only queries labels from the disagreement region—where hypotheses in the class disagree on labels—the disagreement region shrinks as more hypotheses are eliminated, leading to significant label savings. \citet{balcan2006agnostic} extended this idea to the agnostic setting, and \citet{hanneke2007bound,dasgupta2007general} 
further refined the analysis to obtain a tighter sample complexity bound $\widetilde{O}\bra{d\theta\bra{\ln \frac{1}{\varepsilon}+\frac{\nu^2}{\varepsilon^2}}}$. Our work builds upon this algorithmic framework but replaces the passive learner with a passive MDL learner in each round. Beyond disagreement region-based methods, alternative approaches for active learning exist. For instance, \citet{freund1997selective,dasgupta2004analysis,dasgupta2005coarse,castro2008minimax,nowak2011geometry,tosh2017diameter,zhou2024competitive} proposes methods that generalize binary search to query most informative data point at each iteration, achieving  competitive label complexity bounds relative to the instance-optimal solution;~\cite{balcan2007margin,balcan2013active,zhang2014beyond,huang2015efficient} refines the disagreement-based active learning idea by learning from carefully constructed subsets of the disagreement regions; under some regression realizability assumptions, ~\cite{cesa2009robust,dekel2012selective,agarwal2013selective,krishnamurthy2019active,zhu2022efficient,sekhari2024selective} uses rigorous uncertainty quantification on 
examples' Bayes optimal labels to guide label queries. 

%on the Bayes optimal classifier and query examples' whose Bayes optimal labels, and query those that are uncertainlabel uncertain.

%%\chicheng{I think this was first shown in~\cite{dasgupta2007general}; also a slightly suboptimal bound was shown in~\cite{hanneke2007bound}
%}

% \chicheng{  
% Perhaps we could expand this more. Works worth noting I think are. 
% }

%proposed an algorithm that queries the 

\begin{comment}
\section{Preliminaries}

\chicheng{
Calligraphical $\Xcal, \Ycal, \Hcal$
$\ln$ is log base $e$
and $\log$ is log base 2
$L \to \ell$ 
$\PP \to \Pr$
}
\end{comment}

%where $\Dcal$ is the set of all distributions
%\Dcal =

\section{Problem Definition and Notations}
\label{sec:prelims}

Let \(\Xcal\) and \(\Ycal\) denote the feature and label spaces, respectively, with \(\Ycal=\{-1,1\}\) (binary classification). Let \(\Hcal\) be a hypothesis class on \(\Xcal \to \Ycal \), and for any \(h\in\Hcal\), define its 0-1 loss on an example \((x,y)\in\Xcal\times\Ycal\) as \(\ell(h,(x,y))=I(h(x)\neq y)\), where \(I(A)\) is the indicator of event \(A\). We use $\ln$ to denote natural logarithm and $\log$ to denote logarithm with base 2. A multi-distribution learning instance $\Dcal$ is a collection of 
\(k\) distributions \( (D_1,\dots,D_k) \) over \(\Xcal\times\Ycal\). The error of \(h\) on a specific distribution \(D\) is given by \(L(h,D)=\E_{(x,y)\sim D}\ell(h,(x,y))\); when the context is clear, we write \(L_i(h)\) for \(L(h,D_i)\) and define $h$'s   (worst-case) \emph{multi-distribution error} as \(L_{\Dcal}(h)=\max_{i \in [k]} L(h,D_i)\). When the context is clear, we drop the subscript and abbreviate $L_{\Dcal}(h)$ as $L(h)$. For any distribution \(w\) on \([k]\), define the mixture distribution \(D_w=\sum_{i\in[k]}w(i)D_i\) with corresponding error \(L(h,w)=\E_{(x,y)\sim D_w}\ell(h,(x,y))\). For any subset of hypotheses \(V\subseteq\Hcal\), its disagreement region is defined as \(\DIS(V)\coloneqq\{x\in\Xcal: \exists\,h_1,h_2\in V \text{ with } h_1(x)\neq h_2(x)\}\), and its complement (the agreement region) is \(\AGE(V) : =\Xcal\setminus\DIS(V)\).
Given $V$ and an example $x \in \AGE(V)$, we slightly abuse notation and denote by $V(x)$ the prediction of any classifier $h \in V$ on $x$, which is independent of the choice of $h$.
Given a distribution \(D\), the disagreement metric between hypotheses \(h\) and \(h'\) is \(\rho_D(h,h')=\Pr_{x\sim D}\sbr{h(x)\neq h'(x)}\), and the disagreement ball centered at \(h\) with radius \(r\) is \(\B_D(h,r)=\{h'\in\Hcal: \rho_D(h,h')\le r\}\). 
We abbreviate $\rho_i(h,h') := \rho_{D_i}(h,h')$, and define $\rho(h,h') := \max_{i \in [k]} \rho_i(h,h')$; it can be seen that $\rho$ is also a metric.
It is well-known that triangle and reverse triangle inequalities hold, i.e., for any $h, h'$ and $i \in [k]$,
\[
\abs{ L_i(h) - L_i(h') } \leq \rho_i(h,h')
\leq 
L_i(h) + L_i(h')
,
\quad
\abs{ L(h) - L(h') } \leq \rho(h,h') 
\leq 
L(h) + L(h').
\]
%\joey{I think it's usually called "reverse triangle inequality".}
%\chicheng{Ok, let's use that phrase then.}
The disagreement coefficient of a reference hypothesis \(h^*\) (with respect to \(\Hcal\) and \(D\))~\citep{hanneke2007bound} is given by
\[\theta_{D,\Hcal,h^*}(r_0)=\sup_{r\ge r_0}\frac{\Pr_{x\sim D}\sbr{x\in\DIS(\B_D(h^*,r))}}{r}
\]
(abbreviated as \(\theta_D\), or \(\theta_i\) when \(D=D_i\)). 
It is well-known that $\theta_D(r_0) \leq \frac{1}{r_0}$ and can be much smaller~\citep{hanneke2007bound,friedman2009active,wang2011smoothness}.
We use $\theta_{\max}(r_0)\coloneqq\max_{i\in[k]}\theta_i(r_0)$ to denote the maximum $\theta_i$ among all distributions. The star number of \(\Hcal\) with respect to a reference hypothesis \(h^*\)~\citep{hanneke2015minimax} is defined as the largest integer \(\mathfrak{s}\) for which there exist points \(x_1,\dots,x_{\mathfrak{s}}\in\Xcal\) and hypotheses \(h_1,\dots,h_{\mathfrak{s}}\in\Hcal\) satisfying \(h_i(x_i)\neq h^*(x_i)\) and \(h_i(x_j)=h^*(x_j)\) for all \(j\neq i\); \(\mathfrak{s}\) is defined as \(\infty\) if the size of such point set can be arbitrarily large.

%Denote by \(D_{i,\Xcal}\) the marginal over \(\Xcal\).

We consider active multi-distribution learning in the PAC setting~\citep{rittler2024agnostic,hanneke2014theory,haghtalab2022demand}. For each distribution \(D_i\), the learner has access to two oracles which it can query interactively: an example oracle \(\textsc{EX}_i\) that independently draws an unlabeled example from $D_i$'s marginal distribution over $\Xcal$ \(D_{i,\Xcal}\), and a labeling  oracle \(\Ocal_i\) that returns a label \(y\sim \Pr_{D_i}(y\mid x)\) for any queried example \(x\). The goal is to output a hypothesis \(\hat{h}\) such that $L\bra{\hat{h}}\le \nu+\varepsilon$, with probability at least \(1-\delta\), using as few label queries as possible, where \( (h^*, \nu) = 
 (\argmin_{h \in \Hcal} L(h), \min_{h\in\Hcal} L(h) )  \) is the best hypothesis in $\Hcal$ and its multi-distribution error, respectively.
The total number of queries made by the learner to any of the $\Ocal_i$, as a function of $\eps$ and $\delta$, is called its \emph{label complexity}.

\section{Large \texorpdfstring{$\varepsilon$}{} Regime}
\label{sec:large-eps}

We start with studying the setting that there is small amount of label noise in the MDL instance $\Dcal$, specifically, 
the target excess error $\eps$ is much larger than the optimal multi-distribution error $\nu$: 
\begin{assumption}
\label{assn:large-eps}
The target error $\varepsilon \geq 100 \nu$, where we recall that $\nu = \min_{h \in \Hcal} L(h)$.  
\end{assumption}
%\max_{i \in [k]} L(h, D_i)

%\subsection{Distribution-dependent analysis}

This setting captures the realizable setting where there exists some $h^* \in \Hcal$ such that $L_i(h) = 0$ for all $i \in [k]$, and is more general in that it allows $\nu$ to be a constant factor smaller than $\eps$. In this setting, state-of-the-art algorithms~\cite[][Algorithm 2 and Theorem 3]{rittler2024agnostic} has a label complexity of $O\rbr{ d k \theta_{\max}(\eps) \ln\frac1\varepsilon }$; albeit novel, it is not clear if this bound is the best we can hope for. For example, in the realizable case where $\nu = 0$, when $\theta_{\max}(\eps) = O\rbr{ \frac 1 {\eps} }$, the bound translates to $\widetilde{O}\rbr{ \frac{d k}{\eps} }$, which is worse than the state-of-the-art passive MDL sample complexity $\widetilde{O}( \frac{d+k}{\eps} )$~\citep{blum2017collaborative}.
This motivates our question: can we design an active MDL algorithm with label complexity no worse than their passive counterparts?

%,nguyen2018improved,chen2018tight

To tackle this question, an naive approach is to perform single-distribution active learning on the average distribution $\overline{D} = \frac1k \sum_{i=1}^k D_i$ with target excess error $\frac{\eps}{k}$. 
By standard guarantees in single-distribution active learning, this yields an algorithm that works in the realizable setting,  with a label complexity of 
$O( \theta_{\overline{D}}(\frac{\eps}{k}) \cdot d \cdot \ln\frac k \eps )$.
While simple and general, we show that its $\theta_{\overline{D}}(\frac{\eps}{k})$ dependence can have a hidden $k$ factor in general. 

%can be a factor of $k$ larger than $\theta_{\max}(\varepsilon)$

\begin{proposition}
\label{prop:avg-fail}
For any $k \in \N$ and $\varepsilon > 0$, there exist a MDL instance  $\Dcal = ( D_i )_{i=1}^k$ and a hypothesis class $\Hcal$ 
such that 
$\theta_{\overline{D}}(\frac{\eps}{k}) \geq k \theta_{\max}(\eps)$.
\end{proposition}

The proof of Proposition~\ref{prop:avg-fail} can be found in Appendix~\ref{sec:def-large-eps}. Importantly, it shows that the above naive approach still leads to an undesirable $O(kd)$ dependence in label complexity, in the worst case.

%Specifically, 

%it is unclear how 
%While simple and general, this strategy fails to achieve a good label complexity even in the realizable setting.  

%\joey{Why is this called proposition? Why not lemma?}
%\chicheng{I prefer to call it a proposition, because this is a side result; not a lemma towards proving some main theorem.}

To bypass this $O(kd)$ barrier, we next propose an algorithm, Algorithm~\ref{alg:main-large-eps}, with a label complexity of $O\rbr{ (k+d) \theta_{\max}(\eps) \ln\frac1\eps }$. In sharp contrast with the previous approach that reduces active MDL to single-distribution active learning in one shot, Algorithm~\ref{alg:main-large-eps} reduces active MDL to a series of passive MDL problems with decreasing target excess error, similar to phased and robust versions of disagreement-based active learning~\citep{cohn1994improving,hanneke2014theory,hsu2010algorithms}. 

Algorithm~\ref{alg:main-large-eps} maintains a version space, $V_n \subset \Hcal$ and aims to shrink it iteratively (step~\ref{step:update-vs}). 
Similar to disagreement-based active learning for the single-distribution setting~\citep{hanneke2014theory}, it queries the label of an example whenever it lies in the disagreement region of $V_n$. 
Specifically, at iteration $n$, it uses the algorithm of~\citet{blum2017collaborative} to perform passive multi-distribution learning on MDL instance $\Dcal_n = ( D_{i,n} )_{i \in [k]}$, 
whose individual distributions' probability mass functions (PMFs) are defined as\footnote{More generally, when $\Xcal$ is not necessarily discrete, define $D_{i,n}$ as:
\[\Pr_{D_{i,n}}[(x,y) \in A] := \Pr_{D_i}[x \in \DIS(V_{n-1}) \wedge (x,y) \in A] + \Pr_{D_i}[x \in \AGE(V_{n-1}) \wedge (x, V_{n-1}(x)) \in A],
\]
for any measurable $A \subset \Xcal \times \Ycal$. In this case, Proposition~\ref{prop:label-cost} can be proved similarly.
%$D_{i,n}(A) = \int_{\Xcal \times \Ycal} \rbr{ I(x \in \DIS(V_{n-1}) \wedge (x,y) \in A) 
%+ I( x \in \AGE(V_{n-1}) \wedge (x, V_{n-1}(x)) \in A ) } D( dx, dy ) 
%$.
}:
%\chicheng{Explain this more.. }
\begin{equation}
\label{eqn:D_i_n}
D_{i,n}(x,y) 
= 
D_i(x,y) I(x \in \DIS(V_{n-1}))
+ 
D_i(x) I(y = V_{n-1}(x)) I(x \in \AGE(V_{n-1})),
\end{equation}
where we recall that $V_{n-1}(x)$ denotes the unanimous prediction of classifiers in $V_{n-1}$ on $x \in \AGE(V_{n-1})$.

A random sample $(x,y)$ from $D_{i,n}$ can be obtained in a label-efficient manner
as in Algorithm~\ref{alg:sample-imputed} in Appendix~\ref{appdix:sampling}: first, use example oracle $\EX_i$ to draw $x \sim D_{i,X}$ (step~\ref{step:get-x}); if $x \in \DIS(V_{n-1})$, query the labeling oracle $\Ocal_i$ for label $y$ (step~\ref{step:dis});
otherwise $x \in \AGE(V_{n-1})$, we infer label $y$ to be the prediction of any $h \in V_{n-1}$ on $x$ (step~\ref{step:agr}). As we will see, Algorithm~\ref{alg:main-large-eps} maintains the invariant that $h^* \in V_n$, and thus the inferred label equals $h^*(x)$, which maintains a favorable bias for PAC learning the original distributions~\cite[][Lemma 5.2; see also Lemma~\ref{lem:favbias}]{hsu2010algorithms}.
It can be readily seen that each call has expected label cost $\Pr_{D_i}\sbr{x \in V_{n-1}}$ (Proposition~\ref{prop:label-cost} in Appendix~\ref{appdix:sampling}).

%%$D_{i,\Ycal}(x)$ is the conditional probability of $Y$ under $D_i$ given $x$. 

%$I(A)$ is the indicator of event $A$ and
% \blue{This seems to be a weird way of writing it down. I think a more correct way of doing this is:
% \[
% D_{i,n}(x,y) 
% = D_{i,\Xcal}(x)\bra{D_{i,\Ycal}(x)I(x \in \DIS(V_{n-1})+I(y = V_{n-1}(x)) I(x \notin \DIS(V_{n-1})}
% \]
% }

% \chicheng{I see what you try to say -- did you mean 
% \[
% D_{i,n}(x,y) 
% = D_{i,\Xcal}(x)\bra{D_{i,\Ycal}(y \mid x)I(x \in \DIS(V_{n-1})+I(y = V_{n-1}(x)) I(x \notin \DIS(V_{n-1})}
% \]
% }

%%\chicheng{Change $\PP$ to $\Pr$}

%\chicheng{Provide a proof of this in the appendix.}

%It can be seen that the above procedure queries the 

%This gives a classifier $\hat{h}$ such that 
%$L( \hat{h}, \overline{D} ) \leq $

\begin{algorithm}[h]
\caption{Distribution-dependent active multi-distribution learning, large $\eps$ regime}
\label{alg:main-large-eps}
\begin{algorithmic}[1]
\REQUIRE{Target error $\eps > 0$, failure probability $\delta > 0$}

\STATE Initialization: $V_0 \gets \Hcal$, $n_0 \gets \lceil \log\frac{1}{\varepsilon} \rceil$,
$\eps_n = 2^{-n}$, $\delta_n = \frac{\delta}{2n^2}$

\FOR{$n = 1,.., n_0$}
%\STATE Call~\cite{blum2017collaborative} with target error $\eps_m = 2^{-m}$ 

%\STATE For the operation below, if examples are sampled in $\DIS(V_n)$ we query for its label, otherwise we infer its label (the one predicted unanimously by elements in $V_n$)

%\STATE Define $D_n = $

%\STATE $h_n \gets \textsc{Passive-MDL}(V_n, \cbr{ {D_i}|_{\DIS(V_n)} }_{i \in [k]}, \frac{\varepsilon_n}{ 
%\max_{j \in [k]} \PP_{D_j}(\DIS(V_n))}, \delta_n)$

\STATE Define $D_{i,n}$ as in Eq.~\eqref{eqn:D_i_n}, for all $i \in [k]$. 

\STATE $h_n \gets \textsc{Passive-MDL}(V_{n-1}, ( D_{i,n} )_{i \in [k]}, \varepsilon_n, \delta_n)$, where we
choose \textsc{Passive-MDL} to be the passive multi-distribution learning algorithm of~\cite{blum2017collaborative}, and 
use Algorithm~\ref{alg:sample-imputed} to obtain iid samples from $D_{i,n}$. 

%\chicheng{This algorithm seems only working under the realizable setting. The problem is that although globally, the optimal error is $\nu$, when zooming into $\DIS(V)$, the optimal error can be as large as $\frac{\nu}{ \min_j \PP_{D_j}(\DIS(V))}$. (note: the denomiator's min is not a typo)}

%Run an multi-distribution algorithm with accuracy parameter $\varepsilon_n=2^{-n}$ and probability parameter $\delta_n=\frac{\delta}{\log\frac{1}{\varepsilon}}$ and get a hypothesis $h_n$, using $m_n$ unlabelled samples.

% \begin{comment}
% \FOR{$r=1,..,\log k$:}

% \STATE Train $\hat{h}_r$ in $\frac{1}{N_r} \sum_{i \in N_r} D_i$ with target error $\eps = 2^{-m}$ 

% \STATE Test $\hat{h}_r$ on $D_i$'s, $i \in N_r$ with target error $\eps = 2^{-m}$ 

% \STATE Update $N_r$

% \ENDFOR
% \end{comment}

\STATE Update version space $V_{n} \gets \cbr{h \in V_{n-1}: \rho(h,h_n)\le 2\varepsilon_n}$ 

\label{step:update-vs}

\ENDFOR
\RETURN $\hat{h}$, an arbitrary classifier from $V_{n_0}$.
\end{algorithmic}
\end{algorithm}

We now present the guarantee of Algorithm~\ref{alg:main-large-eps}:

\begin{theorem}\label{thm:largeepsilon}
Suppose Assumption~\ref{assn:large-eps} holds.
If Algorithm~\ref{alg:main-large-eps} takes into target error $\eps$ and failure probability $\delta$, then with probability $1-\delta$,
(1) its output classifier $\hat{h}$ is such that 
$L(\hat{h})
\leq \nu + \eps$,
(2) it queries $O \rbr{ (d + k)\theta_{\max}(\eps) \ln\frac1\eps } $ labels.
\end{theorem}

Our theorem implies a $O \rbr{ (d + k)\theta_{\max}(\eps) \ln\frac1\eps }$ label complexity upper bound for active MDL. In Section~\ref{sec:lower-bound} below, we show a $\Omega( k \theta_{\max} (\varepsilon) )$ information-theoretic label complexity lower bound; this, combined with the $\Omega(d \theta(\varepsilon) )$ lower bound in~\cite{hanneke2014theory}, shows that our label complexity upper bound is unimprovable in general, up to a $O\rbr{ \ln\frac1\varepsilon}$ factor. 

The proof of Thoerem~\ref{thm:largeepsilon} can be found in Appendix~\ref{sec:def-large-eps}. 
Its key idea is as follows: Algorithm~\ref{alg:main-large-eps} iteratively shrinks the version spaces $V_n$ and maintains two invariants with high probability: first, $h^* \in V_n$; second, for all $h \in V_n$, $\max_{i \in [k]} \rho_i(h, h^*) = \rho(h, h^*) \leq O(\eps_n)$. The first invariant ensures that the distributions $D_{i,n}$ have favorable biases (as mentioned above), ensuring the final PAC learning guarantee.
The second invariant generalizes similar claims in the analyses of single-distribution active learning~\cite[e.g.][]{hanneke2014theory,hsu2010algorithms}; it ensures that the disagreement region of $V_n$ are have small probabilities (i.e., $\leq \theta_{\max}(\varepsilon) \cdot \eps_n$) under all distributions $(D_i)_{i=1}^k$. To this end, it learns $h_n$ such that it has small errors (i.e., $\leq \eps_n$) under all $D_{i,n}$ simulatenously (see Lemma~\ref{lem:passive-mdl}).

\section{Active MDL in Small \texorpdfstring{$\varepsilon$}{} Regime}
\label{sec:small-eps}

% in the large $\eps$ regime

We next move on to the more challenging small $\eps$ regime, where $\nu$, the noise level in the distributions $\rbr{D_i}_{i=1}^k$, is large compared with target excess error $\eps$.
As mentioned in the introduction section, state-of-the-art result of~\cite{rittler2024agnostic} generalizes the idea of robust single-distribution active learning, achieving a label complexity of $\widetilde{O}\bra{\frac{\nu^2}{\varepsilon^2}kd\theta_{\max}(\nu)^2+\frac{k}{\varepsilon^2}}$. Again, this label complexity may sometimes be worse than passive learning: for example, when  $\theta_{\max}(\nu) = O(\frac{1}{\nu})$, the bound becomes $O(\frac{kd}{\varepsilon^2})$, which is higher than state-of-the-art passive multi-distribution learning sample complexity $O(\frac{k + d}{\varepsilon^2})$~\citep{zhang2024optimal,peng2024sample}. 

%\theta_{\max}(\nu)
% specifically,  in Algorithm~\ref{alg:main-large-eps}, labels in the agreement region are inferred rather than queried, which reduces the number of label queries. However, this inference procedure fails when the excess error becomes much smaller—because, in the agnostic case, the error of predictions in the agreement region can be of the order \(O(\nu)\). While this is acceptable when the target excess error is large relative to $\nu$, it becomes problematic when the target is much smaller.

Although Algorithm~\ref{alg:main-large-eps} has near-optimal label complexity when $\eps \geq O(\nu)$, it cannot be directly applied to the small-$\eps$ regime that $\eps \leq O(\nu)$. 
Indeed, the version space construction of Algorithm~\ref{alg:main-large-eps} (step~\ref{step:update-vs}) ensures the invariant $h^* \in V_n$ only when the target excess error $\eps$ is much larger than $\nu$.
We further illustrate the challenge with an example, showing the necessity of label querying in the agreement region of the hypothesis class:

\begin{example}\label{exp:smalleps}
Let \(\Hcal=\{h_1,h_2\}\). Suppose \(D_1\) is a distribution over \(X_1\subseteq\Xcal\) with \(L_1(h_1)=0\) and \(L_1(h_2)=2\nu'\). Also, let 
\(D_2=\nu' D_{X_2}+(1-\nu')D_{X'_2}\), where \(D_{X_2}\) and \(D_{X'_2}\) are distributions over \(X_2\) and \(X'_2\), respectively, with \(X_2\subseteq\DIS(\Hcal)\) and \(X'_2\subseteq\AGE(\Hcal)\). Assume further that \(X_1\), \(X_2\), and \(X'_2\) are pairwise disjoint, and that \(L(h_1,D_{X_2})=1\) while \(L(h_2,D_{X_2})=0\). Now consider two cases for setting the error in the agreement region:
\begin{enumerate}
    \item[(a)] If we set \((1-\nu')L(h_1,D_{X'_2})=(1-\nu')L(h_2,D_{X'_2})\le \nu'-\varepsilon\), then \(L_2(h_1)\le 2\nu'-\varepsilon\) and \(L_2(h_2)\le \nu'-\varepsilon\), so that \(h_1\) is the only valid output when the target excess error is \(\varepsilon\).
    \item[(b)] If we set \((1-\nu')L(h_1,D_{X'_2})=(1-\nu')L(h_2,D_{X'_2})\ge \nu'+\varepsilon\), then \(L_2(h_1)\ge 2\nu'+\varepsilon\), making \(h_2\) the only valid output when the target excess error is \(\varepsilon\).
\end{enumerate}
Note that this example is valid only when \(\varepsilon\ll\nu\) (since \(\nu'-\varepsilon\ge0\)), which aligns with our discussion.
\end{example}

%, in the sense that
%In Algorithm~\ref{alg:main-large-eps},
%To overcome this issue
%that aims at addressing the above challenges

In view of this, we present an algorithm, namely Algorithm~\ref{alg:main-small-eps},  for active MDL in the small \(\varepsilon\) regime. Our approach consists of two stages. In stage one, we run Algorithm~\ref{alg:main-large-eps} with a target excess error of \(\varepsilon=100\nu\), obtaining a version space $V_0$ that contains $h^*$ with small radius, in the sense that $\max_{h \in V_0} \rho(h,h^*) \leq O(\nu)$. From Theorem~\ref{thm:largeepsilon}, this stage costs $O( \theta_{\max}(\nu) (d+k) \ln\frac1\nu )$ label queries. 
After obtaining $V_0$, we introduce a second stage that relies on estimating the error of $h^*$ in the agreement region for each of the \(k\) distributions. In particular, it  
first draws iid samples $(S_i)_{i=1}^k$ in the agreement region and use their empirical distributions
as surrogates of $D_i|_{\AGR(V_0)}$
(step~\ref{step:surrogate-emp}), and then runs the passive MDL algorithm on distributions $(D_i')_{i=1}^k$, which can be viewed as surrogates of $(D_i)_{i=1}^k$
(step~\ref{step:surrogate}).
Sampling from $D_i'$ can be done in a label-efficient manner, as we detail in Algorithm~\ref{alg:sample-imputed-smalleps}: with probability $1-\Pr_{D_i}[x \in \DIS(V_0)]$, we draw samples from $D_i|_{\DIS(V_0)}$ by making new label queries; otherwise, we draw an example uniformly at random from the already-labeled dataset $S_i$. 
%The full algorithm is presented below:

%uses the sample-based estimation of the error in the agreement region 

% To derive the desired label complexity bound, we refine the existing sample complexity bound for passive MDL in \citet{zhang2024optimal} to \(\widetilde{O}\Bigl(\frac{(d+k)\nu}{\varepsilon^2}\Bigr)\), as we prove in Section~\ref{sec:passivesamplecomplexity}.

\begin{algorithm}[h]
\caption{Distribution-dependent Active Multi-distribution Learning, Small \(\varepsilon\) Regime}
\label{alg:main-small-eps}
\begin{algorithmic}[1]
\REQUIRE Target excess error \(\varepsilon>0\), optimal error \(\nu\), failure probability \(\delta>0\).
\STATE Run Algorithm~\ref{alg:main-large-eps} with target excess error \(\varepsilon' = 100\nu\) and failure probability \(\delta' = \frac{1}{6}\delta\); let \(h'\) be the output.
\label{step:stage-one}

\STATE Initialize 
$V_0 \gets \Hcal' = \{h\in\Hcal : \dist(h,h')\le 2\varepsilon'\}$ and $n_0 = \Bigl\lceil \frac{100(\varepsilon+\nu)}{\varepsilon^2}\ln\frac{k}{\delta'}\Bigr\rceil$.
\label{step:stage-two-begin}

\FOR{\(i=1,\dots,k\)}
    \STATE Sample \(n_0\) samples from \(D_i\) conditioned on \(x\in\AGE(V_0)\) and query their labels; denote the sample set as $S_i=\{(x_1,y_1),\dots,(x_{n_0},y_{n_0})\}$.
    \label{step:agreement-query}

    \STATE Define the empirical distribution $D_{S_i}(x,y)=\frac{1}{n_0}\sum_{(x',y')\in S_i} I(x'=x \text{ and } y'=y)$.
    \label{step:surrogate-emp}
    
    \STATE Define the surrogate distribution \(D'_i\) as
    \[
    D'_i(x,y)=I\rbr{x\in\DIS(V_0)}D_i(x,y)+\Pr_{x\sim D_{i,\Xcal}}[x\in\AGE(V_0)]\,D_{S_i}(x,y).
    \]
    \label{step:surrogate}
\ENDFOR
\RETURN \(\hat{h}\gets \textsc{Passive-MDL}\bra{V_0,\{D'_i\}_{i\in[k]},\frac{\varepsilon}{2},\frac{\delta}{6}}\), where we choose \textsc{Passive-MDL} to be the \texttt{MDL-Hedge-VC} algorithm of~\cite{zhang2024optimal} (with different hyperparameters; see Section~\ref{sec:passivesamplecomplexity}), and 
use Algorithm~\ref{alg:sample-imputed-smalleps} to obtain iid samples from $D_i'$.
\label{step:stage-two-end}
\end{algorithmic}
\end{algorithm}
%in Appendix~\ref{appdix:samplingsmalleps}

\begin{algorithm}[h]
\caption{Sampling from \(D'_i\)}
\label{alg:sample-imputed-smalleps}
\begin{algorithmic}[1]
\REQUIRE Samples $(S_i)_{i=1}^k$ drawn from $(D_i)_{i=1}^k$, and version space \(V_0\).
\STATE Use the example oracle \(\EX_i\) to sample \(x\sim D_{i,\Xcal}\).
\IF{\(x\in\DIS(V_0)\)}
    \STATE Set \(y\gets \Ocal_i(x)\).
\ELSE
    \STATE Uniformly sample \((x,y)\) from \(S_i\).
\ENDIF
\RETURN \((x,y)\).
\end{algorithmic}
\end{algorithm}

The following theorem gives the correctness and label complexity of Algorithm~\ref{alg:main-small-eps}.

%\chicheng{Need to discuss it, I think. 
%How is this proved? Why is $n_0$ chosen in that way? How close is this to information-theoretic limit?
%}

\begin{theorem}\label{thm:smalleps-main-thm}
If Algorithm~\ref{alg:main-small-eps} is run with target error \(\varepsilon\) and failure probability \(\delta\), then with probability at least \(1-\delta\): (1) its output classifier \(\hat{h}\) satisfies 
$
L\bra{\hat{h}}\le \nu+\varepsilon,
$
and (2) its label complexity is at most:
\[
O\Bigl(\Bigl(\frac{k \nu}{\varepsilon^2}+\frac{\theta_{\max}\bra{100\nu}(d+k)(\nu+\varepsilon)^2}{\varepsilon^2}\Bigr)\cdot \polylog\Bigl(k, d, \frac{1}{\varepsilon}, \frac{1}{\delta}\Bigr)\Bigr)
\]
\end{theorem}

The proof of Theorem~\ref{thm:smalleps-main-thm} and the auxiliary lemmas are provided in Appendix~\ref{sec:proofofsmallepsmain}. On a high level, the second stage's label cost dominates that of the first stage.
In the second stage, 
we choose the number of samples \(n_0\) for $S_i$ so that we estimate $h^*$'s error in $\AGR(V_0)$ with precision $O(\eps)$, which allows us to apply the sample complexity bound of the passive MDL algorithm; specifically $n_0 = \widetilde{O}(\frac{\nu}{\eps^2})$ is enough, resulting in the $\frac{k \nu}{\varepsilon^2}$ term in the label complexity.

Furthermore, suppose \PMDL makes a total of $n_1$ calls to any of the sampling procedure from $D_i'$'s; by Chernoff bound, it will make an additional $n_1 \max_i \Pr_{D_i}(x \in \DIS(V_0))$ label queries.  
Note that if we were to directly apply the passive sample complexity upper bound \(n_1 \leq \widetilde{O}\bra{\frac{d+k}{\varepsilon^2}}\) from \citet{zhang2024optimal}, we would obtain a label complexity bound of \(\widetilde{O}\bra{\frac{k\nu}{\varepsilon^2}+\frac{\theta_{\max}(100 \nu) (d+k) \nu }{\varepsilon^2}}\). 
Such bound still does not fully exploit the benefit of having small $\nu$: for example, when $\nu = O(\eps)$ and $\theta_{\max}(100 \nu) = O(\frac 1 {\nu})$, this gives a $\widetilde{O}(\frac{d+k}{\eps^2})$ label complexity, while the naive  $\widetilde{O}(\frac{kd(\nu + \eps)}{\eps^2})$ sample complexity baseline we mentioned in Section~\ref{sec:intro} evaluates to $\widetilde{O}(\frac{kd}{\eps})$, which can be significantly better for small $\eps$. Motivated by this, we refine the sample complexity bound of the passive MDL algorithm by~\cite{zhang2024optimal} and use it to establish an improved $\widetilde{O}\Bigl( \frac{k \nu}{\varepsilon^2}+\frac{\theta_{\max}\bra{100\nu}(d+k)(\nu+\varepsilon)^2}{\varepsilon^2} \Bigr)$ label complexity as stated in Theorem~\ref{thm:smalleps-main-thm}.

%Compared with the $\widetilde{O}( \theta(\nu) (1 + \frac{\nu^2}{\eps^2}) ) $ label complexity in the PAC ($k=1$) setting~\citep{dasgupta2007general,hanneke2007bound}, the second term can be significantly worse. 

%which could still be worse than the passive algorithm when $\nu = O(\eps)$ and \(\theta_{\max}(\nu)\) is large. Therefore, it is necessary to refine the sample complexity bound of the passive MDL algorithm by~\cite{zhang2024optimal}, as we detail below.

\subsection{Refined Sample Complexity Bounds for Passive MDL}
\label{sec:passivesamplecomplexity}

We refine the sample complexity of \texttt{MDL-Hedge-VC}, the passive MDL algorithm of \citet[][Algorithm 1]{zhang2024optimal} 
from $\widetilde{O}(\frac{d+k}{\eps^2})$
to $\widetilde{O}(\frac{(d+k)(\eps + \nu)}{\eps^2})$, 
by choosing different hyperparameters $\eta$, $T$ and $T_1$. 
For completeness, we give a brief recap of \texttt{MDL-Hedge-VC} in Appendix~\ref{sec:proofofpassivetheorem}. 
This is analogous to refined sample complexity analysis of empirical risk minimization from $\widetilde{O}(\frac{d}{\eps^2})$ to $\widetilde{O}(\frac{d (\nu + \eps)}{\eps^2})$ in the single-distribution PAC learning setting~\citep[][Section 5]{boucheron2005theory,vapnik1982estimation}.
When both \(\nu\) and \(\varepsilon\) are \( \ll 1 \), Our new bound is tighter.
We summarize our refined algorithm and guarantees below, with proofs deferred to Appendix~\ref{sec:proofofpassivetheorem}:

%\(\widetilde{O}\bra{\frac{(d+k)(\nu+\varepsilon)}{\varepsilon^2}}\)
%, which is specified in the theorem
%Note that we did not change the algorithm, but instead 
\begin{theorem}\label{theorem:passiveupperbound} Set $\varepsilon_1 = \frac{\varepsilon}{100}$ and $\eta = \frac{\varepsilon_1}{100(\varepsilon_1 +\nu)}$. Further, set $T = 20000\rbr{\frac{1}{\varepsilon_1}+\frac{\nu}{\varepsilon_1^2}}\ln \rbr{\frac{k}{\delta\eps}}$ and 
$T_1 = 4000\rbr{\frac{1}{\varepsilon_1}+\frac{\nu}{\varepsilon_1^2}}\rbr{k\ln\rbr{\frac{k}{\varepsilon}}+d\ln\rbr{\frac{kd}{\varepsilon}}+\ln \rbr{\frac{1}{\delta}}}$.  
Then the randomized hypothesis $h^{\final}$ returned by Algorithm \ref{alg:mdl-hedge-vc} satisfies $L(h^{\final}) \le \nu + \varepsilon$ with probability at least $1-\delta$, provided the total sample size exceeds
\[
O\rbr{ \frac{(d+k)(\nu+\varepsilon)}{\varepsilon^2} \cdot\polylog(k,d,1/\varepsilon,1/\delta) }.
\]
\end{theorem}
Notably, our bound provides a smooth transition between the \(\widetilde{O}\bra{\frac{d+k}{\varepsilon^2}}\) bound in \citet{zhang2024optimal} for the agnostic setting and the \(\widetilde{O}\bra{\frac{d+k}{\varepsilon}}\) bound in \citet{blum2017collaborative} for the realizable setting. 
At a high level, the improvement in the sample complexity bounds is achieved via better concentration bounds. We employ Bernstein-style concentration bounds instead of Hoeffding-style bounds, which yield a tighter analysis. One caveat is that since we have increased the step size $\eta$, a more delicate analysis is required to bound the norm of the weight vector $\norm{\overline{w}^T}_1$, which is crucial in establishing the sample complexity bounds.  
%\chicheng{need to make a pass}

%The proof and relevant lemmas appear in Appendix~\ref{sec:proofofpassivetheorem}. 

%\chicheng{Need to discuss when this is better than the original bound. Also, it smoothly interpolates between $\nu = 0$ and $\nu > 0$ -- let's mention this?}

\begin{comment}
\subsubsection{Proof of Theorem~\ref{theorem:passiveupperbound}}

\chicheng{
Writing style for making adjustments on some existing analysis to slightly extend it 
\url{https://arxiv.org/abs/1612.06246}
}

First we state three useful lemmas. The first lemma states that $h^t$ at each round is close to have minimum loss.
\chicheng{Let's highlight what is the difference between the lemma stated here and the one in the original~\cite{zhang2024optimal} paper? It is in the choice of $T_1$, I think?
}
\end{comment}

\section{Lower Bounds}
\label{sec:lower-bound}

We present our lower bounds in this section, for realizable and agnostic cases, respectively. Our lower bounds 
are information-theoretic in nature, and apply to algorithms we present in previous sections. 

%are meaningful in that they demonstrate that for any algorithm, there exists a problem instance where our algorithm's label complexity is matched, implying that no algorithm can be uniformly better on every instance—even ignoring computational cost.

\subsection{Lower Bound in the Realizable Setting}
We first prove a $\Omega\bra{k\theta_{\max}}$ lower bound in the realizable setting. As mentioned in Section~\ref{sec:large-eps}, this lower bound in conjunction with the $\Omega\bra{d\theta_{\max}}$ lower bound in the single distribution setting~\citep{hanneke2014theory} shows that the label complexity given in Theorem~\ref{thm:largeepsilon} is not improvable in general. 

%The lemma \joey{or theorem?} is presented below.
\begin{theorem}
\label{thm:realizable-lb}
For any $k \geq 2, d \geq 1, \vartheta \geq 1, \varepsilon \in (0, \frac{1}{2\vartheta})$, 
and hypothesis class $\Hcal$ whose $\VC(\Hcal)\leq d$ and star number is $\geq k \vartheta$, 
and proper active learning algorithm $A$, there exists a problem instance $\Dcal = \rbr{D_i}_{i\in[k]}$ such that: 
(1) $\min_{h \in \Hcal} L_\Dcal(h) = 0$; 
(2) $\theta_{\max}(\varepsilon) \leq \vartheta$; 
(3) 
unless \(A\) queries \(\Omega\bra{k\vartheta}\) labels, the probability that \(A\) returns an \(\varepsilon\)-optimal hypothesis is at most \(0.7\). 
\end{theorem}

In proving Theorem~\ref{thm:realizable-lb}, we construct an example with $k\theta_{\max}$ points and $k$ distributions, each span a disjoint section of the feature space, which any algorithm has to query at least a constant fraction to get information.

%\joey{We complete this section with a $\Omega\bra{d\theta_{\max}}$ lower bound mentioned in \citet{hanneke2014theory}.}

\subsection{Lower Bound in the Agnostic Setting}

We next prove a \(\Omega\bra{\frac{k\nu}{\varepsilon^2}}\) label copmlexity lower bound for all proper active MDL learners: 

% \chicheng{Discuss why this lower bound only holds for $\nu \geq \epsilon$!!}

% in a more general setting
%of \(\Omega\bra{\frac{k\nu}{\varepsilon^2}}\) by extending the idea of Example~\ref{exp:smalleps}
%This bound shows the $\Omega\bra{k\frac{\nu}{\varepsilon^2}}$ term is essential for all proper active learners for MDL.

\begin{theorem}
\label{thm:agnostic-lb}
For any $k \geq 2,d \geq 2,\varepsilon > 0,\frac{1}{2}\ge\nu > 0$ with \(\nu\ge 8\varepsilon\), and a hypothesis class $\Hcal$ that contains $h_1, h_2$ that agrees on at least two examples and disagrees on $k$ examples with $\VC(\Hcal) \leq d$, 
any proper active learning algorithm \(A\),
there exists a problem instance $\rbr{D_i}_{i\in[k]}$ such that: 
(1) $\min_{h \in \Hcal} L(h, \Dcal) \leq \nu$; 
(2) unless \(A\) queries \(\Omega\bra{k\frac{\nu}{\varepsilon^2}}\) labels, the probability that \(A\) returns an \(\varepsilon\)-optimal hypothesis is at most \(0.9\). 
\end{theorem}

%Moreover, in this problem instance \(\VC(\Hcal)\le\mathfrak{s}\le\abs{\Hcal}=2\), where $\mathfrak{s}$ is the star number of the problem instance.

Theorem~\ref{thm:largeepsilon} demonstrates the necessity of the $k\frac{\nu}{\eps^2}$ term of label complexity in Theorem~\ref{thm:smalleps-main-thm}, since Algorithm~\ref{alg:main-small-eps} is a proper learning algorithm. This formalizes the intuition in~\citet[][Example 1]{rittler2024agnostic} and our Example~\ref{exp:smalleps}
that, in sharp contrast to active single-distribution learning, sampling from the disagreement region only may not be enough for active multi-distribution learning. 
Indeed, our lower bound instances highlight the need in querying the labels in the agreement region of the hypothesis class, ensuring the returned classifier balances its worst-case error across all $k$ distributions (see Appendix~\ref{sec:def-lower-bound} for the proof). 
%This highlights a unique challenge in active multi-distribution learning, even for $k=2$. 

In light of our upper bound in Section~\ref{sec:large-eps}, we observe that such $\Omega(k \frac{\nu}{\eps^2})$ lower bound cannot appear in the large $\eps$ regime, i.e., when $\nu \leq \frac{\eps}{100}$. This shows an interesting ``phase transition'' behavior of the fundamental label complexity of active multi-distribution learning, which does not appear in classical PAC active learning~\citep{hanneke2014theory,raginsky2011lower}. We conjecture that similar lower bounds can be established for improper learning algorithms, and leave it as an interesting open question.

\color{black}

\section{Active MDL Algorithms in Distribution-free Settings}
\label{sec:distn-free}

Theorems~\ref{thm:largeepsilon} and~\ref{thm:smalleps-main-thm} provide useful algorithmic results on active MDL with improved label complexity; even though optimality properties have been established, their bounds can sometimes be suboptimal when considering other problem-dependent complexity measures. One important distribution-free quantity that characterizes the complexity of active learning is the \emph{star number} of the hypothesis class $\Hcal$~\citep{hanneke2015minimax}, denoted as $\s$ (recall its definition in Section~\ref{sec:prelims}). 
It is known that for any distribution $D$ and $r > 0$, $\s \geq \theta_{D}(r)$~\citep{hanneke2015minimax}, and this bound can sometimes be tight. Translating our distribution-dependent upper bounds Theorems~\ref{thm:largeepsilon} and~\ref{thm:smalleps-main-thm} in terms of star number, we obtain that Algorithms~\ref{alg:main-large-eps} and~\ref{alg:main-small-eps} have label complexities $O\rbr{ \s (d + k) \ln\frac1\epsilon }$ and $O\rbr{ \s (d+k) \rbr{1+\frac{\nu^2}{\eps^2}} + \frac{d \nu}{\eps^2}}$, in the small and large $\eps$ regimes, respectively. How tight are these bounds?

% be improved

We start with making the simple observation that these direct translations can sometimes result in suboptimal label complexity bounds. Specifically, in the realizable setting ($\nu = 0$), consider the naive algorithm that performs disagreement-based active learning over the average distribution $\overline{D} = \frac1k \sum_{i=1}^k D_i$ with target error $\frac \eps k$. This ensures that we return $\hat{h}$ such that $\frac1k \sum_{i=1}^k L_i(\hat{h}) \leq \frac \eps k$, and therefore $L_i(\hat{h}) \leq \eps$. Its label complexity is at most $\widetilde{O}\rbr{ 
\s \ln\frac{k}{\eps} }$~\citep{wiener2015compression}, much better than $O\rbr{ \s (d + k) \ln\frac1\epsilon }$ provided by Algorithm~\ref{alg:main-large-eps}. 
However, when generalizing to the nonrealizable setting ($\nu > 0$), this reduction can only yield PAC guarantees when $\nu \lesssim \frac{\eps}{k}$, and 
in combination with state-of-the-art distribution-free agnostic active learning guarantees~\cite[][Theorem 8]{hanneke2015minimax}, this gives a label complexity of 
$\widetilde{O} \rbr{ \s d \polylog(\frac1\eps) }$. This motivates our question: 
can we design active MDL algorithms with sharp label complexity guarantees in terms of star number, that smoothly interpolates between realizable and nonrealizable regimes? 
%\joey{Is our bound really a smooth translation? Also two questions in two consecutive paragraphs is repetitive.}

We answer this question in the positive in this section by designing an active MDL algorithm, Algorithm~\ref{alg:main-large-eps-df} with label complexity $\widetilde{O}(\s \ln \frac1\eps)$ when the target error $\eps \geq (k+d) \nu$. 
Similar to distribution-free algorithms for single-distribution active learning~\citep{wiener2015compression,kane2017active}, Algorithm~\ref{alg:main-large-eps-df} progressively learns  Reliably and Probably Useful (RPU) classifiers $f_n$~\citep{rivest1988learning,hopkins2020power} with larger coverages. RPU classifiers are those that have an extra option of outputting $0$, indicating ``I don't know''; we now give its formal definition: 
\begin{definition}
\label{def:rpu-classifier}
A classifier $f: \Xcal \to \cbr{-1, +1, 0}$ is said to be $\xi$-Reliable and Probably Useful (RPU) with respect to $h^*: \Xcal \to \cbr{-1, +1}$ and distribution $D$, if it is simultaneously: 
\begin{enumerate}
\item Reliable: $\Pr[ f(x) \neq 0, f(x) \neq h^*(x) ] = 0$. 

%$f(x)$ always agrees with $h^*(x)$ whenever $f(x) \neq 0$;
%the abstention 
%probability of $f_n$, 

\item Probably useful: $\Pr[ f_n(x) = 0 ] \leq \xi$.
\end{enumerate}
Additionally, $f$ is said to be $\xi$-RPU with respect to $h^*$ and $\Dcal = (D_i)_{i=1}^k$, if it is $\xi$-RPU with respect to $h^*, D_i$ for all $i$'s.
\end{definition}

% in other words, its coverage $\Pr[ f_n(x) \neq 0 ]$, is large. 
% By ``Reliably and Probably Useful'', $f_n$ is meant to be simultaneously: 
% \begin{enumerate}
% \item Reliable: $f(x)$ always agrees with $h^*(x)$ whenever $f(x) \neq 0$; 

% \item Probably useful: the abstention probability of $f_n$, $\Pr[ f_n(x) = 0 ]$ is small; in other words, its coverage $\Pr[ f_n(x) \neq 0 ]$, is large. 
% \end{enumerate}

Specifically, we design a subprocedure \textsc{Passive-RPU-MDL} (Algorithm~\ref{alg:prmdl} in Appendix~\ref{sec:prmdl-grt})
to iteractively refine such RPU classifiers using label queries (step~\ref{step:refine-rpu}), 
such that the abstention probabilities of $f_n$ shrink exponentially in $n$, uniformly across all distributions $\cbr{D_i}_{i=1}^k$. Taking advantage of active learning, at iteration $n$ we only make label queries in the abstention region $\cbr{x: f_{n-1}(x) = 0}$ (see Definition of $D_{i,n}$ in step~\ref{step:d-in-df}).

\begin{algorithm}[t]
\caption{Distribution-free active multi-distribution learning, large $\eps$ regime}
\label{alg:main-large-eps-df}
\begin{algorithmic}[1]
\REQUIRE{Target error $\eps > 0$, failure probability $\delta > 0$}

\STATE Initialization: $V_0 \gets \Hcal$, $n_0 \gets \lceil \log\frac{(d+k)}{\s\varepsilon} \rceil$,
$\eps_n = 2^{-n}$, $\delta_n = \frac{\delta}{2n^2}$, and $f_0 \equiv 0$.

\FOR{$n = 1,.., n_0$}

\STATE Define $D_{i,n}$ as:
\[
D_{i,n}(x,y) 
= 
D_i(x,y) I( f_{n-1}(x) = 0 )
+ 
I(y = f_{n-1}(x)) I(f_{n-1}(x) \neq 0),
\]
for all $i \in [k]$. 

\label{step:d-in-df}

\IF{$n < n_0$}

\STATE $f_{n} \gets \textsc{Passive-RPU-MDL}(\Hcal, \cbr{ {D_{i,n} } }_{i \in [k]}, \varepsilon_n, \delta_n)$

\label{step:refine-rpu}

\ELSE

%\COMMENT 
\STATE \blue{// In this case, $\Pr(f_n(x) = 0) \leq \frac{\s}{(d+k)} \eps$}

\RETURN $\hat{h} \gets \textsc{Passive-MDL}(\Hcal, \cbr{ {D_{i,n} } }_{i \in [k]}, \varepsilon_n, \delta_n)$.

\label{step:return-last-it}

\ENDIF

\ENDFOR

\end{algorithmic}
\end{algorithm}

%(\joey{$f_{n_0-1}$?}) 

%\chicheng{The algorithm walkthrough can be done better here by referring to the steps in the algorithm}

At the last epoch $n_0$, we have access to an RPU classifier $f_{n_0-1}$ with abstention probability $O(\frac{\s}{d+k} \eps)$ in all $D_i$'s. 
Were we to directly covert $f_{n_0-1}$ to a binary classifier by predicting arbitrarily in its abstention region, we would get a somewhat undesirable  $O(\frac{\s}{(d+k)} \eps)$ excess error guarantee. 
Our key observation here is that, with a constant factor overhead of label cost, we can do a better RPU-to-PAC conversion by reusing the \textsc{Passive-MDL} procedure in the preceding sections.
Specifically, with $O(\s)$ labels, \textsc{Passive-MDL} on $D_{i,n}$'s
outputs a classifier $\hat{h}$ with excess error $\eps$ (step~\ref{step:return-last-it}).

We present the performance guarantees of Algorithm~\ref{alg:main-large-eps-df} in the theorem below: 

\begin{theorem}
\label{thm:main-df}
Suppose $\eps \geq 100 (k+d) \nu$.
If Algorithm~\ref{alg:main-large-eps-df} takes into target error $\eps$ and failure probability $\delta$, then with probability $1-\delta$,
(1) its output classifier $\hat{h}$ is such that 
$
L(\hat{h})
\leq \nu + \eps
$,
(2) it queries $O \rbr{ \s  \ln\frac1\eps } $ labels.
\end{theorem}
%\max_{i \in [k]} L( \hat{h}, D_i )

%Moreover, since $\max_{i \in [k]} L( h^*, D_{i,n} ) \leq \nu$, by triangle inequality, 

The proof of  Theorem~\ref{thm:main-df} can be found in Appendix~\ref{sec:df-deferred}. 
Its main ideas are twofold. First, we argue that with $\widetilde{O}( \s )$ label queries, at every iteration $n$, \PRMDL computes $f_n$ whose abstention region is half the size of those of $f_{n-1}$, measured in all $D_i$'s. Due to the presence of label noise, \PRMDL needs to be designed in a robust way; specifically it relies on our design of a new single-distribution RPU algorithm, \PRSDL (Algorithm~\ref{alg:robust-rpu}), that can tolerate adversarial label noise. 
Second, we argue that at the last iteration, calling \PMDL makes a total of $O( \frac{d+k}{\eps} )$ samples to any of the $D_i$'s; since we only need to query $f_{n_0-1}$'s abstention regions, this results in a factor of $O(\frac{\s}{d+k} \eps)$ label savings, yielding a final label complexity of $O(\s \ln\frac1\eps)$. 

For the single-distribution setting that $k=1$, Algorithm~\ref{alg:main-large-eps-df} achieves a label complexity of $O(\s \ln\frac1\eps)$ when the target error $\eps \geq O( d \nu )$. To the best of our knowledge, this result already improves over the state of the art $O(d \s)$~\cite[][Theorem 8]{hanneke2015minimax}, which may be of independent interest.

%Our contribution goes further in the following sense: a naive combination of our single-distribution result with the uniform mixture reduction gives an algorithm with label complexity $O(\s \ln\frac1\eps)$ when the target error $\eps \geq k d \nu$; we significantly broaden the regime where such ``best possible'' $O(\s \ln\frac1\eps)$ label complexity is achievable, from $\eps \geq k d \nu$ to $\eps \geq (k+d) \nu$.
%\joey{I personally find this paragraph a bit messy.}
%\chicheng{Agreed, I will make a pass}

%\chicheng{Formalize the lower bound of HJZ'22 in the active setting?}

\section{Conclusion}
In this paper, we establish novel and tighter label complexity bounds for active MDL, significantly improving the dependence on \( kd \) to an additive \( (k+d) \). Specifically, we develop algorithms and lower bounds in both the large and small \( \varepsilon \) regimes, achieving distribution-dependent and distribution-free bounds that are sometimes tight. Our results bridge the gap between MDL and active PAC learning, refining previous bounds and providing a more nuanced understanding of label complexity in multi-distribution settings. 
All our algorithms we present are proper, and our $\Omega(\frac{k \nu}{\eps^2})$ lower bound applies to proper learning only; it would be interesting to investigate whether improper learning has an benefit in improving label complexity. Our passive and active learning algorithms requires the knowledge of $\nu$; it would be nice to design adaptive algorithms without such knowledge. Another promising direction is to analyze a wider variety of noise settings beyond agnostic with optimal error $\nu$, such as the ones studied in~\cite{hanneke2015minimax}. 
We are also interested in designing computationally efficient versions of our algorithms, perhaps by utilizing regression-based active learning~\citep[e.g.,][]{zhu2022efficient,sekhari2024selective}.

% Acknowledgments---Will not appear in anonymized version
\acks{
We thank the anonymous COLT reviewers for their constructive comments. 
We thank Eric Zhao for helpful communications about the results in~\cite{haghtalab2022demand}. 
We thank Steve Hanneke for helpful conversations about some preliminary results in this paper.
We thank Zihan Zhang for confirming a technical detail in~\cite{zhang2024optimal}. 
CZ would like to thank Nick Rittler and Kamalika Chaudhuri for sparking  interest in the active multi-distribution learning problem. CZ acknowledges
support from the University of Arizona FY23 Eighteenth Mile TRIF Funding. YZ was supported by the NSF CAREER Award CCF-1751040 and by the NSF AI Institute for Foundations of Machine Learning (IFML). YZ thanks his advisor, Eric Price, for his encouragement and guidance.
}

%We thank a bunch of people and funding agency.

\bibliography{ref}

\clearpage

\appendix

% \crefalias{section}{appendix} % uncomment if you are using cleveref

\section{Deferred Materials for  Section~\ref{sec:large-eps}}
\label{sec:def-large-eps}

\subsection{Sampling Algorithm}\label{appdix:sampling}
We formally present the sampling algorithm used in Algorithm~\ref{alg:main-large-eps} as the following.
\begin{algorithm}[h]
\caption{Sampling from $D_{i,n}$}
\label{alg:sample-imputed}
\begin{algorithmic}
\REQUIRE Version space $V_n$
\STATE Query $\EX_i$ to sample $x \sim D_{i,X}$
\label{step:get-x}
\IF{$x \in \DIS(V_n)$}
\STATE $y \gets \Ocal_i(x)$
\label{step:dis}
\ELSE 
\STATE $y \gets h(x)$, where $h$ is an arbitrary classifier in $V_n$
\label{step:agr}
\ENDIF
\RETURN $(x,y)$
\end{algorithmic}
\end{algorithm}

Then Proposition~\ref{prop:label-cost} gives the expected label cost of each call of Algorithm~\ref{alg:sample-imputed} and the proof is given in Appendix~\ref{sec:proof-props}:
\begin{proposition}
\label{prop:label-cost}
Each call to Algorithm~\ref{alg:sample-imputed} 
returns $(x,y)$, an independent sample drawn from $D_{i,n}$, and 
has an expected label cost of  $\Pr_{D_i}\sbr{x \in V_{n-1}}$. 
\end{proposition}

\subsection{Proofs of Propositions}
\label{sec:proof-props}

%\chicheng{Have we decided whether to use square bracket or round bracket for probabilities?}\\
%\joey{I've used square bracket...}

% \joey{Need to prove Proposition 1}
% \chicheng{I wrote it up. Let me know what you think!}

\begin{proof}[Proof of Proposition~\ref{prop:avg-fail}]
Let $\Xcal = \cbr{x_0, x_1, \ldots, x_k}$, and $\Hcal = \cbr{h^*, h_1, \ldots, h_k }$, such that $h^* \equiv -1$, and for every $i \in [k]$, $h_i$ is defined as:  
\[
h_i( x_j ) = 
\begin{cases}
+1 & j = i \\
-1 & \text{otherwise}
\end{cases}
\]
For every $i \in [k]$, we construction distribution $D_i$ with PMF:
\[
D_i(x,y)= (1-\eps) I(x = x_0, y = -1)  + \eps I(x = x_i, y = +1 ). 
\]
It can be readily checked that for every $i \in [k]$, for $r \geq \eps$, 
$
\B_{D_i}(h^*, r) = \Hcal
$,
and $\DIS(\Hcal) = \cbr{ x_1, \ldots, x_k }$.
Therefore, for every $i \in [k]$
\[
\theta_i(\eps)
= 
\sup_{r \geq \eps} 
\frac{ \Pr_{D_i}\sbr{ x \in \DIS( \B_{D_i}(h^*, r) ) } }{r}
= 
\sup_{r \geq \eps} \frac{\eps}{r}
= 1.
\]
We now calculate $\theta_{\bar{D}}(\frac \eps k)$. Note that $\bar{D}$ has PMF
\[
\bar{D}(x,y)
= 
(1-\eps) I(x = x_0, y = -1)  
+ 
\sum_{i=1}^k \frac{\eps}{k} I(x = x_i, y = +1 ).
\]
For $r \geq \frac \eps k$, $\B_{\bar{D}}(h^*, r) = \Hcal$, and $\DIS(\Hcal) = \cbr{x_1, \ldots, x_k}$. Therefore,
\[
\theta_{\bar{D}}(\frac \eps k)
= 
\sup_{r \geq \frac \eps k} 
\frac{ \Pr_{\bar{D}}\sbr{ x \in \DIS( \B_{\bar{D}}(h^*, r) ) } }{r}
= 
\sup_{r \geq \frac \eps k} \frac{\eps}{r}
= k.
\]
Therefore, for this MDL instance $\Dcal$ and hypothesis class $\Hcal$,
$\theta_{\bar{D}}(\frac \eps k) = k \max_{i \in [k]} \theta_i(\eps)$.
\end{proof}

%such that all $|\Xcal_i| = m$. 

\begin{proof}[Proof of Proposition~\ref{prop:label-cost}]
Denote by $(x,y)$ the random example returned by Algorithm~\ref{alg:sample-imputed}. Fix any $x_0, y_0$:
\[
\Pr\sbr{x = x_0,y = y_0}
= 
\Pr\sbr{x = x_0} \Pr\sbr{ y = y_0 \mid x = x_0 }
\]
Now, according to step~\ref{step:get-x}, $\Pr\sbr{x = x_0} = D_i(x_0)$. For $\Pr\sbr{ y = y_0 \mid x = x_0 }$, it is equal to $D_i(y_0 \mid x_0)$ according to step~\ref{step:dis}, and $I(y_0 = V_{n-1}(x_0))$ according to step~\ref{step:agr}. Therefore, 
\begin{align*}
\Pr\sbr{x = x_0,y = y_0}
= & 
D_i(x_0) \rbr{  D_i(y_0 \mid x_0) I(x_0 \in \DIS(V_{n-1})) + I( y_0 = V_{n-1}(x_0) ) I(x_0 \in \AGE(V_{n-1})) } 
\\
= &
D_{i,n}(x_0, y_0).
\end{align*}
This completes the proof of the first part. For the second part, we note that we make a query to oracle $\Ocal_i$ if $x \in \DIS(V_n)$, which happens with probability 
$\Pr_{D_i}\sbr{ x \in \DIS(V_n) }$.
\end{proof}

\subsection{\PMDL and its Guarantees}

Algorithm~\ref{alg:main-large-eps} uses a passive multi-distribution learning algorithm $\textsc{Passive-MDL}$ as input. 
Therein, we choose $\textsc{Passive-MDL}$ to be the collaborative PAC learning algorithm of~\cite{blum2017collaborative}, which we recall has the following guarantee: 

\begin{lemma}[\cite{blum2017collaborative}, Theorem D.2]
\label{lem:passive-mdl}
Suppose hypothesis class $\Hcal$ and 
distributions $(\mu_1, \ldots, \mu_k)$ are such that 
$$
\min_{h \in \Hcal} \max_{i\in[k]} L(h,\mu_i)
\leq \eta.
$$
In addition, the target error $\zeta \geq 100 \eta$. 
Then, $\textsc{Passive-MDL}(\zeta, \delta)$ satisfies that: 
(1) with probability $1-\delta$, it outputs a classifier $\hat{h}$ such that 
$$
\max_{i\in[k]} L\bra{\hat{h},\mu_i}
\leq \zeta,
$$ 
(2) the total number of times it samples from any of the $D_i$'s is $\widetilde{O} \rbr{ \frac{d + k}{ \zeta }  }$. 
\end{lemma}

\subsection{Proof of Theorem~\ref{thm:largeepsilon}}

\begin{proof}[Proof of Theorem~\ref{thm:largeepsilon}]
Denote by $E_n$ the success event in Lemma~\ref{lem:passive-mdl} when calling \PMDL for iteration $n$; the lemma implies that $\Pr(E_n) \geq 1-\delta_n$. 
Define $E := \cap_{n=1}^{n_0} E_n$.
By a union bound, $\Pr(E) \geq 1-\delta$. We henceforth condition on event $E$ holding.

We will next prove by induction on $n$ that: (1) $h^* \in V_n$; (2) for all $h \in V_n$, $\rho(h, h^*) \leq 4\varepsilon_{n}$. 

\paragraph{Base case.} For $n=0$, $V_0 = \Hcal$.
$h^* \in V_0$ trivially holds, and for all $h \in \Hcal$, $\rho(h, h^*) \leq 1 \leq 4 \eps_0$.

\paragraph{Inductive case.} Suppose the inductive claim holds for iteration $n-1$, specifically $h^* \in V_{n-1}$. 

For round $n$, by the definition of $E_{n}$, 
$h_n$ output by $\PMDL$
is such that for every $i \in [k]$, $L(h_n, D_{i,n}) \leq \eps_n$. 
In addition, for $h^*$, we also have that for every $i \in [k]$,
\[
L(h^*, D_{i,n})
=
\Pr_{D_n}[ h^*(x) \neq y, x \in \AGE(V_{n-1}) ]
\leq 
\nu \leq \eps_n.
\]
% \joey{Why is it less than $\nu$, which is defined wrt to $D_i$, I think it should be $2\nu$.}
% \chicheng{I added an equality above - let me know if this clarifies things}
Therefore, by triangle inequality, for every $i$,
\[
\rho_i( h_n, h^* ) \leq L(h_n, D_{i,n}) + L(h^*, D_{i,n}) \leq 2\eps_n.
\]
Hence, taking the maximum over all $i \in [k]$, we have $\rho(h_n, h^*) \leq 2\eps_n$, implying that $h^* \in V_n$. 
%\chicheng{Need to define this notation}
% \[
% \Pr_{D_{i,n}}( h_n(x) \neq y)
% \leq 
% \eps_n
% \]

% \[
% \Pr_{D_i}( h_n(x) \neq h^*(x) )
% =
% \Pr_{D_i}( h_n(x) \neq h^*(x), x \in \DIS(V_{n-1}) ) 
% +
% \Pr_{D_i}( h^*(x) \neq y, x \notin \DIS(V_{n-1}) ) \leq 2\eps_n.
% \]

% Using the definition of $D_{i,n}$, we have that the above is equivalent to: 
% \begin{equation}
% \Pr_{D_i}( h_n(x) \neq y, x \in \DIS(V_{n-1}) ) 
% +
% \Pr_{D_i}( h_n(x) \neq h^*(x), x \notin \DIS(V_{n-1}) )
% \leq \eps_n,
% \label{eqn:d-i-n-err}
% \end{equation}
% In addition, since $\Pr_{D_i}( 
% h^*(x) \neq y, x \in \DIS(V_{n-1}) ) \leq \nu \leq \eps_n$, by triangle inequality, for all $i$,

% \begin{align*}
%     &\Pr_{D_i}( h_n(x) \neq h^*(x) )\\
% =&
% \Pr_{D_i}( h_n(x) \neq h^*(x), x \in \DIS(V_{n-1}) ) 
% +
% \Pr_{D_i}( h_n(x) \neq h^*(x), x \notin \DIS(V_{n-1}) )\\
% \le&
% \Pr_{D_i}( h^*(x) \neq y, x \in \DIS(V_{n-1}) )
% +
% \Pr_{D_i}( h_n(x) \neq y, x \notin \DIS(V_{n-1}) )
% +
% \Pr_{D_i}( h^*(x) \neq y, x \notin \DIS(V_{n-1}) )\\
% \le&\eps_n + \eps_n = 2\eps_n
% \end{align*}
% where in the last inequality, we bound first two terms using Eq.~\eqref{eqn:d-i-n-err}, and bound the third term by $\eps_n$. \chicheng{Joey, I rewrote this part. Let me know if this is OK.}
% Therefore, $h^* \in V_{n}$. 

Additionally, for all $h \in V_{n}$, $\rho(h, h_n) \leq 2\eps_n$. By triangle inequality, 
$\rho(h, h^*) \leq \rho(h, h_n) + \rho(h^*, h_n) \leq 4 \eps_n$. Together, these show that the inductive claim holds for iteration $n$. This completes the induction.

%We now prove that for every round $n$, for all $h \in V_n$, $\dist(h, h^*) \leq 2\eps_n$.

Applying the claim with $n = n_0$, $\hat{h}$ is such that $\rho(\hat{h}, h^*) \leq 4 \eps_{n_0} \leq \eps$. Therefore, 
\[
L(\hat{h})
\leq 
L(h^*)
+ 
\rho(\hat{h}, h^*) 
\leq \nu + \eps,
\]
establishing the PAC guarantee.

%\chicheng{Not 100\% rigorous since we need to take union bound, etc.}
% can be analyzed as follows

We now analyze the label complexity of Algorithm~\ref{alg:main-large-eps}. 
For each iteration $n$, 
Lemma~\ref{lem:passive-mdl} implies that 
the number of samples to any of $D_{i,n}$ is at most $m_n = \widetilde{O}( \frac{k+d}{\eps_n} )$, which also upper bounds the number of calls to Algorithm~\ref{alg:sample-imputed}. By Proposition~\ref{prop:label-cost} and Chernoff bound (Lemma~\ref{lemma:chernoff}), with probability $1-\delta_n$, the number of label queries $N_n$ at iteration $n$ is at most 
\[
N_n 
\leq  
O\rbr{
m_n \max_i \Pr_{D_i}\sbr{x \in \DIS(V_n)}
+ 
\ln\frac1{\delta_n} }
\leq 
O\rbr{\theta_{\max}(\eps) \cdot 
(k+d)
}.
\]
Summing over all rounds $n \in [n_0]$, the total number of label queries throughout is at most
\[
O\rbr{\theta_{\max}(\eps) 
(k+d) \ln\frac1\eps
}.
\]
\end{proof}

% \section{My Proof of Theorem 2}

% This is a complete version of a proof sketched in the main text.

\section{Deferred Materials for Section~\ref{sec:small-eps}}\label{sec:proofofsmallepsmain}

%\subsection{Sampling Algorithm}\label{appdix:samplingsmalleps}
%The sampling algorithm used in Algorithm~\ref{alg:main-small-eps} is given below.
% \begin{algorithm}[H]
% \caption{Sampling from \(D'_i\)}
% \label{alg:sample-imputed-smalleps}
% \begin{algorithmic}[1]
% \REQUIRE Samples \(S_1,\dots,S_k\) and version space \(V_0\).
% \STATE Use the example oracle \(\EX_i\) to sample \(x\sim D_{i,\Xcal}\).
% \IF{\(x\in\DIS(V_0)\)}
%     \STATE Set \(y\gets \Ocal_i(x)\).
% \ELSE
%     \STATE Uniformly sample \((x,y)\) from \(S_i\).
% \ENDIF
% \RETURN \((x,y)\).
% \end{algorithmic}
% \end{algorithm}

\subsection{Auxiliary Lemmas}
First we show that \(D'_i\) is a valid distribution and that Algorithm~\ref{alg:sample-imputed-smalleps} correctly samples according to \(D'_i\).

\begin{lemma}\label{lemma:samplingcorrectnesssmalleps}
    For every \(i\in[k]\), the surrogate distribution \(D'_i\) is a valid distribution, and Algorithm~\ref{alg:sample-imputed-smalleps} samples according to \(D'_i\).
\end{lemma}
\begin{proof}
    To show that \(D'_i\) is a valid distribution, we must verify that its total mass is 1, i.e., 
    $\int_{\Xcal \times \Ycal} \de D_{i}'(x, y) = 1$. 
    
    %\(\E_{(x,y)\sim D'_i}[1]=1\). 

    % \chicheng{I find this sentence a bit weird.. 
    % If we use the notation $\E_{(x,y)\sim D'_i}$, haven't we already committed to acknowledging $D'_i$ is a probability distribution?
    % So I propose that we verify 
    % $\int_{\Xcal \times \Ycal} D_{i}'(dx, dy) = 1$?
    % }\\
    % \joey{I prefer $\int_{\Xcal \times \Ycal} \de D_{i}'(x, y) = 1$}
    % \chicheng{Ok, then I agree!}
    
    By definition,
    \begin{align*}
    \int_{\Xcal \times \Ycal} \de D_{i}'(x, y)
    = & 
    \int_{\Xcal \times \Ycal} I\rbr{x\in\DIS(V_0)} \de D_i(x,y)+
    \int_{\Xcal \times \Ycal}  \Pr_{x\sim D_{i,\Xcal}}[x\in\AGE(V_0)]\, \de D_{S_i}(x,y)
    \\
    = & 
    \E_{(x,y)\sim D_i}\bigl[I\rbr{x\in\DIS(V_0)}\bigr] + \Pr_{x\sim D_{i,\Xcal}}[x\in\AGE(V_0)]\,\E_{(x,y)\sim D_{S_i}}[1].
    \end{align*}
    Since \(\E_{(x,y)\sim D_{S_i}}[1]=1\) and because \(\Pr_{x\sim D_{i,\Xcal}}[x\in\DIS(V_0)] + \Pr_{x\sim D_{i,\Xcal}}[x\in\AGE(V_0)] = 1\), it follows that \(\int_{\Xcal \times \Ycal} \de D_{i}'(x, y)=1\). Next, observe that in Algorithm~\ref{alg:sample-imputed-smalleps}:
    \begin{itemize}
        \item If \(x\in\DIS(V_0)\), the algorithm queries \(\Ocal_i(x)\) so that the sampled pair \((x,y)\) is drawn with probability \(D_i(x,y)\).
        \item Otherwise, when \(x\in\AGE(V_0)\), the algorithm uniformly samples \((x,y)\) from \(S_i\); hence the probability of obtaining \((x,y)\) is \(\Pr_{x\sim D_{i,\Xcal}}[x\in\AGE(V_0)]\,D_{S_i}(x,y)\).
    \end{itemize}
    This exactly matches the definition of \(D'_i\), so the algorithm correctly samples from \(D'_i\).
\end{proof}
Next, we show that the empirical distribution \(D_{S_i}\) approximates the conditional distribution \(D_{i|\AGE(V_0)}\) well.

\begin{lemma}\label{lemma:smallepsdistribution}
    Assume that $h^*\in V_0$. If we set $n_0 = \frac{100(\varepsilon+\nu)}{\varepsilon^2}\ln\frac{k}{\delta'}$, then with probability at least \(1-\delta'\), for every \(i\in[k]\) and every \(h\in V_0\),
    \[
    \Bigl|L(h,D'_i) - L(h,D_i)\Bigr| \le \frac{\varepsilon}{4}.
    \]
\end{lemma}

\begin{proof}
First notice that by definition, for every $h$, 
\begin{align*}
&\Bigl|L(h,D'_i) - L(h,D_i)\Bigr|\\
=&\abs{\Pr_{x\sim D_{i,\Xcal}}\sbr{x\in\AGE(V_0)}L(h,D_{S_i})-\Pr_{(x,y)\sim D_i}\sbr{x\in\AGE(V_0)\text{ and }h(x)\neq y}}.
\end{align*}
For every $(x_j,y_j)\in S_i$, let $Z_j=\Pr_{x\sim D_{i,\Xcal}}L(h,D_{(x_j,y_j)})$ where $D_{(x_j,y_j)}=I(x=x_j\text{ and }y=y_j)$ is the singleton distribution. Let $Z=\sum_{j=1}^{n_0}Z_j=\Pr_{x\sim D_{i,\Xcal}}\sbr{x\in\AGE(V_0)}L(h,D_{S_i})$ and 
\begin{align*}
\E[Z]=&\Pr_{(x,y)\sim D_i}\sbr{x\in\AGE(V_0)\text{ and }h(x)\neq y}\\
&\Pr_{(x,y)\sim D_i}\sbr{x\in\AGE(V_0)\text{ and }h^*(x)\neq y}\\
&\le\nu,
\end{align*}
where the first equality comes from the definition of $Z$ and the second equality comes from the assumption that $h^*\in V_0$. Then from Bernstein's Inequality (Theorem~\ref{thm:bernstein}), if we draw $n_0 = \frac{100(\varepsilon+\nu)}{\varepsilon^2}\ln\frac{k}{\delta'}$ samples from \(D_{i|\AGE(V_0)}\), then for any $i$, with probability at least \(1-\frac{\delta'}{k}\),
\[
\Bigl|L(h,D'_i) - L(h,D_i)\Bigr| \le \frac{\varepsilon}{4}.
\]
Applying the union bound over all $k$ distributions and the proof finishes.
\end{proof}

\subsection{Proof of Theorem~\ref{thm:smalleps-main-thm}}
\paragraph{Correctness:} Recall that we have verified
the distributions \(\{D'_i\}_{i\in[k]}\) are well-defined and that Algorithm~\ref{alg:sample-imputed-smalleps} correctly samples from \(D'_i\) (Lemma~\ref{lemma:samplingcorrectnesssmalleps}) , and $h^*=\argmin_{h\in\Hcal}\max_{i\in[k]}L\bra{h,D_i}$. By Theorem~\ref{thm:largeepsilon}, with probability at least \(1-\delta/6\), \(V_0\) contains \(h^*\) and the diameter of \(V_0\) with respect metric $\rho$ is at most \(400\nu\). Moreover, Theorem~\ref{theorem:passiveupperbound} guarantees that with probability at least \(1-\delta/6\),
\begin{equation}\label{eq:smallepscorrectness1}
\max_{i\in[k]}L\bra{\hat{h},D'_i}\le \min_{h\in V_0}\max_{i\in[k]}L\bra{h,D'_i}+\frac{\varepsilon}{2}.
\end{equation}
Let \(i_{\max}^{*}=\argmax_{i\in[k]}L\bra{h^*,D'_i}\). 
%\chicheng{Notation $i_{\max}^*$, and $\hat{i}_{\max}$ seems better?}
Then by Lemma~\ref{lemma:smallepsdistribution}, with probability at least \(1-\delta/6\),
\begin{equation}\label{eq:smallepscorrectness2}
\min_{h\in V_0}\max_{i\in[k]}L\bra{h,D'_i}\le \max_{i\in[k]}L\bra{h^*,D'_i}\le L\bra{h^*,D_{i^{*}_{\max}}}+\frac{\varepsilon}{4}\le \nu+\frac{\varepsilon}{4}.
\end{equation}
Similarly, let \(\hat{i}_{\max}=\argmax_{i\in[k]}L\bra{\hat{h},D_i}\) and observe that
\begin{equation}\label{eq:smallepscorrectness3}
\max_{i\in[k]}L\bra{\hat{h},D_i}\le L\bra{\hat{h},D'_{\hat{i}_{\max}}}+\frac{\varepsilon}{4}\le \max_{i\in[k]}L\bra{\hat{h},D'_i}+\frac{\varepsilon}{4}.
\end{equation}
Combining inequalities~(\ref{eq:smallepscorrectness3}), (\ref{eq:smallepscorrectness1}), and (\ref{eq:smallepscorrectness2}) with the union bound yields
\[
L\bra{\hat{h}}\le \nu+\varepsilon,
\]
with probability at least \(1-\delta/2\).

%\edit{For bookkeeping}{}, 
%\edit{use $\theta_{\max}$ to denote $\theta_{\max}\bra{100\nu}$}{
\paragraph{Sample Complexity:}  We abbreviate $\theta_{\max}\bra{100\nu}$ as $\theta_{\max}$. Theorem~\ref{thm:largeepsilon} implies that the stage one of Algorithm~\ref{alg:sample-imputed-smalleps} (step~\ref{step:stage-one}) requires \(\widetilde{O}\bra{\theta_{\max}(d+k)}\) samples with probability at least \(1-\frac{\delta}{6}\). In stage two of Algorithm~\ref{alg:sample-imputed-smalleps} (steps~\ref{step:stage-two-begin} to~\ref{step:stage-two-end}), we query labels only in the disagreement region of \(V_0\) during the execution of \textsc{Passive-MDL}. Since \(V_0\subseteq \B_D(h^*,400\nu)\) 
%\chicheng{In the prelims section we used another $B$ symbol}
and by the definition of \(\theta_{\max}\), the probability that a single sample lands in the disagreement region is at most \(400\theta_{\max}\nu \le 400\theta_{\max} \cdot (\nu+\varepsilon)\). 
%\chicheng{I think $\theta_{\max}(\nu+\varepsilon)$ is ambiguous (here and the next eqn display). 
%Can we use $\theta_{\max} \times (\nu+\varepsilon)$ to be explicit?
%}
By Theorem~\ref{theorem:passiveupperbound}, \textsc{Passive-MDL} samples
\[
O\Bigl(\frac{(d+k)(\nu+\varepsilon)}{\varepsilon^2}\cdot \polylog\Bigl(k, d, \frac{1}{\varepsilon}, \frac{1}{\delta}\Bigr)\Bigr)
\]
examples from any of the $D_i'$'s; applying the Chernoff bound (Lemma~\ref{lemma:chernoff}) gives that, with probability at least \(1-\frac{\delta}{3}\), we query 
\[
O\Bigl(\theta_{\max}(\nu+\varepsilon)\cdot \frac{(d+k)(\nu+\varepsilon)}{\varepsilon^2}\cdot \polylog\Bigl(k, d, \frac{1}{\varepsilon}, \frac{1}{\delta}\Bigr)\Bigr)
\]
fresh labels. Additionally, step~\ref{step:agreement-query} in Algorithm~\ref{alg:main-small-eps} samples \(k\cdot n_0\) points to construct the surrogate distributions, which requires
\[
O\Bigl(\frac{k(\nu+\varepsilon)}{\varepsilon^2}\cdot \polylog\Bigl(k, \frac{1}{\varepsilon}\Bigr)\Bigr)
\]
labels. Taking the sum and applying the union bound, we conclude that, with probability at least \(1-\delta/3\), the overall label complexity is
\[
O\Bigl(\Bigl(\frac{k(\nu+\varepsilon)}{\varepsilon^2}+\frac{\theta_{\max}(d+k)(\nu+\varepsilon)^2}{\varepsilon^2}\Bigr)\cdot \polylog\Bigl(k, d, \frac{1}{\varepsilon}, \frac{1}{\delta}\Bigr)\Bigr).
\]
\paragraph{Conclusion:} Taking a union bound over the above events, with probability at least \(1-\delta\) both the correctness and the label complexity guarantees hold. This completes the proof.

\section{Deferrd Materials for Section~\ref{sec:passivesamplecomplexity}}
\label{sec:proofofpassivetheorem}

% (with the only exception that we use $(\mu_1, \ldots, \mu_k)$ to denote the $k$ distributions to learn from)
\subsection{Overview}
We follow \citet{zhang2024optimal}'s analysis of their \texttt{MDL-Hedge-VC} algorithm and Theorem~1, as well as their notations closely. As mentioned in Section~\ref{sec:passivesamplecomplexity}, we will choose the hyperparameters in line~\ref{step:hyperparameters} differently as in Theorem~\ref{theorem:passiveupperbound}, which is different from \citet{zhang2024optimal}'s original setting.

\texttt{MDL-Hedge-VC} solves passive multi-distribution learning by simulating a two-player zero-sum game similar to~\cite{haghtalab2022demand}, where the row player chooses classifier $h^t \in \Hcal$ and the column player chooses weight $w^t \in \Delta^{k-1}$ over the $k$ distributions. Specifically, at each round $t$, the column player chooses $w^t$ in a no-regret manner, and the row player chooses $h^t$ as an approximate best response. For completeness, we replicate it in Algorithm~\ref{alg:mdl-hedge-vc}.

\begin{algorithm}
\caption{\texttt{MDL-Hedge-VC}~\citep{zhang2024optimal}}
\label{alg:mdl-hedge-vc}

\begin{algorithmic}[1]
\STATE \textbf{Input:} labeled data distributions $(D_1,\ldots, D_k)$, hypothesis class $\Hcal$, target excess error $\eps$, failure probability $\delta$

\STATE \textbf{Hyperparameters:} step size $\eta$, number of rounds $T$, auxiliary excess error level $\eps_1$, auxiliary sample size $T_1$
\label{step:hyperparameters}

\STATE \textbf{Initialization:} weight $W_i^1 = 1$ for all $i \in [k]$, $\hat{w}_i^0 = 0$ and $n_i^0 = 0$, $\Scal = \emptyset$

\FOR{$t=1,\ldots,T$}

\STATE Set $w^t = (w_1^t, \ldots, w_k^t)$, where $w_i^t = \frac{W_i^t}{\sum_j W_j^t}$ 
and $\hat{w}^t = (\hat{w}_1^t, \ldots, \hat{w}_k^t)$, where $\hat{w}_i^t = \hat{w}_i^{t-1}$, $\forall i$
\label{step:hedge-update}

\IF{there exists $j \in [k]$ such that $w_j^t \geq 2\hat{w}_{j}^{t-1}$}

\STATE $\hat{w}_i^t \gets \max( w_i^t,  \hat{w}_i^{t-1} )$, for all $i \in [k]$

\FOR{$i=1,\ldots,k$}

\STATE $n_i \gets \lceil T_1 \hat{w}_i^t \rceil$

\STATE draw $n_i^t - n_i^{t-1}$ independent samples from $D_i$, and add these samples to $\Scal$

\ENDFOR

\ENDIF

\STATE Compute $h^t \gets \argmin_{h \in \Hcal} \hat{L}^t(h, w^t)$, where 
\[
\hat{L}^t(h, w^t) = \sum_{i=1}^k \frac{w_i^t}{n_i^t} \sum_{j=1}^{n_i^t} \ell(h, (x_{i,j}, y_{i,j})),
\]
here, $(x_{i,j}, y_{i,j})$ denotes the $j$-th example from $D_i$'s samples in $\Scal$.

\STATE $\bar{w}_i^t \gets \max_{\tau \in \cbr{1,\ldots,t}} w_i^\tau$, for all $i \in [k]$
\label{step:bar-w-update}

\FOR{$i=1,\ldots,k$}

\STATE Draw from $D_i$ $\lceil k\bar{w}_i^t \rceil$ independent samples $\cbr{(x_{i,j}^t, y_{i,j}^t)}_{j=1}^{\lceil k\bar{w}_i^t \rceil}$, and set 
\[
\hat{r}_i^t = \frac{1}{\lceil k\bar{w}_i^t \rceil} \sum_{j=1}^{\lceil k\bar{w}_i^t \rceil} \ell(h^t, (x_{i,j}^t, y_{i,j}^t))  
\]
\label{step:r_i_t}

\STATE Update weight $W_i^{t+1} \gets W_i^t e^{\eta \hat{r}_i^t}$.
\ENDFOR
\ENDFOR
\RETURN{a hypothesis $h^{\mathrm{final}}$ that predicts by following a hypothesis uniformly at random from $\cbr{h^t}_{t=1}^T$}

\end{algorithmic}
\end{algorithm}

%\chicheng{It'd be nice to say a few words about their high level approach to be self-contained, e.g. ``they solve the problem by simulating the repeated play of a two player zero sum game, where the row player chooses $h^t$ and the column player choose $w^t$, ... '' Otherwise I think readers may get lost.. }

We prove enhanced versions of their three main lemmas, after which the proof of Theorem~\ref{theorem:passiveupperbound} follows identically as in Section 5 of \citet{zhang2024optimal}. The first lemma states that at each iteration $t$, \(h^t\) is roughly the best hypothesis under the distribution \(D_{w^t}\).
\begin{lemma}[Enhanced Lemma 1 from \citet{zhang2024optimal}]\label{lemma:htupperbound}
With probability at least \(1-\frac{\delta}{4}\),
\[
L(h^t,w^t)\le\min_{h\in\Hcal}L(h,w^t)+\varepsilon_1
\]
holds for all \(1\le t\le T\), where \(h^t\) (resp. \(w^t\)) is the hypothesis (resp. weight vector) computed in round \(t\) of Algorithm~\ref{alg:mdl-hedge-vc}, with the modifications stated in Theorem~\ref{theorem:passiveupperbound}.
\end{lemma}
The second key lemma states that, under Lemma~\ref{lemma:htupperbound}, the hypothesis returned by Algorithm 1 in \citet{zhang2024optimal} has good performance, thereby establishing the correctness of the algorithm.

%are replaced with some oracle that 
\begin{lemma}[Enhanced Lemma 2 from \citet{zhang2024optimal}]\label{lemma:hfinaloracle}
Suppose that lines 6--11 in Algorithm~\ref{alg:mdl-hedge-vc} give a hypothesis \(h^t\) satisfying \(L\bigl(h^t, w^t\bigr)\le\min_{h \in \mathcal{H}} L\bigl(h, w^t\bigr)+\varepsilon_1\) in the \(t\)-th round for each \(1 \le t \le T\) and that we choose the hyperparameters as stated in Theorem~\ref{theorem:passiveupperbound}. With probability exceeding \(1 - \tfrac{\delta}{4}\), the hypothesis \(h^{\final}\) output by Algorithm 1 is \(\varepsilon\)-optimal in the sense that
\[
L\bigl(h^{\final}\bigr)\le \nu+\varepsilon.
\]
\end{lemma}
The last key lemma helps bound the \(\ell_1\)-norm of the weight vector \(\overline{w}^t\) updated by \texttt{MDL-Hedge-VC} (step~\ref{step:bar-w-update}). It is crucial in controlling  the algorithm's sample complexity (see step~\ref{step:r_i_t}).
\begin{lemma}[Enhanced Lemma 3 from \citet{zhang2024optimal}]\label{lemma:boundw1}
Assume that lines 6--11 in Algorithm~\ref{alg:mdl-hedge-vc} returns a hypothesis \(h^t\) satisfying \(L(h^t, w^t) \leq \min_{h \in \mathcal{H}} L(h, w^t) + \eps_1\) in the \(t\)-th round for each \(1 \leq t \leq T\). If we choose the hyperparameters as stated in Theorem~\ref{theorem:passiveupperbound}, then with probability at least \(1 - \frac \delta 4\), for all $t$,
\[
\|\overline{w}^t\|_1 \leq O\left(\ln^8 \left(\frac{k}{\delta \eps}\right)\right).
\]
\end{lemma}

\subsection{Proof of Lemma~\ref{lemma:htupperbound}}

%We use the same notation as in \citet{zhang2024optimal} 
%\chicheng{To be self-contained, maybe still good to introduce some notations such as $w_i$, $n_i$ and $T$ here.

%Also Step 1 is pretty long -- I think it would be great to mention at the beginning where we are going with this.
%}
We closely follow \citet{zhang2024optimal}'s proof for their Lemma 14. The major difference is that we enhance their step 1 to Bernstein-style bound. Let's recall that $\ell(h,(x,y))$ is the 0-1 loss of hypothesis $h$ on the feature-label pair $(x,y)$.

\paragraph{Step 1: concentration bounds for any fixed $n=\cbr{n_i}_{i=1}^k$ and $w\in\Delta(k)$.} We first prove the following claim that establishes a fine-grained uniform concentration of weighted empirical losses to their expectations.
Below, we use $\hat{L}(h, w) := \sum_{i=1}^kw_i\frac{1}{n_i}\sum_{j=1}^{n_i}\ell\bra{h,\bra{x_{i,j},y_{i,j}}}$ to denote the empirically importance-weighted 0-1 loss of $h$. 

\begin{claim}
For fixed $n =\cbr{n_i}_{i=1}^k$ such that $\max_i n_i \leq T_1$ and $w \in\Delta(k)$ (where $\Delta(k)$ denotes the $k$-dimensional probability simplex),
%\chicheng{Did we use $\Delta_k$ or $\Delta^{k-1}$ consistently? I think I used the latter..}

\begin{align*}    &\Pr\squr{\max_{h\in\Hcal}\frac{\bra{\hat{L}(h,w)-L(h,w)}^2}{4\rbr{\hat{L}(h,w)+L(h,w)}}\ge \varepsilon'}
\le
\bra{2kT_1+1}^d \exp\rbr{ -\frac{ \epsilon' }{ \underset{i\in[k]}{\max}\frac{w_i}{n_i} } }.
\end{align*}
\end{claim}

Note that the denominator term $\hat{L}(h,w)+L(h,w)$ depends on $h$, which allows loss concentration to be tighter for $h$'s that have smaller losses. This generalizes classical relative VC inequalities~\citep{vapnik1982estimation} to losses that are weighted averages over samples from multiple distributions.

%\chicheng{can we just use notations $S$ and $S^+$?}
\begin{proof}
We denote samples $S \coloneqq\cbr{\bra{x_{i,j},y_{i,j}}}_{i \in [k], j \in [n_i]}$
and denote a set of ``ghost'' samples
$S^+
\coloneqq\cbr{\bra{x_{i,j}^+,y_{i,j}^+}}_{i \in [k], j \in [n_i]}$ independently drawn from the same distribution as $S$. 
For any $h\in\Hcal$, conditioned on samples $\bra{x_{i,j},y_{i,j}}$ and $\bra{x^+_{i,j},y^+_{i,j}}$, the random variable
\[
\eps_{i,j}\bra{\ell\bra{h,\bra{x_{i,j},y_{i,j}}}-\ell\bra{h,\bra{x^+_{i,j},y^+_{i,j}}}},
\]
(where $\eps_{i,j} \sim \mathrm{Uniform}(\cbr{-1,+1})$'s are independent Rademacher random variables), is sub-Gaussian with parameter $\bra{\ell\bra{h,\bra{x_{i,j},y_{i,j}}}+\ell\bra{h,\bra{x^+_{i,j},y^+_{i,j}}}}$~\citep[][Example 2.5.8]{vershynin2018high}. 
Then 
\begin{equation}
\sum_{i=1}^k\frac{w_i}{n_i}\sum_{j=1}^{n_i}\eps_{i,j}\bra{\ell\bra{h,\bra{x_{i,j},y_{i,j}}}-\ell\bra{h,\bra{x^+_{i,j},y^+_{i,j}}}}
\label{eqn:symmetrized}
\end{equation}
is a sub-Gaussian random variable with parameter
\begin{align*}
\sigma^2(h, S, S^+) 
=\sum_{i=1}^k\frac{w_i^2}{n_i}\bra{\hat{\E}_i[\ell]+\hat{\E}_i^+[\ell]},
\end{align*}
where $\hat{\E}_i[\ell]=\frac{1}{n_i}\sum_{j=1}^{n_i}\ell\bra{h,\bra{x_{i,j},y_{i,j}}}$ and $\hat{\E}^+_i[\ell]=\frac{1}{n_i}\sum_{j=1}^{n_i}\ell\bra{h,\bra{x^+_{i,j},y^+_{i,j}}}$~\citep[][Proposition 2.6.1]{vershynin2018high}.

Therefore, \eqref{eqn:symmetrized} also has $\psi_2$-Orlicz norm at most $4 \sigma^2(h, S, S^+)$ for some constant $c > 0$. Applying~\cite[][Lemma 2.7.6]{vershynin2018high}, the random variable
\[
\bra{\sum_{i=1}^k\frac{w_i}{n_i}\sum_{j=1}^{n_i}\eps_{i,j}\bra{\ell\bra{h,\bra{x_{i,j},y_{i,j}}}-\ell\bra{h,\bra{x^+_{i,j},y^+_{i,j}}}}}^2
\]
has $\psi_1$-Orlicz norm at most $4 \sigma^2(h, S, S^+)$. This implies that 

%is 
%\red{$\sigma^2$-subexponential}
%~\citep[][Lemma 2.7.7]{vershynin2018high}. 

%$\bra{\sigma^4,\sigma^2}$-sub-gamma
%\chicheng{I am being nitpicky here.. The cited source is about subexponential, I think? Why don't we use that notation then?
%}
%for any $\lambda\in\bra{0,\frac{1}{\sigma^2}}$

\[ 
\underset{\eps}{\E} \exp\bra{\frac{1}{4 \sigma^2(h, S, S^+)} \bra{\sum_{i=1}^k\frac{w_i}{n_i}\sum_{j=1}^{n_i}\eps_{i,j}\bra{\ell\bra{h,\bra{x_{i,j},y_{i,j}}}-\ell\bra{h,\bra{x^+_{i,j},y^+_{i,j}}}}}^2}
\leq
2
%\lambda \sigma^2 + \frac{\lambda^2 \sigma^4}{1 - \lambda \sigma^2}.
\]
% By reparameterizing, we have for any $\lambda\in\bra{0,1}$, 
% \begin{align*}
% &\log \underset{\eps}{\E} \exp\bra{\frac{\lambda}{\sigma^2} \bra{\sum_{i=1}^k\frac{w_i}{n_i}\sum_{j=1}^{n_i}\eps_{i,j}\bra{\ell\bra{h,\bra{x_{i,j},y_{i,j}}}-\ell\bra{h,\bra{x^+_{i,j},y^+_{i,j}}}}}^2}
% \leq&\lambda + \frac{\lambda^2}{1 - \lambda}.
% \end{align*}
 Let $\Ccal\coloneqq (S, S^+)$. Further, let's define $H_{\min,\Ccal}\subseteq\Hcal$ to be the minimum-cardinality subset of $\Hcal$ that results in the same labeling outcome
as $\Hcal$ when applied to the unlabeled examples of $\Ccal$, namely, $H_{\min,\Ccal}\bra{\Ccal}=\Hcal\bra{\Ccal}$ and $\abs{H_{\min,\Ccal}}=\abs{\Hcal\bra{\Ccal}}$.
Then if we take the maximum over all hypothesis and take expectations over $S, S^+$, we have
\begin{align}
    &\E\max_{h\in\Hcal}\exp\bra{\frac{1}{4 \sigma^2(h, S, S^+)} \bra{\sum_{i=1}^k\frac{w_i}{n_i}\sum_{j=1}^{n_i}\bra{\ell\bra{h,\bra{x_{i,j},y_{i,j}}}-\ell\bra{h,\bra{x^+_{i,j},y^+_{i,j}}}}}^2}\\
    =&\underset{S, S^+}{\E}\underset{\eps}{\E}\sbr{\max_{h\in\Hcal}\exp\bra{\frac{1}{4 \sigma^2(h, S, S^+)} \bra{\sum_{i=1}^k\frac{w_i}{n_i}\sum_{j=1}^{n_i}\eps_{i,j}\bra{\ell\bra{h,\bra{x_{i,j},y_{i,j}}}-\ell\bra{h,\bra{x^+_{i,j},y^+_{i,j}}}}}^2}\middle | \Ccal}\\
    \le&\underset{S, S^+}{\E}\underset{\eps}{\E}\sbr{\sum_{h\in\Hcal_{\min,\Ccal}}\exp\bra{\frac{1}{4 \sigma^2(h, S, S^+)} \bra{\sum_{i=1}^k\frac{w_i}{n_i}\sum_{j=1}^{n_i}\eps_{i,j}\bra{\ell\bra{h,\bra{x_{i,j},y_{i,j}}}-\ell\bra{h,\bra{x^+_{i,j},y^+_{i,j}}}}}^2}\middle | \Ccal}\\
    =&\underset{S, S^+}{\E}
    \sbr{\sum_{h\in\Hcal_{\min,\Ccal}}\underset{\eps}{\E}\sbr{\exp\bra{\frac{1}{4 \sigma^2(h, S, S^+)} \bra{\sum_{i=1}^k\frac{w_i}{n_i}\sum_{j=1}^{n_i}\eps_{i,j}\bra{\ell\bra{h,\bra{x_{i,j},y_{i,j}}}-\ell\bra{h,\bra{x^+_{i,j},y^+_{i,j}}}}}^2}\middle | \Ccal}
    }\\
    \le& 2\cdot \bra{2kT_1+1}^d,
    \label{eqn:symmetrization-bound}
\end{align}
%\exp\bra{\lambda + \frac{\lambda^2}{1 - \lambda}}
where the last step is by Sauer's Lemma, 
$
\abs{\Hcal_{\min,\Ccal}}\le \bra{2kT_1+1}^d
$,
where we use that $\Hcal$ has VC dimensional at most $d$, and $\Ccal$ has at most $2 k T_1$ examples. 

As a result,
\begin{align*}   
&\underset{S}{\E}\max_{h\in\Hcal}\exp\bra{\dfrac{\bra{\hat{L}(h,w)-L(h,w)}^2}{4 \bra{\underset{i\in[k]}{\max}\frac{w_i}{n_i}}\bra{\hat{L}(h,w)+L(h,w)}}}\\
&\underset{S}{\E}\max_{h\in\Hcal}\exp\bra{\dfrac{\bra{\sum_{i=1}^k\frac{w_i}{n_i}\sum_{j=1}^{n_i}\ell\bra{h,\bra{x_{i,j},y_{i,j}}}-L(h,w)}^2}{4 \bra{\underset{i\in[k]}{\max}\frac{w_i}{n_i}}\bra{\sum_{i=1}^k\frac{w_i}{n_i}\sum_{j=1}^{n_i}\ell\bra{h,\bra{x_{i,j},y_{i,j}}}+L(h,w)}}}\\
=&\underset{S}{\E}\max_{h\in\Hcal}\exp\bra{\frac{\bra{\sum_{i=1}^k\frac{w_i}{n_i}\sum_{j=1}^{n_i}\bra{\ell\bra{h,\bra{x_{i,j},y_{i,j}}}-\underset{S^+}{\E}\sbr{\ell\bra{h,\bra{x^+_{i,j},y^+_{i,j}}}}}}^2}{4 \bra{\underset{i\in[k]}{\max}\frac{w_i}{n_i}}\bra{\sum_{i=1}^k\frac{w_i}{n_i}\sum_{j=1}^{n_i}\bra{\ell\bra{h,\bra{x_{i,j},y_{i,j}}}+\underset{S^+}{\E}\sbr{\ell\bra{h,\bra{x^+_{i,j},y^+_{i,j}}}}}}}}\\
\le&\underset{S}{\E}\max_{h\in\Hcal}\exp\underset{S^+}{\E}\bra{\frac{\bra{\sum_{i=1}^k\frac{w_i}{n_i}\sum_{j=1}^{n_i}\bra{\ell\bra{h,\bra{x_{i,j},y_{i,j}}}-\ell\bra{h,\bra{x^+_{i,j},y^+_{i,j}}}}}^2}{4 \bra{\underset{i\in[k]}{\max}\frac{w_i}{n_i}}\bra{\sum_{i=1}^k\frac{w_i}{n_i}\sum_{j=1}^{n_i}\bra{\ell\bra{h,\bra{x_{i,j},y_{i,j}}}+\ell\bra{h,\bra{x^+_{i,j},y^+_{i,j}}}}}}}\\
\le&\underset{S}{\E}\max_{h\in\Hcal}\underset{S^+}{\E}\exp\bra{\frac{\bra{\sum_{i=1}^k\frac{w_i}{n_i}\sum_{j=1}^{n_i}\bra{\ell\bra{h,\bra{x_{i,j},y_{i,j}}}-\ell\bra{h,\bra{x^+_{i,j},y^+_{i,j}}}}}^2}{4 \bra{\underset{i\in[k]}{\max}\frac{w_i}{n_i}}\bra{\sum_{i=1}^k\frac{w_i}{n_i}\sum_{j=1}^{n_i}\bra{\ell\bra{h,\bra{x_{i,j},y_{i,j}}}+\ell\bra{h,\bra{x^+_{i,j},y^+_{i,j}}}}}}}\\
\le&\E\max_{h\in\Hcal}\exp\bra{\frac{\bra{\sum_{i=1}^k\frac{w_i}{n_i}\sum_{j=1}^{n_i}\bra{\ell\bra{h,\bra{x_{i,j},y_{i,j}}}-\ell\bra{h,\bra{x^+_{i,j},y^+_{i,j}}}}}^2}{4 \bra{\underset{i\in[k]}{\max}\frac{w_i}{n_i}}\bra{\sum_{i=1}^k\frac{w_i}{n_i}\sum_{j=1}^{n_i}\bra{\ell\bra{h,\bra{x_{i,j},y_{i,j}}}+\ell\bra{h,\bra{x^+_{i,j},y^+_{i,j}}}}}}}\\
\leq&\underset{S,S^+}{\E}\sbr{\max_{h\in\Hcal}\exp\bra{\frac{\bra{\sum_{i=1}^k\frac{w_i}{n_i}\sum_{j=1}^{n_i}\eps_{i,j}\bra{\ell\bra{h,\bra{x_{i,j},y_{i,j}}}-\ell\bra{h,\bra{x^+_{i,j},y^+_{i,j}}}}}^2}{4 \sigma^2(h, S, S^+)}}}\\
\le& 2\cdot \bra{2kT_1+1}^d.
\end{align*}

%\exp\bra{\lambda + \frac{\lambda^2}{1 - \lambda}}

%\chicheng{I can do to change using the subexponential terminology.}

In the above, for the first and second inequality, we used Jensen's inequality twice ($e^x$ and $\frac{(a-x)^2}{c(b+x)}$ are both convex). Then in the third inequality, we moved the expectation outside of the max function using Jensen's inequality. In the fourth inequality, we used the fact that
\[
\sigma^2(h, S, S^+)=\sum_{i=1}^k\frac{w_i^2}{n_i}\bra{\hat{\E}_i[\ell]+\hat{\E}_i^+[\ell]}\le \bra{\max_{i\in[k]}\frac{w_i}{n_i}}\sum_{i=1}^kw_i\bra{\hat{\E}_i[\ell]+\hat{\E}^+_i[\ell]},
\]
and the last inequality is from Eq.~\eqref{eqn:symmetrization-bound}. 

As a result, by applying Markov's inequality, we have
%\bra{\underset{i\in[k]}{\max}\frac{w_i}{n_i}}
\begin{align*}
    &\Pr\squr{\max_{h\in\Hcal}\frac{\bra{\hat{L}(h,w)-L(h,w)}^2}{4 \bra{\hat{L}(h,w)+L(h,w)}}\ge \varepsilon'}\\
    =&\Pr\squr{\exp\max_{h\in\Hcal}\frac{\bra{\hat{L}(h,w)-L(h,w)}^2}{4 \bra{\underset{i\in[k]}{\max}\frac{w_i}{n_i}}\bra{\hat{L}(h,w)+L(h,w)}}\ge\exp\bra{\frac{\varepsilon'}{\underset{i\in[k]}{\max}\frac{w_i}{n_i}}}}\\
    \le&\Pr\squr{\max_{h\in\Hcal}\exp\rbr{\frac{\bra{\hat{L}(h,w)-L(h,w)}^2}{4 \bra{\underset{i\in[k]}{\max}\frac{w_i}{n_i}}\bra{\hat{L}(h,w)+L(h,w)}}} \ge\exp\bra{\frac{\varepsilon'}{\underset{i\in[k]}{\max}\frac{w_i}{n_i}}}}\\
    \le&2 \bra{2kT_1+1}^d \cdot \exp\rbr{-\frac{\varepsilon'}{\underset{i\in[k]}{\max}\frac{w_i}{n_i}}}.
\end{align*}
\end{proof}

%Simplifying, we get 

%for every $\lambda\in\bra{0,1}$

%\chicheng{I'm not following the simplification here..}

%\min_{0\le\lambda\le1}
%\exp\bra{\lambda + \frac{\lambda^2}{1 - \lambda}-\lambda\varepsilon'}

% \chicheng{I thought that the denominator 
% $
% \bra{\underset{i\in[k]}{\max}\frac{w_i}{n_i}}\bra{\sum_{i=1}^kw_i\frac{1}{n_i}\sum_{j=1}^{n_i}\ell\bra{h,\bra{x_{i,j},y_{i,j}}}+L(h,w)}
% $ is important and we cannot drop it.
% Discuss it in the meeting.
% }

\paragraph{Step 2: uniform concentration bounds over epsilon-nets w.r.t. $n$ and $w$.} 
First we recall some notations in \citet{zhang2024optimal} as below.
\begin{itemize}
    \item We use $\Delta_{\varepsilon_2}(k) \subseteq \Delta(k)$ to denote an $\varepsilon_2$-net of the probability simplex $\Delta(k)$
    %\chicheng{we previously used $\Delta_k$ notation}
    — namely, for any $x \in \Delta(k)$, there exists a vector $x_0 \in \Delta_{\varepsilon_2}(k)$ obeying $\|x - x_0\|_{\infty} \leq \varepsilon_2$. We shall choose $\Delta_{\varepsilon_2}(k)$ properly so that
    \[
    |\Delta_{\varepsilon_2}(k)| \leq (1/\varepsilon_2)^k.
    \]
    
    %\chicheng{Is there a reference on this result?}\\
    %\joey{I think this result is quite standard, at least when we don't care about constant. I believe I just copied this from their paper.}
    
    \item Define the following set
    \[
    \mathcal{B} = \left\{ n = \{n_i\}_{i=1}^{k}, w = \{w_i\}_{i=1}^{k} \ \middle|\ \frac{n_i}{w_i} \geq \frac{T_1}{2},\  n_i \in [0, T_1] \cap \mathbb{N}, \forall i \in [k], w \in \Delta_{\varepsilon_1/(16k)}(k) \right\},
    \]
    %\edit{which clearly satisfies}{
    by the constraints on $n_i$'s and the definition of $\Delta_{\varepsilon_1/(16k)}(k)$,
    \[
    |\mathcal{B}| \leq T_1^k \cdot \left( \frac{16 k}{\varepsilon_1} \right)^k.
    \]
\end{itemize}
It is clear from the definition of $\Bcal$ that any $(n,w)$ in $\Bcal$ satisfies that
\[
\max_{i\in[k]}\frac{w_i}{n_i}\le\frac{2}{T_1}.
\]
Then by taking a union bound over all $(n,w)$'s in $\Bcal$, we get 
%by picking $\varepsilon''=\frac{\varepsilon'}{\max_{i\in[k]}\frac{w_i}{n_i}}$, we get
\begin{align*}    &\Pr\sbr{\exists(n,w)\in\Bcal,\max_{h\in\Hcal}\frac{\bra{\hat{L}(h,w)-L(h,w)}^2}{4 \rbr{\hat{L}(h,w)+L(h,w)}}\ge\varepsilon'}\\
\le&4\bra{16 kT_1/\varepsilon_1}^k\bra{2kT_1+1}^d\exp\bra{-\frac{T_1\varepsilon'}{4}}.
\end{align*}

%and setting $\lambda=\frac{1}{2}$
% \chicheng{I have a preference on using $\hat{L}(h,w)$ and $L(h,w)$ notations since these are also used below. 
% (6/6): I will make a pass and adopt this notation
% }
%Let $A(h)=\hat{L}(h,w)$ and $B(h)=L(h,w)$, 
%\edit{then
%by definition,}{; with this notation,}
Hence,
\begin{align*}
&\Pr\sbr{\exists(n,w)\in\Bcal,\max_{h\in\Hcal}\frac{\bra{\hat{L}(h,w)-L(h,w)}^2}{4(\hat{L}(h,w)+L(h,w))}\ge\varepsilon'}\\
=&\Pr\sbr{\exists(n,w)\in\Bcal,h\in\Hcal\text{ such that }\abs{\hat{L}(h,w)-L(h,w)}\ge\sqrt{4 \varepsilon'\bra{\hat{L}(h,w)+L(h,w)}}} 
\\
\geq &\Pr\sbr{\exists(n,w)\in\Bcal,h\in\Hcal\text{ such that }\abs{\hat{L}(h,w)-L(h,w)}\ge2\bra{4\varepsilon'+\sqrt{4 \varepsilon'\min\cbr{\hat{L}(h,w),L(h,w)}}}},
\end{align*}
where the last inequality is from Lemma~\ref{lemma:quadraticinequality}.

% \begin{align*}    &\Pr\sbr{\exists(n,w)\in\Bcal,h\in\Hcal\text{ such that }\abs{\hat{L}(h,w)-L(h,w)}\ge2\bra{4\varepsilon'+\sqrt{4 \varepsilon'\min\cbr{\hat{L}(h,w),L(h,w)}}}}\\
% \le&\Pr\sbr{\exists(n,w)\in\Bcal,h\in\Hcal\text{ such that }\abs{\hat{L}(h,w)-L(h,w)}\ge\sqrt{\varepsilon'\bra{\hat{L}(h,w)+L(h,w)}}}.
% \end{align*}

Putting it together, we have
\begin{align*}   &\Pr\sbr{\exists(n,w)\in\Bcal,h\in\Hcal\text{ such that }\abs{\hat{L}(h,w)-L(h,w)}\ge2\bra{4\varepsilon'+\sqrt{4\varepsilon'\min\cbr{\hat{L}(h,w),L(h,w)}}}}\\
\le&4\bra{16 kT_1/\varepsilon_1}^k\bra{2kT_1+1}^d\exp\bra{-\frac{T_1\varepsilon'}{4}}.
\end{align*}

% \chicheng{One step seems missing here: what about the discretization effect of $\Bcal$?}

% \chicheng{$L(h,w)$ seems not always smaller than $\nu$}

% \chicheng{Keep the $L(h,w)$ term in this step}
%\chicheng{May be good to recall what  $\varepsilon_1$ is here.}

\paragraph{Step 3: concentration bounds w.r.t. $n^t$ and $w^t$.} By the definition of $\Bcal$, one can always find \( \rbr{n^t, \tilde{w}^t} \in \Bcal \) satisfying
\[
\|w^t - \tilde{w}^t\|_1 \le k\|w^t - \tilde{w}^t\|_\infty \le \frac{\eps_1}{16}.
\]
As a result, $\abs{ L(h, w^t) - L(h, \tilde{w}^t) } \leq \frac{\varepsilon_1}{16}$ and $\abs{ \hat{L}(h, w^t) - \hat{L}(h, \tilde{w}^t) } \leq \frac{\varepsilon_1}{16}$, where $\hat{L}(h, w^t) := \sum_{i=1}^kw^t_i\frac{1}{n^t_i}\sum_{j=1}^{n^t_i}\ell\bra{h,\bra{x_{i,j},y_{i,j}}}$, and $\hat{L}(h, \tilde{w}^t) := \sum_{i=1}^k \tilde{w}^t_i\frac{1}{n^t_i}\sum_{j=1}^{n^t_i}\ell\bra{h,\bra{x_{i,j},y_{i,j}}}$.

%\chicheng{I think what's going on is:
%$\abs{ \sum_{i=1}^kw^t_i\frac{1}{n^t_i}\sum_{j=1}^{n^t_i}\ell\bra{h,\bra{x_{i,j},y_{i,j}}} - \sum_{i=1}^k \tilde{w}^t_i\frac{1}{n^t_i}\sum_{j=1}^{n^t_i}\ell\bra{h,\bra{x_{i,j},y_{i,j}}} } \leq \frac{\varepsilon_1}{8}$
%and $\abs{ L(h, w^t) - L(h, \tilde{w}^t) } \leq \frac{\varepsilon_1}{8}$, 

% \chicheng{I'm not following this --
% I thought  but I may be wrong.
% }
%\|w^t - \tilde{w}^t\|_1 \le k\|w_r^t - w^t\|_\infty \le \frac{\eps_1}{8}.

%\chicheng{I think this is only used in the next step?}
%and substituting the value of $T_1$ as defined in Lemma~\ref{lemma:htupperbound}, 

Therefore, for any $\eps' \leq \frac{\eps_1}{64}$, 
\begin{align*}
& \Pr\sbr{\exists t\in[T],
h \in \Hcal, 
\abs{\hat{L}(h, w^t)-L(h,w^t)}
\geq 2\rbr{\varepsilon'+\sqrt{\varepsilon'\min\cbr{\hat{L}(h,w^t),L(h,w^t)}}}+\frac{\eps_1}{4}
}\\
\leq & 
\Pr\sbr{\exists t\in[T],
h \in \Hcal, 
\abs{\hat{L}(h, \tilde{w}^t)-L(h,\tilde{w}^t)}
\geq 2\rbr{\varepsilon'+\sqrt{\varepsilon'\min\cbr{\hat{L}(h,\tilde{w}^t),L(h, \tilde{w}^t)}}}
} \\
\leq & 
\Pr\sbr{\exists t\in[T], (n,w) \in \Bcal, 
h \in \Hcal, 
\abs{\hat{L}(h, w)-L(h,w)}
\geq 2\rbr{\varepsilon'+\sqrt{\varepsilon'\min\cbr{\hat{L}(h,w),L(h,w)}}}
} \\
\leq & 
4 T \bra{16 kT_1/\varepsilon_1}^k\bra{2kT_1+1}^d\exp\bra{-\frac{T_1\varepsilon'}{4}},
\end{align*}
where the first inequality also uses that 
$\sqrt{\varepsilon'\min\cbr{\hat{L}(h,\tilde{w}^t),L(h, \tilde{w}^t)}} -
\sqrt{\varepsilon'\min\cbr{\hat{L}(h, w^t),L(h, w^t)}} \leq \sqrt{ \varepsilon' \cdot \frac{\eps_1}{8} } \leq \frac{\eps_1}{8}$, and the second inequality is because $\rbr{n^t, \tilde{w}^t} \in \Bcal$, and last inequality is from Step 2.

%which is smaller than below:

%\begin{align*}
%     &\Pr\sbr{\exists t\in[T],\max_{h\in\Hcal}\frac{\abs{\sum_{i=1}^kw^t_i\frac{1}{n^t_i}\sum_{j=1}^{n^t_i}\ell\bra{h,\bra{x_{i,j},y_{i,j}}}-L(h,w^t)}}{2\rbr{\varepsilon'+\sqrt{\varepsilon'\min\cbr{\hat{L}(h,w),L(h,w)}}}+\frac{\eps_1}{4}}\ge1}\\
%     \le&4\bra{8kT_1/\varepsilon_1}^k\bra{2kT_1+1}^d\exp\bra{-10\bra{k\log\bra{k/\varepsilon_1}+d\log\bra{kd/\varepsilon_1}+\log\bra{1/\delta}}}\\
%     \le&\frac{\delta}{2}.
% \end{align*}
% Let $\hat{L}(h,w)=\sum_{i=1}^kw^t_i\frac{1}{n^t_i}\sum_{j=1}^{n^t_i}\ell\bra{h,\bra{x_{i,j},y_{i,j}}}, 
% L(h,w^t)=A(h)$. \chicheng{double check the sentence above.}

\paragraph{Step 4: putting things together.} We take $\varepsilon'=\frac{\eps^2_1}{64\bra{\varepsilon_1+\nu}}\le\frac{\eps_1}{64}$, and apply Step 3 with $T_1 = 4000\rbr{\frac{1}{\varepsilon_1}+\frac{\nu}{\varepsilon_1^2}}\rbr{k\ln\rbr{\frac{k}{\varepsilon}}+d\ln\rbr{\frac{kd}{\varepsilon}}+\ln\rbr{\frac{1}{\delta}}}$. This implies that that failure probability in Step 3 is at most 
\[
4 T \bra{8kT_1/\varepsilon_1}^k\bra{2kT_1+1}^d\exp\bra{-10\bra{k\ln \bra{k/\varepsilon_1}+d\ln \bra{kd/\varepsilon_1}+\ln \bra{1/\delta}}} \le \frac{\delta}{2}.
\]

%\chicheng{Also, I think $\hat{L}(h,w)$ is a great notation.
%Can we define it at the beginning of this section? We may also define $\hat{L}^+(h,w)$. }

To summarize, we showed that, with probability $1-\delta/2$,
\begin{equation}
\abs{\hat{L}(h,w^t)-L(h,w^t)}\le2\bra{\varepsilon'+\sqrt{\varepsilon'\min\cbr{\hat{L}(h,w),L(h,w)}}}+\frac{\eps_1}{4}
\label{eqn:conc-clean}
\end{equation}
holds simultaneously for all $t\in[T]$ and all $h\in\Hcal$. Let $h'=\argmin_{h\in\Hcal} L(h,w^t)$; by its definition, $L(h', w^t) \leq \min_{h \in \Hcal} \max_{i \in [k]} L(h, D_i) = \nu$. 
Therefore, repeatedly applying Eq.~\eqref{eqn:conc-clean}, we get:
%\chicheng{Let's change  $\epsilon''$ to $\epsilon'$}
\begin{align*}
L(h^t,w^t)&\le\hat{L}(h^t,w^t)+2\eps'+2\sqrt{\eps'\hat{L}(h^t, w^t)}+\frac{\eps_1}{4}\\
&\le\hat{L}(h',w^t)+2\eps'+2\sqrt{\eps'\hat{L}(h', w^t)}\\
&\le \hat{L}(h',w^t)+2\eps'+2\sqrt{\eps'\rbr{L(h', w^t)+2\bra{\eps'+\sqrt{\eps'L(h', w^t)}}+\frac{\eps_1}{4}}}\\
%&\le \hat{L}(h',w^t)+2\eps'+2\sqrt{\eps'\rbr{B(h')+2\bra{\eps'+\sqrt{\eps'B(h')}}+\frac{\eps_1}{4}}}\\
&\le L(h',w^t)+4\eps'+2\sqrt{\eps'\rbr{L(h', w^t)+2\bra{\eps'+\sqrt{\eps'L(h', w^t)}}+\frac{\eps_1}{4}}}+2\sqrt{\eps'L(h', w^t)}+\frac{\eps_1}{4}\\
&=\min_{h\in\Hcal}L(h,w^t)+4\eps'+2\sqrt{\eps'\rbr{L(h', w^t)+2\bra{\eps'+\sqrt{\eps'L(h', w^t)}+\frac{\eps_1}{4}}}}+2\sqrt{\eps'L(h', w^t)}+\frac{\eps_1}{4}.
\end{align*}
Recall that \(L(h', w^t) \le \nu\), we have
\[
\sqrt{\eps'L(h', w^t)} \le \frac{\eps_1}{8}.
\]
Combining this with the previous bounds yields
\[
L(h^t, w^t) \le \min_{h \in \Hcal} L(h, w^t) + \eps_1.
\]

\subsection{Proof of Lemma \ref{lemma:hfinaloracle}}
We replace the usage of Azuma-Hoeffding in the proof of Lemma 2 in \citet{zhang2024optimal} with Freedman's inequality, resulting in a stronger bound. Let $\Fcal_i$ denote all of the information up to iteration $i$. Then for any $i\in[k]$, let 
\[
X_i^j=\sum_{t=1}^j\bra{\hat{r}_i^t-L_i\bra{h^t}},
\]
where $\hat{r}_i^t$ is defined in step~\ref{step:r_i_t} in Algorithm~\ref{alg:mdl-hedge-vc}. Then the sequence $\cbr{X^t_i}_{t=1}^T$ is a martingale because
\[
\E\sbr{X^{j+1}_i-X^j_i\middle | \Fcal_{j}}=\E_j\sbr{\hat{r}^j_i-L_i\bra{h^j}}=0,
\]
where $\E_j$ denote the conditional expectation on $\Fcal_j$. The last equality is because $\hat{r}^j_i$ is an unbiased estimator of $L_i\bra{h^j}$ conditioned on $\Fcal_j$. Let $\bra{\sigma_i^j}^2$ denote the conditional variance of the martingale difference $X_i^j-X_i^{j-1}$. Then
\[
\bra{\sigma_i^j}^2\le\E_j\sbr{\bra{\hat{r}^j_i-L_i\bra{h^j}}^2}
\leq
\E_j\sbr{(\hat{r}^j_i)^2}
\le\E_j\sbr{\hat{r}^j_i}
=
L_i\bra{h^j},
\]

%\chicheng{the second inequality does not check for me.. 
%How about
%$\E_j\sbr{\bra{\hat{r}^j_i-L_i\bra{h^j}}^2}
%\leq 
%\E_j\sbr{\bra{\hat{r}^j_i}^2}
%\leq \E_j\sbr{\hat{r}^j_i}
%$
%}

so
\[
\sum_{t=1}^T\bra{\sigma_i^t}^2\le \sum_{t=1}^T L_i\bra{h^t}.
\]
Let $\delta'\coloneqq\frac{\delta}{4(T+k+1) \log T}$, then by Freedman's inequality (Theorem~\ref{thm:freedman-easytouse}) and taking the union bound over all $i$, we see that with probability at least $1- (2 k \log T) 
 \delta'$, for all $i$, 
\begin{equation}\label{eq:lemma2eq1}
\abs{\sum_{t=1}^T\hat{r}^t_i-\sum_{t=1}^TL_i\bra{h^t}}\le 4\sqrt{\ln(1/\delta')\sum_{t=1}^TL_i\bra{h^t}}
+
2\ln\bra{1/\delta'}.
\end{equation}
Similarly, $Y^j=\sum_{t=1}^j\bra{\inner{w^t}{\hat{r}^t}-L\bra{h^t,w^t}}$ is also a martingale. Moreover, the sum of its conditional variance is also upper bounded by $\sum_{t=1}^T L\bra{h^t,w^t}$ by the same argument. Then again by applying Freedman's inequality, we get with probability at least $1-
2 \log T \cdot \delta'$, 
\begin{equation}\label{eq:lemma2eq2}
\abs{\sum_{t=1}^T\inner{\hat{r}^t_i}{w^t}-\sum_{t=1}^TL\bra{h^t,w^t}}\le 4\sqrt{\ln(1/\delta')\sum_{t=1}^TL\bra{h^t,w^t}}
+
2\ln \bra{1/\delta'}.
\end{equation}
Taking the union bound, with probability at least $1-(2k+2) \log T \cdot \delta'$, equations~(\ref{eq:lemma2eq1}) and (\ref{eq:lemma2eq2}) both hold. We then resort to standard analysis for the Hedge algorithm~\citep{freund1997decision}. We recall that $W^t$ is the unnormalized weight vector at iteration $t$. Direct calculation gives
\begin{align}
\label{eq:hedgeeq1}
\ln \left(\frac{\sum_{i=1}^k W_i^{t+1}}{\sum_{i=1}^k W_i^t}\right)
&\overset{\text{(i)}}{=} \ln \left(\sum_{i=1}^k w_i^t \exp(\eta \hat{r}_i^t)\right)\nonumber\\
&\overset{\text{(ii)}}{\leq} \ln\left(\sum_{i=1}^k w_i^t \left(1 + \eta \hat{r}_i^t + \eta^2 (\hat{r}_i^t)^2\right)\right)\nonumber\\
&\leq \ln\left(1 + \eta \sum_{i=1}^k w_i^t \hat{r}_i^t + \eta^2 \sum_{i=1}^k w_i^t (\hat{r}_i^t)^2\right)\nonumber\\
&\leq \eta \sum_{i=1}^k w_i^t \hat{r}_i^t + \eta^2\inner{w^t}{\hat{r}^t}.
\end{align}
Here, (i) is valid since $w_i^t = \frac{W_i^t}{\sum_j W_j^t}$ and $W_i^{t+1} = W_i^t \exp(\eta \hat{r}_i^t)$ (cf. lines 5 and 15 of Algorithm 1); (ii) arises from the elementary inequality $e^x \leq 1 + x + x^2$ for $x \in [0,1]$ as well as the facts that $\eta \leq 1$ and $\lvert \hat{r}_i^t \rvert \leq 1$. Summing the inequality (\ref{eq:hedgeeq1}) over all $t$ and rearranging terms, we are left with
\begin{align*}
\eta \sum_{t=1}^T \langle w^t, \hat{r}^t \rangle &\geq \sum_{t=1}^T \left\{\ln\left(\frac{\sum_{i=1}^k W_i^{t+1}}{\sum_{i=1}^k W_i^t}\right) - \eta^2\inner{w^t}{\hat{r}^t} \right\}\\
&= \ln\left(\frac{\sum_{i=1}^k W_i^{T+1}}{\sum_{i=1}^k W_i^1}\right) - \eta^2\sum_{t=1}^T\inner{w^t}{\hat{r}^t}\\
&= \ln\left(\sum_{i=1}^k W_i^{T+1}\right) - \ln\left(\sum_{i=1}^k W_i^1\right) - \eta^2\sum_{t=1}^T\inner{w^t}{\hat{r}^t}\\
&\geq \max_{1 \leq i \leq k} \ln(W_i^{T+1}) - \ln(k) - \eta^2\sum_{t=1}^T\inner{w^t}{\hat{r}^t}\\
&= \eta \max_{1 \leq i \leq k} \sum_{t=1}^T \hat{r}_i^t - \ln(k) - \eta^2\sum_{t=1}^T\inner{w^t}{\hat{r}^t},
\end{align*}
where the penultimate line makes use of $W_i^1 = 1$ for all $i \in [k]$, and the last line holds since $\ln(W_i^{T+1} \exp(\eta \hat{r}_i^t)) \geq \eta \hat{r}_i^t$. Dividing both sides by $\eta$ yields
\[
\sum_{t=1}^T \langle w^t, \hat{r}^t \rangle \geq \max_{i\in[k]} \sum_{t=1}^T \hat{r}_i^t - \frac{\ln(k)}{\eta} - \eta\sum_{t=1}^T\inner{w^t}{\hat{r}^t}.
\]
Combining the above inequality with (\ref{eq:lemma2eq1}) and (\ref{eq:lemma2eq2}), we have shown that with probability at least $1-(2k+2) \log T \cdot \delta'$,
\begin{align*}
&\sum_{t=1}^TL(h^t,w^t)\\
\ge&\max_{i\in[k]}\sum_{t=1}^TL_i(h^t)\\
-&\bra{\frac{\ln(k)}{\eta}+\eta\sum_{t=1}^TL(h^t,w^t)+4\sqrt{\ln(1/\delta')\max_{i\in[k]}\sum_{t=1}^TL_i\bra{h^t}}+5\sqrt{\ln(1/\delta')\sum_{t=1}^TL\bra{h^t,w^t}}+5\ln(1/\delta')}.
\end{align*}
Applying Lemma~\ref{lemma:quadraticinequalityvariant} with $A=\max_{i\in[k]}\sum_{t=1}^TL_i(h^t)$, $B=5\sqrt{2\ln\rbr{1/\delta'}}$, $C=\sum_{t=1}^TL(h^t,w^t)$ and $D=\frac{\ln(k)}{\eta}+\eta\sum_{t=1}^TL(h^t,w^t)+5\ln\rbr{1/\delta'}$ so we get
\begin{align*}
\max_{i\in[k]}\sum_{t=1}^TL_i(h^t)&\le \sum_{t=1}^TL(h^t,w^t)+50 \ln\rbr{1/\delta'}+\frac{3}{2}\rbr{\frac{\ln(k)}{\eta}+\eta\sum_{t=1}^TL(h^t,w^t)+5\ln\rbr{1/\delta'}}\\
&+5\sqrt{2\ln\rbr{1/\delta'}\sum_{t=1}^TL(h^t,w^t)}.
\end{align*}
From the assumption of Lemma~\ref{lemma:hfinaloracle}, $\sum_{t=1}^T L(h^t,w^t)\le T\rbr{\nu+\eps_1}$. As a result, we have
\begin{align}   \max_{i\in[k]}\sum_{t=1}^TL_i(h^t)
\le & \sum_{t=1}^T L(h^t,w^t)+\bra{\frac{3\ln(k)}{2\eta}+\frac{3}{2}\eta T(\nu+\varepsilon_1)+58\ln(1/\delta')+5\sqrt{2\ln(1/\delta')T(\nu+\varepsilon_1)}}
\label{eqn:regret-form}
\\
\le & 
T\rbr{\nu+\eps_1} +\bra{\frac{3\ln(k)}{2\eta}+\frac{3}{2}\eta T(\nu+\varepsilon_1)+58\ln(1/\delta')+5\sqrt{2\ln(1/\delta')T(\nu+\varepsilon_1)}}. 
\end{align}
By the definition of $h^{\final}$ (average over all $T$ rounds) and substitute in $\varepsilon_1=\frac{1}{100}\varepsilon$, $T=\frac{2000\bra{\varepsilon_1+\nu}\ln\bra{\frac{k}{\delta\varepsilon}}}{\varepsilon_1^2}$ and $\eta=\frac{\varepsilon_1}{100(\varepsilon_1+\nu)}$, we have that each of the four terms in the parenthesis above is at most $\eps / 5$, and thus, 
\begin{align*}   \max_{i\in[k]}L_i(h^{\final})=\max_{i\in[k]}\frac{1}{T}\sum_{t=1}^TL_i(h^t)\le \nu+\varepsilon
\end{align*}
with probability at least $1-(2k+2)  \log T \cdot \delta'$. Since $\delta'=\frac{\delta}{4(T+k+1) \log T}$, the proof finishes.

\subsection{Proof of Lemma~\ref{lemma:boundw1}}
Again, our proof closely follows that of Lemma 3 in \citet{zhang2024optimal}, with the only relevant change being the new step size \(\eta=\frac{\varepsilon_1}{100\bra{\varepsilon_1+\nu}}\). The entire argument is based on Lemma~13 in \cite{zhang2024optimal}, and since this lemma holds as long as $\eta \leq \frac 1 {20}$, the derivation from Lemma~13 in \cite{zhang2024optimal} to Lemma~\ref{lemma:boundw1} is identical to that presented in Section B.3 of \citet{zhang2024optimal}. Therefore, our main focus will be on proving Lemma~13 in \cite{zhang2024optimal}, which we restate below.

%Suppose that the assumptions of Lemma 3 in \citet{zhang2024optimal} hold. 
\begin{lemma}[Lemma 13 from \citet{zhang2024optimal}]\label{lemma:lemma13}
Assume that lines 6--11 in Algorithm~\ref{alg:mdl-hedge-vc} returns a hypothesis \(h^t\) satisfying \(L(h^t, w^t) \leq \min_{h \in \mathcal{H}} L(h, w^t) + \eps_1\) in the \(t\)-th round for each \(1 \leq t \leq T\). 
Then with probability exceeding $1 - 8T^4 k \delta'$, 
\begin{equation}
|\mathcal{W}_j| \leq 8 \cdot 10^7 \cdot \left( (\log \rbr{\frac 1 {\eta(\nu+\epsilon_1)}} + 1)^2 \log^2(k) (\ln(k) + \ln(1/\delta'))^3 (\log(T) + 1) \right) \cdot 2^j
\end{equation}
holds for all $1 \leq j \leq \overline{j}$, with  $\overline{j}$ and $\Wcal_j$  defined in (67) and (68a) in \citet{zhang2024optimal}.
\end{lemma}

We derive Lemma~\ref{lemma:lemma13}  from Lemmas~14, 15, 16 and 17 in \citet{zhang2024optimal}. An examination of the proof of their Lemma 14 shows that it still holds as long as $\eta \leq \frac 1 {20}$, which is satisfied by our choice of $\eta$.
Their Lemmas 15 and 16 continue to hold as they are independent of the choice of $\eta$. Therefore, first we enhance Lemma~17 in \citet{zhang2024optimal} as below, where the major differences are: first, 
our enhanced version has an extra multiplicative factor of $\frac{1}{\nu+\varepsilon}$ on the right hand side of Eq.~\eqref{eqn:enhanced-length-lb}; second, we change the definition of $v^t$ from
$L(h^t, w^t) - \nu$ to 
$L(h^t, w^t) - \bar{\nu}$, where \[
\bar{\nu} = \min_{p \in \Delta(\Hcal)} \max_{i} \E_{h \sim p} L(h, D_i), 
\]
and $\Delta(\Hcal)$ is the set of probability distributions over $\Hcal$.
Note that $\bar{\nu} \leq \nu$  always holds, while they are not equal in general.

\begin{lemma}[Enhanced Lemma 17 from \citet{zhang2024optimal}]\label{lemma:segmentlb}
Let $j_{\max} = \lfloor \log(1/\eta(\nu+\varepsilon_1)) \rfloor + 1$. Assume that lines 6--11 in Algorithm~\ref{alg:mdl-hedge-vc} returns a hypothesis \(h^t\) satisfying \(L(h^t, w^t) \leq \min_{h \in \mathcal{H}} L(h, w^t) + \eps_1\) in the \(t\)-th round for each \(1 \leq t \leq T\). Suppose $(t_1, t_2)$ is a $(p, q, x)$-segment~\citep[Definition 1]{zhang2024optimal} satisfying $p \geq 2q > 0$. Then,
\begin{equation}
t_2 - t_1 \geq \frac{x}{2\eta}.
\label{eqn:length-lb-naive}
\end{equation}
Moreover, if
\[
\frac{qx^2}{50(\log(1/\eta(\nu+\eps_1)) + 1)^2} \geq \frac{1}{k},
\]
holds, then with probability exceeding $1 - 6T^4 k\delta'$, at least one of the following two claims holds:
%\chicheng{Am I right that the $\eps$'s below are all meant to be $\eps_1$?}
\begin{enumerate}
    \item The length of the segment satisfies
    \begin{equation}
    t_2 - t_1 \geq \frac{qx^2}{200(\log(1/\eta(\nu+\eps_1)) + 1)^2 \eta^2(\nu+\varepsilon_1)}.
    \label{eqn:enhanced-length-lb}
    \end{equation}
    \item The quantities $\{v^t = L\bra{h^t,w^t}-\bar{\nu}\}$ obey
    \[
    4 \sum_{\tau = t_1}^{t_2 - 1} (-v^\tau + \varepsilon_1) \geq \frac{qx^2}{100 (\log(1/\eta(\nu+\eps_1)) + 1)^2 \eta}.
    \]
\end{enumerate}
\end{lemma}
\begin{proof}
Eq.~\eqref{eqn:length-lb-naive} follows an indentical proof as~\citet[][ Section C.4, Part 1]{zhang2024optimal}.

For the second claim, we would like to improve Eq. (114) in Step 4 of~\citet[][Section C.4, Part 2]{zhang2024optimal} to the following form with probability $1-\delta'$:

%\edit{, with logarithmic factors omitted}{: }

%\chicheng{$\KL$ or $\DKL$ throughout?}
\begin{equation}
\DKL( w^t \parallel w^{t_2} )
\leq 
8(t_2 - t) \eta \varepsilon_1 - 4(t_2 - t) \eta v^t + 2\eta \sqrt{ \frac{ (t_2 - t) \nu \ln\frac2{\delta'} }{k} } 
 + 68 \eta \ln\frac2{\delta'}.
\label{eqn:refinement}
\end{equation}

%\edit{+ \varepsilon}{}
%\edit{+ 2(t_2 - t) \eta^2 \nu}{}
%\chicheng{I removed the $2(t_2 - t) \eta^2 \nu$ term since I think it can be absorbed into the first one. Please check!}

%To show this, we will refine the analysis of Part 2 of Lemma 17 in~\cite{zhang2024optimal}.

\paragraph{Bounding the KL Divergence.} First, in their Step 1, let us avoid using inequalities and write down the following equation: 
\begin{align*}
\DKL( w^t \parallel w^{t_2} )
= \inner{w^t}{\ln\frac{w^t}{w^{t_2}}}=\inner{w^t}{\ln\frac{\frac{W^t}{Z^t}}{\frac{W^{t_2}}{Z^{t_2}}}}=
\ln\frac{Z^{t_2}}{Z^t}
- 
\eta \sum_{\tau = t}^{t_2 - 1} \inner{  \hat{r}^\tau }{ w^t },
\end{align*}
where $Z^t$ is the normalizing factor of the weight vector $W^t$. In addition, their Step 2 can be refined to 
\begin{align*}
\ln\frac{Z^{t_2}}{Z^t}&=\sum_{\tau=t}^{t_2-1}\ln\bra{\sum_{i=1}^kw_i^\tau\exp\bra{\eta\hat{r}_i^\tau}}\\
&\le\sum_{\tau=t}^{t_2-1}\ln\bra{\sum_{i=1}^kw_i^\tau+\sum_{i=1}^kw_i^\tau\bra{\eta\hat{r}^\tau_i}+2\sum_{i=1}^kw_i^\tau\eta^2\bra{\hat{r}^\tau_i}^2}\\
&\le\sum_{\tau=t}^{t_2-1}\ln\bra{1+\eta\sum_{i=1}^kw_i^\tau\hat{r}^\tau_i+2\eta^2\sum_{i=1}^kw_i^\tau\hat{r}^\tau_i}\\
&\leq 
(\eta + 2\eta^2) \sum_{\tau = t}^{t_2 - 1} \inner{\hat{r}^\tau}{w^\tau} 
\end{align*}
Combining the above two, we have
\begin{align}
\DKL( w^t \parallel w^{t_2} )
\leq &
(\eta + 2\eta^2) \sum_{\tau = t}^{t_2 - 1} \inner{\hat{r}^\tau}{w^\tau} 
-
\eta \sum_{\tau = t}^{t_2 - 1} \inner{  \hat{r}^\tau }{ w^t }
\label{eqn:kl-bound}
\\
= & 
\eta \rbr{ \sum_{\tau = t}^{t_2 - 1} \inner{\hat{r}^\tau}{w^\tau} - \sum_{\tau = t}^{t_2 - 1} \inner{\hat{r}^\tau}{w^t}  }
+ 2\eta^2 \sum_{\tau = t}^{t_2 - 1} \inner{\hat{r}^\tau}{w^\tau}
\label{eqn:kl-bound-2}
\end{align}
Let us now apply concentration bounds (Freedman's inequality, Theorem~\ref{thm:freedman-easytouse}) on 
$\sum_{\tau = t}^{t_2 - 1} \inner{\hat{r}^\tau}{w^\tau}$ and $\sum_{\tau = t}^{t_2 - 1} \inner{  \hat{r}^\tau }{ w^t }$ respectively.
First,

\begin{align*}
\Var[\inner{\hat{r}^\tau}{w^\tau} \mid \Fcal_\tau]
\leq &
\sum_{i=1}^k (w_i^\tau)^2 \cdot \frac{\Var_{(x,y) \sim D_i}[\ell(h^\tau, (x,y))]}{ k \bar{w}_i^\tau } \\
\leq &
\frac{1}{k} \sum_{i=1}^k w_i^\tau \ell(h^\tau, D_i)
= \frac 1 k \ell(h^\tau, w^\tau) 
%\\
%\leq & \frac 1 k \rbr{ \nu + \varepsilon }
\end{align*}
%where the last inequality uses that $L(h^\tau, w^\tau) \leq \min_{h\in\Hcal} L(h, w^\tau)+\varepsilon \leq \nu + \varepsilon$.
%}

%where the second inequality uses that $L(h^\tau, w^\tau) \leq \min_{h\in\Hcal} L(h, w^\tau)+\varepsilon \leq \nu + \varepsilon$. 
%$v = (t_2 - t)(\nu + \eps)$, 

Therefore, Freedman's inequality (Theorem~\ref{thm:freedman-easytouse}) with $X_t = \inner{\hat{r}^\tau}{w^\tau}$,
$b = 1$
implies that with probability $1-\log T \cdot \delta' /2$, 
\begin{align}
\sum_{\tau = t}^{t_2 - 1} \inner{\hat{r}^\tau}{w^\tau}
\leq & 
\sum_{\tau = t}^{t_2 - 1}
L(h^\tau, w^\tau)
+ 
4\sqrt{ \frac{ \ln\frac2{\delta'}}{k} \sum_{\tau = t}^{t_2 - 1}
L(h^\tau, w^\tau) }
+ \ln\frac2{\delta'}
\\
\leq & 
(t_2 - t) (\bar \nu + \eps_1)
+ 
4 \sqrt{ \frac{ \ln\frac2{\delta'}}{k} (t_2 - t) (\bar \nu + \eps_1) }
+ 
\ln\frac2{\delta'},
\label{eqn:cumulative-reward-hedge}
\end{align}
where the second inequality uses the assumption that $L(h^\tau, w^{\tau}) \leq \min_{h \in \Hcal} L(h, w^t) + \eps_1$, which in turn   
$=
\min_{p \in \Delta(\Hcal)} \E_{h \sim p}\sbr{L(h, w^t)} + \eps_1
\leq 
\bar{\nu} + \eps_1$. 

As a side result, the above inequality combined with AM-GM inequality implies that
\begin{equation}
\sum_{\tau = t}^{t_2 - 1} \inner{\hat{r}^\tau}{w^\tau}
\leq 
2 (t_2 - t) (\bar \nu + \eps_1) 
+ 
5 \ln\frac{2}{\delta'}
\label{eqn:cumulative-reward-hedge-coarse}
\end{equation}

%\leq & 
%\sum_{\tau = t}^{t_2 - 1} L(h^\tau, w^\tau)
%+ 
%\sqrt{\sum_{\tau = t}^{t_2 - 1} L(h^\tau, w^\tau) \frac{\ln \frac2{\delta'}}{k} } + \ln \frac2{\delta'} \\

Similarly, with probability $1- \log T \cdot \delta'/2$, 
\begin{equation}
- \sum_{\tau = t}^{t_2 - 1} \inner{\hat{r}^\tau}{w^t}
\leq 
- \sum_{\tau = t}^{t_2 - 1} L(h^\tau, w^t)
+ 
4 \sqrt{\frac{\ln \frac2{\delta'}}{k} \sum_{\tau = t}^{t_2 - 1} L(h^\tau, w^t)  } + \ln \frac2{\delta'}.
\label{eqn:baseline-wt}
\end{equation}

% \chicheng{Freedman's inequality requires that conditional variance threshold $v$ to be a constant, not a random variable. This is a bit annoying in that the second application is no longer straightforward.
% Can we use a slightly more flexible version of Freeman's inequality Theorem~\ref{thm:freedman-easytouse})?
% }

Here, we know that for all $\tau$, $L(h^\tau, w^t) \geq \min_{h \in \Hcal} L(h^\tau, w^t) \geq L(h^t, w^t) - \varepsilon_1$ by the assumption of the lemma. Thus, 
\[
\sum_{\tau = t}^{t_2 - 1} L(h^\tau, w^t)
\geq 
(t_2 - t) (L(h^t, w^t) - \varepsilon_1)
= 
(t_2 - t) (\bar \nu + v^t - \varepsilon_1),
\]
where we recall $v^t=L(h^t,w^t)-\bar \nu$. Now we do a case analysis based on the value of $(t_2 - t) (\bar \nu + v^t - \varepsilon)$: 
\begin{itemize}
\item If $(t_2 - t) (\bar \nu + v^t - \varepsilon_1) \geq \frac{16\ln\frac2{\delta'}}{k}$, 
observe that $f(z) = -z + \sqrt{A z}$ is monotonically decreasing for $z \in [A, +\infty)$, we have that 
\begin{align*}
- \sum_{\tau = t}^{t_2 - 1} \inner{\hat{r}^\tau}{w^t}
\leq & 
-(t_2 - t) (\bar \nu + v^t - \varepsilon_1)
+
\sqrt{ \frac{16 \ln \frac{2}{\delta'}}{k} (t_2 - t) (\bar \nu + v^t - \varepsilon_1)   } + \ln \frac2{\delta'} 
% \leq & \edit{
% -(t_2 - t) \nu
% +
% \sqrt{  (t_2 - t) \nu \frac{16 \ln \frac2{\delta'}}{k} } + \ln \frac2{\delta'},}{}
\end{align*}
where the second inequality uses the fact that $v^t \leq \varepsilon_1$.
%\chicheng{(6/9/25) 1. I cannot verify the second inequality; 2. Without it the proof seems to work fine}
Combining this with inequality~\eqref{eqn:cumulative-reward-hedge}, we have
\begin{align*}
\eta \rbr{ \sum_{\tau = t}^{t_2 - 1} \inner{\hat{r}^\tau}{w^\tau}
- \sum_{\tau = t}^{t_2 - 1} \inner{\hat{r}^\tau}{w^t} }
\leq & 
\eta \rbr{
(t_2 - t) (2\eps_1 -v^t) 
+ 
8 \sqrt{ \frac{\ln\frac2{\delta'}}{k} (t_2 - t) (\bar \nu + \eps_1) }
+ 2 \ln\frac2{\delta'}
} \\
\leq & \eta \rbr{
(t_2 - t) (6\eps_1 -v^t) 
+ 
8 \sqrt{ \frac{\ln\frac2{\delta'}}{k} (t_2 - t) \bar \nu }
+ 6 \ln\frac2{\delta'}
},
\end{align*}
where the second inequality uses the elementary fact that $\sqrt{A+B} \leq \sqrt{A} + \sqrt{B}$ and 
AM-GM inequality that $8 \sqrt{ \frac{\ln\frac2{\delta'}}{k} (t_2 - t) \eps_1 } 
\leq 
4 (t_2 - t) \eps_1 + \frac{\ln\frac2{\delta'}}{k}$. 

in addition, by Eq.~\eqref{eqn:cumulative-reward-hedge-coarse},  
\begin{align*}
2\eta^2 \sum_{\tau = t}^{t_2 - 1} \inner{\hat{r}^\tau}{w^\tau}
\leq &
4 \eta^2 (t_2 - t) (\bar \nu + \eps_1) 
+ 
10 \eta^2 \ln\frac{2}{\delta'} \\
\leq & 
\eta (t_2 - t) \eps_1 + \eta \ln\frac2{\delta'},
\end{align*}
where the second inequality uses that $\eta = \frac{\eps_1}{100(\nu + \eps_1)} \leq \frac1{100}$.  Plugging these bounds into 
Eq.~\eqref{eqn:kl-bound-2}, we get the desired inequality~\eqref{eqn:refinement}.

\item Otherwise, $(t_2 - t) (\bar \nu + v^t - \eps) < \frac{16 \ln\frac2{\delta'}}{k}$. 
Intuitively, this means that canceling out the $(t_2 - t) \nu$ factors, i.e., the leading terms in Eqs.~\eqref{eqn:cumulative-reward-hedge} and~\eqref{eqn:baseline-wt}, is no longer important. In this case, continuing inequality~\eqref{eqn:cumulative-reward-hedge}, we have:
\begin{align*}
\sum_{\tau=t}^{t_2-2} \inner{\hat{r}^\tau}{ w^\tau } 
\leq & 
2 (t_2 - t) (\bar \nu + \varepsilon) + 2 \ln \frac{2}{\delta'} \tag{AM-GM}  \\
% = & 2 (t_2 - t) \nu + 2 (t_2 - t)\varepsilon + 2 \ln \frac{2}{\delta'}\\
% \leq & 2 (t_2 - t) (\varepsilon - v^t) + 2 (t_2 - t) \varepsilon + \frac{\ln\frac2{\delta'}}{k} + 2 \ln \frac{2}{\delta'}   \\
\leq & 
2 (t_2 - t ) (2 \varepsilon - v^t) + 34 \ln\frac{2}{\delta'} \tag{assumption for this case}
\end{align*}
Using inequality~\eqref{eqn:kl-bound} and dropping the (nonnegative) second term, we have
\begin{align*}
\DKL( w^t \parallel w^{t_2} )
\leq & 
2\eta \cdot \sum_{\tau=t}^{t_2-2} \inner{\hat{r}^\tau}{ w^\tau } \\
\leq & 
4\eta \cdot (t_2 - t ) (2 \varepsilon - v^t) + 68 \eta \ln\frac{2}{\delta'}
%\\
%\le & 
%\edit{4\eta \cdot (t_2 - t ) (2 \varepsilon) + 68 \eta \ln\frac{2}{\delta'}.}{}
\end{align*}
%\chicheng{I thought that the $-v^t$ cannot be dropped since it can well be negative.}
\end{itemize}
\paragraph{Final Steps} Let's recall that $T=\widetilde{O}\bra{\frac{\varepsilon+\nu}{\varepsilon^2}}$, $\eta=\tilde{O}\bra{\sqrt{\frac{1}{T\bra{\varepsilon+\nu}}}}=O\bra{\frac{\varepsilon}{\varepsilon+\nu}}$, $\varepsilon_1=O\bra{\varepsilon}$ and $j_{\max}=\left\lfloor\log\frac{1}{\eta(\nu+\eps_1)}\right\rfloor+1$. 
Due to the slight change in $j_{\max}$, 
Eq. (116) of~\citet{zhang2024optimal} now needs to be modified to: there exists some $1 \leq \tilde{j} \leq j_{\max}$ such that 
\[
\ln \rbr{ \frac{y_{\tilde{j}+1}}{y_{\tilde{j}}} }
\geq 
\frac{x}{ \log \frac{1}{\eta (\nu + \eps_1)} + 1 }
\]

%\chicheng{Can we give more details? I think this uses the assumption of the lemma which is kind-of important.}

For for this $\tilde{j}$, 
by the stronger KL divergence upper bound (Eq.~\eqref{eqn:refinement}), equation (119) in \citet{zhang2024optimal} becomes the following:
\begin{align}
\tau_{\widetilde{j}+1}-\tau_{\widetilde{j}}
\gtrsim &
\min\rbr{ 
\frac{qx^2}{\bra{\log\frac{1}{\eta(\nu + \eps_1)}+1}^2}\min\brc{\frac{1}{\eta\varepsilon_1},\frac{2^{\widetilde{j}-1}}{\eta}},
\frac{k q^2 x^4}{ 
\eta^2 \nu  \log\frac1{\delta'} \bra{\log\frac{1}{\eta(\nu + \eps_1)}+1}^4
}} \\
\gtrsim &
\frac{qx^2}{\bra{\log\frac{1}{\eta(\nu + \eps_1)}+1}^2} \cdot \frac{2^{\tilde{j}-1}}{\eta}
\end{align}
Here, the second inequality holds since
\[
\frac{qx^2 \cdot 2^{\tilde{j}-1}}{\eta}
\lesssim 
\frac{qx^2}{\eta \eps_1}
\lesssim 
\frac{k q^2 x^4}{\eta^2 \nu \log\frac1{\delta'} \bra{\log\frac{1}{\eta(\nu + \eps_1)}+1}^2 }
\]
where the first inequality uses that $2^{\tilde{j}-1} \leq 2^{j_{\max}-1}
\leq \frac{1}{\eta(\nu+\epsilon)}
$
and the second uses the assumption that $\frac{q x^2}{50 (\log \frac{1}{\eta(\nu+\eps_1)} + 1)^2} \geq \frac1k$ and $\eta \nu \lesssim \eps_1$. 

%Note that both $\frac{1}{\eta\varepsilon_1}$ and $\frac{1}{\eta^2\nu}$ are acceptable so we only need to bound the middle term. 

Then Case 1 ($\tilde{j} = j_{\max}$) in Step 5 of Section C.4 in \citet{zhang2024optimal} becomes
\[
t_2-t_1
\geq
\tau_{\tilde{j}+1}
- 
\tau_{\tilde{j}}
\gtrsim 
\frac{qx^2}{\bra{\log\frac{1}{\eta(\nu + \eps_1)}+1}^2} \cdot \frac{2^{j_{\max}-1}}{\eta}
\gtrsim
\frac{qx^2}{200\bra{\log \frac 1 {\eta(\nu + \eps_1)} +1}^2\eta^2(\nu+\varepsilon)}
\]
Moreover, Case 2 ($1 \leq \tilde{j} \leq  j_{\max} - 1$) in Step 5 of Section C.4 in \citet{zhang2024optimal} stays the same. This completes the proof of the lemma.
\end{proof}
We now proceed to prove Lemma~\ref{lemma:lemma13} using Lemma~\ref{lemma:segmentlb}.
\begin{proof}
Define $\cbr{\Vcal_j^n}_{n=1}^N$ and $\cbr{(\hat{s}_n, \hat{e}_n)}_{n=1}^N$ the disjoint expert subsets and time intervals guaranteed to exist by Lemma 16 of~\cite{zhang2024optimal}. 

Using our new Lemma~\ref{lemma:segmentlb} (which is a strengthening of their Lemma 17), the inequality (79) in \citet{zhang2024optimal} can be strengthened to
\begin{align}    T\eta(\nu + \eps)+\brc{4T\varepsilon_1+4\sum_{t=1}^T\bra{-v^t}}
\ge &
(\nu + \eps) \sum_{n=1}^N\bra{\hat{e}_n-\hat{s}_n}\eta+4\sum_{n=1}^N\sum_{\tau=\hat{s}_n}^{\hat{e}_n-1}\bra{-v^\tau+\varepsilon_1}
\\
\gtrsim & 
\frac{2^{-j} \sum_{n=1}^N |\Vcal_j^n| }{\log^2(k) (\log \frac{1}{\eta(\nu + \eps_1)}+1)^2 \cdot \eta}
\end{align}
where the first inequality uses that $-v^\tau + \eps_1 \geq 0$ for all $\tau$ and dropping those terms whose time steps do not lie in $[\hat{s}_n, \hat{e}_n)$; the second inequality is from our new Lemma~\ref{lemma:segmentlb}.

In addition, we strengthen their Eq. (81) and bound $\sum_{t=1}^T (-v^t)$ by the following: 
\begin{align*}
\sum_{t=1}^T (-v^t)
= & \bar{\nu} T - \sum_{t=1}^T  L(h^t, w^t) \\
\leq & \max_{i \in [k]} \sum_{t=1}^T L_i(h^t ) - \sum_{t=1}^T  L(h^t, w^t)\\
\leq & \frac{3\ln(k)}{2\eta}+\frac{3}{2}\eta T(\nu+\varepsilon_1)+58\ln(1/\delta')+5\sqrt{2\ln(1/\delta')T(\nu+\varepsilon_1)} \\
\leq & T \eps
\end{align*}
where 
the first inequality is since 
$\bar{\nu} = \min_{p \in \Delta(\Hcal)} \max_{i} \E_{h \sim p} L(h, D_i) \leq \max_{i \in [k]} \frac 1 T \sum_{t=1}^T L_i(h^t )$, 
and the second inequality is from Eq.~\eqref{eqn:regret-form}.

As a result, inequality (82) in \citet{zhang2024optimal} becomes
\begin{align*}
\frac{\sum_{n=1}^N\cardin{\Vcal_j^n}}{2^j} \lesssim &
\rbr{
\eta\bra{T \eps}+\eta\bra{T\eta(\nu + \eps) + T\varepsilon}
} \cdot \log^2(k) \rbr{\log \frac{1}{\eta(\nu + \eps_1)}+1}^2
\\
\lesssim &
\ln\rbr{ \frac{k}{\delta' \eps} } \cdot \log^2(k) \rbr{ \log \frac{1}{\eta(\nu + \eps_1)}+1 }^2
\end{align*}
as we desired. The rest of the proof then follows.
\end{proof}

\section{Deferred materials from Section~\ref{sec:lower-bound}}
\label{sec:def-lower-bound}

\begin{proof}[Proof of Theorem~\ref{thm:realizable-lb}]
Since $\Hcal$ has star number $k \vartheta$, there exists some $h_0, h_1, \ldots, h_{k \vartheta} \in \Hcal$ 
and a set of $k \vartheta$ examples $X = \cbr{ x_1, \ldots, x_{k \vartheta} }$, 
such that:
for every $i \in [k \vartheta]$, $h_i(x_i) \neq h_0(x_i)$, and $h_i(x_j) = h_0(x_j)$ for all $j \neq i$. 

For \(1\le i\le k\), let 
\[
X^i=\cbr{x_{(i-1)\vartheta +1},\cdots,x_{i\vartheta}},
\]
Thus, 
\[
X=\bigcup_{1 \le i\le k}X^i, 
\]
in other words, $(X^i)_{i=1}^k$ form a partition of the set $X$.
For each \(1\le i\le k\) and \(1\le j \le \vartheta \), let \(D^i_j\) be the distribution with marginal uniform on \(X^i\) and the labels of all examples are consistent with $h_{(i-1)\vartheta + j}$. 
Also, let \(D^i_0\) be the distribution with marginal uniform on \(X^i\) in which the labels of all examples are consistent with \(h_0\).

It can be checked that the disagreement coefficient of any $D \in \cbr{ D^i_j: i \in [k], j \in \cbr{0, 1, \ldots, \vartheta} }$  is at most $\vartheta$; to see this, note that for any $D$ above, $\B_D(h, r) = \cbr{h}$ unless $r \geq \frac{1}{\vartheta}$, and thus $\Pr\sbr{\DIS(\B_D(h, r))} \leq \vartheta r$ for all $r > 0$.

%Moreover, the disagreement coefficient of every distribution is at most \(\vartheta\).

%with the label of \(x_{(i-1)\vartheta+j}\) equal to \( -h_0(x_{(i-1)\vartheta+j})\) and all other labels equal to \( h_0(x_{(i-1)\vartheta+j}) \).

%Then define
%\[
%\Hcal=H_0\cup\cbr{h^i_j:i\in[k],1\le j\le\lceil \vartheta\rceil}.
%\]

%We construct the problem instance as follows. 

% Let \(X_0\) be a set of size \(d-1\) and let \(H_0\) be a set of hypotheses that shatters \(X_0\) and classifies every other point as \(-1\). For \(1\le i\le k\) and \(1\le j\le\lceil \vartheta\rceil\), let \(h^i_j\) be the hypothesis that classifies point \(x^i_j\) as \(1\) and every other point as \(-1\). 

%Thus, the hypothesis class \(\Hcal\) has VC-dimension \(d\).

% Define 
% \[
% \Dcal^i_j=\cbr{D^i_j}\cup\Bigl(\bigcup_{p\in [k]\text{ and }p\neq j}\cbr{D^0_p}\Bigr).
% \]

Define a family of $k \vartheta$ MDL instances $\Dcal_j^i$, $i \in [k]$, $j \in [\vartheta]$ as follows: 
\[
\Dcal^i_j = 
( D^1_0, \ldots, D^{j-1}_0, D^j_i, D^{j+1}_0, \ldots, D^k_0 )
\]
By the previous paragraph, we know that for every $\Dcal_{j}^i$, $\theta_{\max}(\eps)$ is at most $\vartheta$. In addition, it can be easily seen that  $L_{\Dcal^i_j}(h_{(i-1)\vartheta + j}) = 0$, thus $\min_{h \in \Hcal} L_{\Dcal^i_j}(h) = 0$. 
We next show a simple property that all $\Dcal^i_j$'s are sufficiently apart, in that no classifier can be simultaneously near-optimal in any pair of them: 
\begin{claim}
For any classifier $\hat{h}$, and any $(i_1, j_1) \neq (i_2, j_2)$, 
$L_{\Dcal_{j_1}^{i_1}}(\hat{h}) \leq \eps$ and 
$L_{\Dcal_{j_2}^{i_2}}(\hat{h}) \leq \eps$ cannot hold simultaneously. 
\label{lem:instance-sep}
\end{claim}
\begin{proof}
%We first consider the case that $i_1 = i_2$ and $j_1 \neq j_2$. 
$L_{\Dcal_{j_1}^{i_1}}(\hat{h}) \leq \eps$ implies that
\begin{equation}
\Pr_{x \sim \Unif(X^{i_1})}\sbr{\hat{h}(x) = h_0(x),  x = x_{\vartheta(i_1-1)+j_1}} \leq \Pr_{(x,y) \sim D_{j_1}^{i_1}}\sbr{ \hat{h}(x) \neq y } \leq \eps.
\label{eqn:ineq-1}
\end{equation}
%h_{\vartheta(i_1-1)+j_1}(x)
On the other hand, 
$L_{\Dcal_{j_2}^{i_2}}(\hat{h}) \leq \eps$ implies that
\begin{equation}
\Pr_{x \sim \Unif(X^{i_1})}\sbr{\hat{h}(x) \neq h_0(x),  x = x_{\vartheta(i_1-1)+j_1}} \leq \Pr_{(x,y) \sim (\Dcal_{j_2}^{i_2})_{i_1}}\sbr{ \hat{h}(x) \neq y } \leq \eps.
\label{eqn:ineq-2}
\end{equation}
Here, $(\Dcal_{j_2}^{i_2})_{i_1}$ is the $i_1$-th distribution of problem instance $\Dcal_{j_2}^{i_2}$, which is $D_{j_2}^{i_1}$ if $i_1 = i_2$, and is $D_{0}^{i_1}$ if $i_1 \neq i_2$; in either case, the distribution has $x_{\vartheta(i_1-1)+j_1}$ agreeing with $h_0$.

Adding up Eqs.~\eqref{eqn:ineq-1} and~\eqref{eqn:ineq-2} and using triangle inequality, we have 
\[
\frac{1}{\vartheta}
= 
\Pr_{x \sim \Unif(X^{i_1})}(x = x_{\vartheta(i_1-1)+j_1}) \leq 2 \eps.
\]
which contradicts with the assumption that $\eps < \frac{1}{2\vartheta}$. 
\end{proof}

%\edit{
%Note that if the distribution class is \(\Dcal^i_j\) and we set the error tolerance \(\varepsilon\le \frac{1}{10\vartheta}\), then the only acceptable hypothesis in $\Hcal$ is \(h^i_j\).}{I see what this is about - but I think we need to aim for improper learning for this proof?} \joey{I think we agreed that we defer this modification to the arxiv version.}

% and define a ``background'' MDL instance $\Dcal_0$ as:
% \[
% \Dcal_0 = (D^1_0, \ldots, D^k_0). 
% \]

For any algorithm \(A\), consider an algorithm \(A_0\) that chooses label queries identical to \(A\) except that 
every label query made by \(A_0\) on example $x$ returns \(h_0(x) \). Observe that the distribution of the transcript of $A_0$ is the same on every problem instance $\Dcal_j^i$ because different problem instances has the same set of marginal distributions over $\Xcal$.

Next, we show that if \(A\) makes fewer than \(\frac{k\vartheta}{20}\) queries, then it fails to solve a family of two problem instances with good probability.
Denote by $\hat{h}$ the classifier returned by $A$ after \( \frac{k\vartheta}{20}\) queries.
Let \(d(x)\) denote the expected number of queries for each \(x\in\Xcal\) during the first \(\frac{k\vartheta}{20}\) queries of \(A_0\).
Therefore, 
\[
\sum_{x \in X} d(x) \leq \frac{k\vartheta}{20}.
\]
By the Pigeonhole Principle, at least half of $x$ in \(X\) satisfies \(d(x)\le\frac{1}{10}\). Let \(x_{(i_1-1)\vartheta+j_1}\) and \(x_{(i_2-1)\vartheta+j_2}\) be two such points. Then, by Markov's inequality and the union bound, the probability that \(A_0\) queries either of these points is at most \(0.2\). Consider the problem instances $\Dcal$ chosen uniformly from \(\cbr{\Dcal^{i_1}_{j_1},\Dcal^{i_2}_{j_2}}\), 
denote by event $E$ that \(A\) does not query point \(x_{(i_1-1)\vartheta+j_1}\) or \(x_{(i_2-1)\vartheta+j_2}\), then we claim that 
\begin{align*}
    \Pr\sbr{ L_{\Dcal}\rbr{\hat{h}}\le\eps \mid E }\le 0.5,
\end{align*}
%\Dcal \sim \Unif\cbr{\Dcal^{i_1}_{j_1},\Dcal^{i_2}_{j_2}  

where the randomness comes from both the internal randomness of algorithm $A$ and the distribution of the problem instance $\Dcal$. 
To see this, note that conditioned on the event 
$A$ doesn't query $x_{(i_1-1)\vartheta+j_1}$ or $x_{(i_2-1)\vartheta+j_2}$, the posterior distribution of $\Dcal$ after all $\frac{k \vartheta}{20}$ queries is still uniform over $\cbr{\Dcal^{i_1}_{j_1},\Dcal^{i_2}_{j_2} }$. Therefore, the left hand side is 
\[
\frac12 \rbr{ 
\Pr\sbr{ L_{\Dcal_{j_1}^{i_1}}\rbr{\hat{h}}\le\eps \mid E}
+ 
\Pr\sbr{ L_{\Dcal_{j_2}^{i_2}}\rbr{\hat{h}}\le\eps \mid E}
},
\]
which is at most 0.5 because of Lemma~\ref{lem:instance-sep}.

%This is because $A$ does not get any information under this condition so it cannot do better than a random guess. 
Notice that the behavior of \(A_0\) and \(A\) is identical until \(A\) queries some $x$ and encounters a label \(-h_0(x) \). Therefore,
\begin{align*}
    \Pr\sbr{\text{$A$ doesn't query $x_{(i_1-1)\vartheta+j_1}$ or $x_{(i_2-1)\vartheta+j_2}$}}=\Pr\sbr{\text{$A_0$ doesn't query $x_{(i_1-1)\vartheta+j_1}$ or $x_{(i_2-1)\vartheta+j_2}$}}.
\end{align*}
Consequently, the success probability of \(A\) is at most
\begin{align*}
&\Pr\sbr{ L_{\Dcal}\rbr{\hat{h}}\le\eps \mid\text{$A$ doesn't query $x_{(i_1-1)\vartheta+j_1}$ or $x_{(i_2-1)\vartheta+j_2}$}}\\
+&\Pr\sbr{\text{$A_0$ queries $x_{(i_1-1)\vartheta+j_1}$ or $x_{(i_2-1)\vartheta+j_2}$}}\le0.5+0.2=0.7
\end{align*}

Therefore, there exists at least one instance in $\{\Dcal^{i_1}_{j_1}, \Dcal^{i_2}_{j_2}\}$ that $A$ cannot solve with fewer than $\tfrac{k\vartheta}{20}$ queries, with probability exceeding $0.7$.
\end{proof}

\begin{proof}[Proof of Theorem~\ref{thm:agnostic-lb}]
We construct the problem instance as follows. Let the hypothesis class \(\cbr{h_1,h_2}\) be the two hypotheses satisfying the assumption in the theorem. Let \(\Xcal=X_{\DIS}\cup X_{\AGE}\), where:
\begin{itemize}
    \item \(X_{\DIS}=\cbr{x_1,x_2}\) is the disagreement region: for all \(x \in X_{\DIS}\), \(h_1(x) \neq h_2(x)\).
    \item \(X_{\AGE} =\cbr{z_1,\cdots,z_k}\) is the agreement region: for all \(x \in X_{\AGE}\), \(h_1(x) = h_2(x)\).
\end{itemize}
We next define several building blocks for constructing our distributions. 
Define
\[
D_{x_1}(x,y)=I(x=x_1)I(y=h_1(x_1)) \quad \text{and} \quad D_{x_2}(x,y)=I(x=x_2)I(y=h_2(x_2)).
\]
Also, let
\[
D_{z_1}(x,y)=I(x=z_1)I(y=h_1(z_1)),
\]
and for \(2\le i\le k\),
\[
D_{z_i}(x,y)=I(x=z_i)\bra{\bra{1-\frac{\nu-4\varepsilon}{2-\nu}}I\bra{y=h_1(z_i)}+\frac{\nu-4\varepsilon}{2-\nu}I\bra{y\neq h_1(z_i)}},
\]
and
\[
D'_{z_i}(x,y)=I(x=z_i)\bra{\bra{1-\frac{\nu+4\varepsilon}{2-\nu}}I\bra{y=h_1(z_i)}+\frac{\nu+4\varepsilon}{2-\nu}I\bra{y\neq h_1(z_i)}},
\]
for \(2\le i\le k\). Next, define the distributions
\[
D_1=\nu D_{x_1}+(1-\nu) D_{z_1},
\]
and for \(2\le i\le k\),
\[
D_i=\frac{\nu}{2}D_{x_2}+\bra{1-\frac{\nu}{2}}D_{z_{i}},
\]
and
\[
D'_i=\frac{\nu}{2}D_{x_2}+\bra{1-\frac{\nu}{2}}D'_{z_{i}}.
\]

From the construction, we obtain the following table for the error of \(h_1\) and \(h_2\) on the distributions under consideration:
\begin{table}[H]
\centering
\begin{tabular}{|c|c|c|c|c|c|c|c|}
\hline
Error & \(D_1\) & \(D_2\) & \(\ldots\) & \(D_k\) & \(D_2'\) & \(\ldots\) & \(D_k'\) \\
\hline
\(h_1\) & 0 & \(\nu - 2\varepsilon\) & \(\ldots\) & \(\nu - 2\varepsilon\) & \(\nu + 2\varepsilon\) & \(\ldots\) & \(\nu + 2\varepsilon\) \\
\hline
\(h_2\) & \(\nu\) & \(\frac{\nu}{2} - 2\varepsilon\) & \(\ldots\) & \(\frac{\nu}{2} - 2\varepsilon\) & \(\frac{\nu}{2} + 2\varepsilon\) & \(\ldots\) & \(\frac{\nu}{2} + 2\varepsilon\) \\
\hline
\end{tabular}
\caption{Error of \(h_1\) and \(h_2\) on the constructed distributions}
\end{table}

Define our central MDL problem instance as \(\Dcal=\cbr{D_i}_{i\in[k]}\), and define \(k-1\) auxiliary MDL problem instances such that \(\Dcal_i\) is \(\Dcal\) with the \(i\)th distribution \(D_i\) replaced by \(D'_i\) for \(2\le i\le k\). 
Suppose algorithm \(A\) uses a label budget of \(n\) and guarantees that under all instances \(\Qcal \in \cbr{ \Dcal, \Dcal_2, \ldots, \Dcal_{k} }\),
\[
\Pr_{A, \Qcal} \sbr{ L_{\Qcal}(\hat{h}) \leq \min_{h \in \Hcal} L_{\Qcal}(h) + \varepsilon } \geq 0.9.
\]
Note that the event 
\[
\cbr{ L_{\Dcal}(\hat{h}) \leq \min_{h \in \Hcal} L_{\Dcal}(h) + \varepsilon } = \cbr{ \hat{h} = h_1 },
\]
and for all \(i \in \cbr{2, \ldots, k}\),
\[
\cbr{ L_{\Dcal_i}(\hat{h}) \leq \min_{h \in \Hcal} L_{\Dcal_i}(h) + \varepsilon } = \cbr{ \hat{h} = h_2 }.
\]
Therefore, for all \(2\le i\le k\),
\[
\Pr_{A, \Dcal}\sbr{ \hat{h} = h_2 } +
\Pr_{A, \Dcal_i}\sbr{ \hat{h} = h_1 }
\leq 0.2.
\]
By the Bretagnolle-Huber inequality, it follows that
\[
\frac12 \exp\Bigl(-\DKL\rbr{ \PP_{A, \Dcal} \middle\| \PP_{A, \Dcal_i} }\Bigr)
\]
bounds the total variation distance between the joint distributions \(\PP_{A, \Dcal}\) and \(\PP_{A, \Dcal_i}\), where these denote the joint distribution of the interaction transcript \((x_1, y_1, \ldots, x_n, y_n)\) between \(A\) and the problem instances \(\Dcal\) and \(\Dcal_i\), respectively. This implies that 
\[
\DKL( \PP_{A, \Dcal}\| \PP_{A, \Dcal_i}) \geq \ln \frac{5}{2}.
\]
By the divergence decomposition lemma for interactive learning algorithms~\cite[][Lemma 15.1]{lattimore2020bandit},
\[
\DKL\rbr{ \PP_{A, \Dcal}\| \PP_{A, \Dcal_i}} = \E_{ A, \Dcal }\Bigl[ n_i \DKL\rbr{\PP_{Y_i}\middle\| \PP_{Y_i'}} \Bigr],
\]
where \(n_i\) is the number of label queries to \(z_i\), \(\PP_{Y_i}\) is $z_i$'s label distribution  under \(D_i\), and \(\PP_{Y'_i}\) is $z_i$'s label distribution under \(D'_i\). Let $p=\frac{\nu-4\eps}{2-\nu}$ and $q=\frac{\nu+4\eps}{2-\nu}$, for each \(i\), from Lemma~\ref{lemma:KLintegraltrick}, we have
\begin{align*}
    \DKL\bra{\PP_{Y_i}\middle\|\PP_{Y'_i}} 
    &= \int_p^q\frac{x-p}{x(1-x)}\mathrm{d} x.
\end{align*}
Because $0<8\eps\le\nu\le\frac{1}{2}$, 
\begin{align*}
    x \ge
    p = \frac{\nu - 4\eps}{2 - \nu}
    \geq 
    \frac{\nu / 2}{2}
    \geq 
    \frac{\nu}{4},
    %x(1-x) \ge 
    %\frac{\nu}{2}\cdot\frac{1}{2}\le \frac{\nu}{4},
\end{align*}
and
\[
    1 - x
    \geq 
    1 - q
    \geq 
    \frac{2 - \frac{3}{2}\nu}{2 - \nu}
    \geq
    \frac{1}{2}.
%q(1-q)\ge\frac{\nu}{2}\cdot\frac{1}{2}\le \frac{\nu}{4}.
\]
As a result, $x(1-x)\ge\frac{\nu}{8}$ for all $x\in[p,q]$. Thus,
\[
\DKL( \PP_{A, \Dcal}\| \PP_{A, \Dcal_i}) \leq \int_p^q\frac{x-p}{\nu/8}\mathrm{d} x=\frac{8}{\nu}\int_p^q\rbr{x-p}\mathrm{d} x=\frac{4(q-p)^2}{\nu}.
\]
Note that $q-p=\frac{4\eps}{2-\nu}$ so
\[
\frac{4(q-p)^2}{\nu}=\frac{4}{\nu}\rbr{\frac{4\eps}{2-\nu}}^2\le\frac{64\eps^2}{\nu},
\]
where the last inequality comes from the assumption $\nu\le\frac{1}{2}$.
This implies that for all \(i \in \cbr{2, \ldots, k}\),
\[
\E_{A, \Dcal}\sbr{n_i} \geq \frac{\nu}{64\varepsilon^2} \ln \frac{5}{2}.
\]
Taking the summation over all \(i\), we conclude that the expected number of queries made by \(A\) on \(\Dcal\) is at least 
\[
(k-1) \cdot \frac{\nu}{64\varepsilon^2} \ln \frac{5}{2}.
\]
\end{proof}

\section{Deferred materials from Section~\ref{sec:distn-free}}
\label{sec:df-deferred}

%\subsection{Distribution-free setting}
\subsection{\PRMDL and its guarantees}
\label{sec:prmdl-grt}

%$\hat{f}: \Xcal \to \cbr{-1,0,+1}$, such that whenever $\hat{f}$ predicts a binary label, it agrees with $h^*$; the probability that $\hat{f}$ abstains is small, across all distributions $(\mu_i)_{i \in [k]}$
We first present \PRMDL (Alg.~\ref{alg:prmdl}), a key subprocedure used by our distribution-free active learning algorithm Alg.~\ref{alg:main-large-eps-df}; it can be viewed as a  collaborative RPU learning algorithm robust to label noise. 
It takes a collection of distributions $(\mu_i)_{i \in [k]}$ that are approximately realizable by $h^*$, 
and tries to find a $\xi$-RPU classifier for some target $\xi > 0$ (recall definition~\ref{def:rpu-classifier}). 
To this end, it calls a subprocedure \PRSDL (Alg.~\ref{alg:robust-rpu}) that learns a RPU classifier for a single distribution (step~\ref{step:get-fr}). Our algorithmic idea is largely inspired by~\cite{blum2017collaborative} for the ordinary collaborative PAC learning problem: calling \PRSDL on uniform mixtures of subsets of distributions in $(\mu_i)_{i \in [k]}$ (denoted by $N_r$, step~\ref{step:unif-mix}); as soon as our trained RPU classifier has low abstention probability on some  $\mu_j$, we remove $\mu_j$ from the subset (step~\ref{step:prune-learnt}). After learning RPU classifiers that cover all of the $(\mu_i)_{i \in [k]}$, we return $\hat{f}$, which makes a $\pm 1$ prediction whenever one of the learned $f^r$'s makes a $\pm 1$ prediction (step~\ref{step:return-final}).

\begin{algorithm}
\caption{\textsc{Passive-RPU-MDL}}
\label{alg:prmdl}
\begin{algorithmic}
\REQUIRE{Hypothesis class $\Hcal$ with start number $\s$, distributions $( {\mu_{i} } )_{i \in [k]}$, target reliability $\xi$, target confidence $\delta$}

\STATE $N_1 \gets [k]$, $r \gets 1$

%, R = \lceil \log k \rceil
%$r=1, \ldots, R$

\WHILE{$N_r \neq \emptyset$}

\STATE Define $\widetilde{\mu}_r := \frac{1}{|N_r|} \sum_{i \in N_r} \mu_i$

\label{step:unif-mix}

\STATE Let $f^r \gets \textsc{Robust-RPU-Learn}(\Hcal, \widetilde{\mu}_r, \frac{\xi}{2}, \frac{\delta}{2} )$

\label{step:get-fr}

\STATE $N_{r+1} \gets N_r \setminus G_r$, where $G_r = \cbr{ i: \Pr_{D_i}( f^r(x) = 0 ) \leq \xi }$

\label{step:prune-learnt}

\STATE $r \gets r + 1$

\ENDWHILE

\RETURN $\hat{f}$, such that 
\[
\hat{f}(x) = 
\begin{cases}
0, & \forall r, f^r(x) = 0 \\
f^{c}(x), & c = \min\cbr{r: f^r(x) \neq 0}
\end{cases}
\]
\label{step:return-final}
\end{algorithmic}
\end{algorithm}

\PRSDL (Alg.~\ref{alg:robust-rpu}) is our procedure for robust RPU learning for a single distribution 
in the presence of noise.
Different from the previous RPU learning algorithms~\citep{wiener2015compression,kane2017active,hopkins2020power}, it is able to tolerate adversarial label noise.  To this end, it samples multiple subsamples $S_i$'s and builds a RPU classifier $f_i$ for each $S_i$. Noting that some $S_i$'s may be inconsistent with $h^*$, causing the output $f_i$ to be unreliable, we take a thresholded majority vote of the $f_i$'s, denoted as $\hat{f}$ (step~\ref{step:thres-maj}): $\hat{f}$ abstains on $x$ if not enough $f_i$'s output a binary prediction on it; otherwise we predict the majority binary label of the $f_i$'s. 
We present in the following lemma about the sample efficiency and noise tolerance of \PRSDL: 

\begin{lemma}
\label{lem:robust-rpu}
Suppose hypothesis class $\Hcal$ and 
distribution $\mu$ is such that: there exists $h^* \in \Hcal$ such that 
$L\bra{h^*, \mu}\leq \eta$, and the target reliability $\xi \geq 100 \mathfrak{s} \eta$,  then $\textsc{Robust-RPU}(\xi, \delta)$ is such that with probability $1-\delta$: (1) it outputs $\hat{f}$ that is $\xi$-RPU with respect to $(h^*,  \mu)$; 
(2) the total number of iid examples sampled from $\mu$ is $\widetilde{O}( \frac{\mathfrak{s}}{\xi} \ln\frac1\delta )$. 
\end{lemma}

\begin{algorithm}[t]
\caption{$\textsc{Robust-RPU-Learn}$}
\label{alg:robust-rpu}
\begin{algorithmic}

\REQUIRE Hypothesis class $\Hcal$ with star number $\s$, distribution $\mu$, target reliability $\xi$, target confidence $\delta$ 

\STATE Let $N \gets 60 \lceil \ln\frac{1}{\delta} \rceil$

\FOR{$i=1, \ldots, N$:}

\STATE Let $S_i \gets $ Sample $n = O\rbr{ \frac{\mathfrak{s} + \ln\frac{N}{\delta}}{\xi} }$ iid examples from $\mu$ 

\STATE define 
$V_i := \cbr{h \in \Hcal: h(x) = y, \text{for all } (x,y) \in S_i}$

\IF{$V_i = \emptyset$}

%\STATE 

\STATE Define $f_i \equiv 0$. 

\COMMENT{\blue{In this case, we know that sample $S_i$ has been corrupted; we choose $f_i \equiv 0$ as a convention}}

%\blue{Define $f_i \equiv 0$ Why?}
%\chicheng{In this case, $f_i$ can actually be defined arbitrarily, since we know that in this case, the samples have been corrupted. Perhaps I can add this remark.
%}
    
\ELSE

\STATE Define $f_i: \Xcal \to \cbr{-1, +1, 0}$ such that
\[
f_i(x) = 
\begin{cases}
0 & x \in \DIS(V_i) \\
V_i(x) & \text{otherwise}, 
\end{cases}
\]

\ENDIF

%where $h$ is an arbitrary classifier in $V_i$.
\ENDFOR

\RETURN $\hat{f}$, defined as:
\[
\hat{f}(x) = 
\begin{cases}
0 & \sum_{i=1}^N I(f_i(x) \neq 0) \leq \frac{N}{5} \\
\sign( \sum_{i=1}^N f_i(x) ) & \text{otherwise} \\ 
\end{cases}
\]
\label{step:thres-maj}
\end{algorithmic}
\end{algorithm}

% \PROCEDURE{Subprocedure}{}
%     \State Perform some operations
%     \State Return a result
% \ENDPROCEDURE

With this, we now present the sample complexity and noise tolerance guarantee of \PRMDL in Lemma~\ref{lem:passive-rpu-mdl}. The proofs of both lemmas can be found at the end of this subsection. 

%, which relies on the single-distribution RPU learning guarantee of \PRSDL, as we present in Lemma~\ref{lem:robust-rpu}. 

\begin{lemma}
\label{lem:passive-rpu-mdl}
Suppose hypothesis class $\Hcal$ and 
distributions $\mu_1, \ldots, \mu_k$ are such that there exists $h^*$ such that 
$$
\max_{i\in[k]} L\bra{h^*,\mu_i}
\leq \eta,
$$
and the target reliability $\xi \geq 100 \mathfrak{s} \eta$, then $\PRMDL(\xi, \delta)$ is such that with probability $1-\delta$: 
(1) it outputs $\hat{f}$ that is $\xi$-RPU with respect to $(h^*, (\mu_i)_{i=1}^k)$; 
% (1) with probability $1-\delta$, it outputs a classifier $\hat{f}$ such that
% (1a) $\hat{f}(x) \neq 0 \implies \hat{f}(x) = h^*(x)$, and 
% (1b)
% $$
% \max_{i\in[k]} \Pr_{\mu_i}( \hat{f}(x) = 0 )
% \leq \xi,
% $$ 
(2) The total number of examples sampled from any of the $\mu_i$'s is $\widetilde{O}( \frac{\mathfrak{s}}{\xi} )$. 
\end{lemma}

We prove Lemma~\ref{lem:robust-rpu} first. 

\begin{proof}[Proof of Lemma~\ref{lem:robust-rpu}]
Define event $E_i$ as:
\[
E_i := 
\cbr{
f_i \text{ is } \frac{\xi}{2}\text{-RPU with respect to } (h^*, \mu)
}
\]

%\forall (x,y) \in S_i: h^*(x) = y
%\wedge 
%\PP_{\mu}( x \in \DIS(V_i) ) \leq \xi

We first prove a claim that shows that $E_i$ happens with constant probability.
\begin{claim}
$\Pr\sbr{E_i} \geq \frac{9}{10}$.
\label{claim:e-i}
\end{claim}
\begin{proof}
For a set of labeled examples $S$, let $V[h^*, S]$ denote the  set of hypotheses in $\Hcal$ consistent with $h^*$ on $S$.

%version space of $\Hcal$ induced by $(x_1, h^*(x_1)), \ldots, (x_n, h^*(x_n))$
%. 
%\joey{Maybe explain this a bit more, $V[h^*,S]$ is the set of hypothesis consistent with $h^*$ on $S$.}

Define $F_i = \cbr{ \forall (x,y) \in S_i: h^*(x) = y }$ 
and 
$G_i = \cbr{ \Pr_\mu\sbr{x \in \DIS(V[h^*, S])} \le\frac{\xi}{2} }$. It can be seen that $E_i \supseteq F_i \cap G_i$. Indeed, when $F_i \cap G_i$ happens, for all $(x,y) \in S_i$, $h^*(x) = y$
and therefore, 
\[
f_i(x) \neq 0 \implies f_i(x) = h^*(x)
\]
is true. In addition, $V_i = V[h^*, S_i]$. Since when $G_i$ happens, $\Pr_{\mu}\sbr{x \in \DIS(V[h^*, S_i])} \leq \frac \xi 2$, we also have 
$\Pr_{\mu}\sbr{x \in \DIS(V_i)} \leq \frac \xi 2$.
%\joey{$\frac{\xi}{2}$?}.
%\joey{$\frac{\xi}{2}$?}

We now lower bound $\Pr\sbr{F_i}$ and $\Pr\sbr{G_i}$ respectively. 
\begin{itemize}
\item By union bound, 
$\Pr\sbr{F_i} \geq 1 - n \eta \geq \frac{39}{40}$.
\item Meanwhile, by~\cite[][Lemma 8]{wiener2015compression} (see Lemma~\ref{lem:vs-disagree}, also~\cite[][Appendix E.1]{hanneke2024star}), 
$\Pr\sbr{G_i} \geq 1 - \frac{1}{40} = \frac{39}{40}$.

%% \chicheng{I will include a quotation of that lemma in the final appendix section, and I will call that lemma here.}

\end{itemize}
Therefore, by union bound, 
\[
\Pr\sbr{E_i} 
\geq 
1 - \Pr\sbr{F_i^C} - \Pr\sbr{G_i^C}
\geq 
\frac {19} {20}. 
\]
\end{proof}
%\joey{\frac{18}{20}?}

%%\rbr{f_i(x) \neq 0 \implies f_i(x) = h^*(x)} \wedge
%\Pr\sbr{f_i(x) = 0} \leq \frac \xi 2}.

We now continue the proof of Lemma~\ref{lem:robust-rpu}. Define 
\[
\Ical := 
\cbr{ i \in [N]:
f_i \text{ is } \frac{\xi}{2}\text{-RPU with respect to } (h^*, \mu)
}
\]
Using Claim~\ref{claim:e-i} and applying Chernoff bound, we have that with probability $1-\delta$, 
\begin{equation}
|\Ical|
\geq
\frac {9N} {10}. 
\label{eqn:most-batches}
\end{equation}
We henceforth condition on Eq.~\eqref{eqn:most-batches} happening. 

We now prove that $\hat{f}$ is reliable. By the definition of  $\hat{f}(x)$, when $\hat{f}(x) \neq 0$, 
\[
\sum_{i=1}^N I(f_i(x) \neq 0) \geq \frac{N}{5} + 1.
\]
Out of those $i$'s such that $f_i(x) \neq 0$, by Eq.~\eqref{eqn:most-batches}, at most $\frac{N}{10}$ of them disagree with $h^*(x)$. Therefore, at least half of the $f_i$'s that predict $\pm 1$ agree with $h^*$. Formally, 
$\hat{f}(x) = 
\sign( \sum_{i=1}^N f_i(x) ) = h^*(x)$.

To prove that $\hat{f}$ is $\xi$-probably useful, we have:
\begin{align*}
\Pr_{\mu} \sbr{ \hat{f}(x) = 0 } 
& = \Pr_{\mu} \sbr{ \sum_{i=1}^N I( f_i(x) \neq 0 ) \leq \frac{N}{5} } \\
& = \Pr_{\mu} \sbr{ \sum_{i=1}^N I( f_i(x) = 0 ) \geq \frac{4N}{5} } \\
& \leq \Pr_{\mu} \sbr{ \sum_{i \in \Ical} I( f_i(x) = 0 ) \geq \frac{7N}{10} } \\
& \leq \frac{10}{7N} \rbr{ \sum_{i \in \Ical} \E_{\mu}\sbr{ I( f_i(x) = 0 ) } } \\
& \leq \frac{10}{7N} \cdot N \frac \xi 2 \leq \xi,
\end{align*}
where the first inequality is from that $|\Ical| \geq \frac{9N}{10}$; the second inequality is Markov's inequality; the third inequality is from that for all $i \in \Ical$, $f_i$'s are $\frac{\xi}{2}$-probably useful. 

% \joey{It seems to be wrong, the correct one is the following.
% }

Finally, for item (2), the total number of samples drawn from $\mu$ is $N \cdot n = O( \frac{\mathfrak{s}}{\xi} \ln\frac{1}{\delta} )$.
\end{proof}

\begin{proof}[Proof of Lemma~\ref{lem:passive-rpu-mdl}]
For each iteration $r$, 
there exists an event $E_r$, in which 
$f^r$ returned by 
$\textsc{Robust-RPU-Learn}(\Hcal, \widetilde{\mu}_r, \xi, \frac{\delta}{2} )$
is $\xi$-RPU with respect to $(h^*, \tilde{\mu}_r)$. 
%satisfies that: (a) $f^r(x) \neq 0 \implies f^r(x) = h^*(x)$; (b) 
%\begin{equation}
%\Pr_{\widetilde{\mu}_r}( f^r(x) = 0 ) \leq \xi
%.
%\label{eqn:f-r-coverage}
%\end{equation}
Denote by $E = \cap_{r=1}^R E_r$. By union bound, $\Pr(E) \geq 1-\delta$. 

We henceforth condition on $E$ happening. We first show that $N_{R+1} = \emptyset$. At each iteration $r$, $f^r$ is $\xi$-RPU with respect to $(h^*, \tilde{\mu}_r)$, which implies that 
\[
\E_{i \sim \Unif(N_r)} \sbr{ \Pr_{\mu_i} \sbr{ f^r(x) = 0 } } \leq \frac \xi 2
\]
Markov's inequality yields that 
\[
\Pr_{i \sim \Unif(N_r)}\sbr{ \Pr_{\mu_i} \sbr{ f^r(x) = 0} \geq \xi } \leq \frac12,
\]
which is equivalent to $|G_r| \geq \frac{|N_r|}2$. This implies that 
$|N_{r+1}| \geq |N_r| - |G_r| \leq \frac{|N_r|}{2}$, and thus after at most $\lceil \log k \rceil$ iterations $N_r$ will become empty.

%, $|N_{R+1}| < 1$, implying that $N_{R+1} = \emptyset$.

By the reliability of $f^r$ for all $r$, we have that whenever $\hat{f}(x) \neq 0$, $\hat{f}(x) = f^c(x) = h^*(x)$, proving the reliability of $\hat{f}$.

To show that $\hat{f}$ is $\xi$-probably useful with respect to $(\mu_i)_{i=1}^k$, we first observe that $\cbr{f(x) = 0} \subseteq \cbr{ f^r(x) = 0 }$. Now, for every $i \in [k]$, denote by $r_i$ the index $r$ such that $i \in G_r$. Therefore, 
$\Pr_{\mu_i}\sbr{ f(x) = 0 }
\leq
\Pr_{\mu_i} \sbr{ f^{r_i}(x) = 0 }
\leq \xi$. 

For item (2), it follows from Lemma~\ref{lem:robust-rpu} that the total number of examples sampled from any of the $\widetilde{\mu}_r$ at iteration $r$ is $\widetilde{O}(\frac{\mathfrak{s}}{\xi})$, and thus the total number of samples across $\lceil \log k \rceil$ iterations is $\widetilde{O}(\frac{\mathfrak{s}}{\xi})$. 
\end{proof}

% Therefore, fo
% \[
% \Pr_\rho( f(x) = 0 ) 
% \leq 
% \Pr_\rho( )
% \]

%$f^r$, the RPU classifier followed by is agrees with 

\subsection{Proof of Theorem~\ref{thm:main-df}}

%\begin{proof}[Proof of Theorem~\ref{thm:main-df}]
Denote by event $E_n$ that the success event in Lemma~\ref{lem:passive-rpu-mdl} holds at iteration $n$; that lemma implies that $\Pr(E_n) \geq 1-\delta_n$. 
Define $E := \cap_{n=1}^{n_0} E_n$.
By a union bound, $\Pr(E) \geq 1-\delta$. We henceforth condition on event $E$ holding.

We will next prove by induction on $n$ that for all $n \in [n_0-1]$,  
$f_n$ is $2^{-n}$-RPU with respect to $h^*$ and $(D_i)_{i=1}^k$. 

%(1) whenever $f_n(x) \in \cbr{-1,+1}$, 
%$f_n(x) = h^*(x)$; (2) $\max_{i \in [k]} \Pr_{D_i}( f_n(x) = 0 ) \leq 2^{-n}$. 

\paragraph{Base case.} For $n=0$, we have $f_0 \equiv 0$, which is 1-RPU with respect to $h^*$ and $(D_i)_{i=1}^k$ trivially.

\paragraph{Inductive case.} Suppose the inductive claim holds for iteration $f_{n-1}$.
Then, each distribution $D_{i,n}$ can be equivalently expressed as:
\[
D_{i,n}(x,y) = D_i(x,y) I( f_{n-1}(x) = 0 )
+ 
I(y = h^*(x)) I(f_{n-1}(x) \neq 0).
\]
Therefore, for every $i \in [k]$, 
$L(h^*, D_{i,n}) \leq L(h^*, D_i) \leq \nu $.
%:= \eta_n
We also have that for $n \leq n_0$,  $\varepsilon_n \geq 2^{-n_0-1} \geq \frac{\s \eps}{d+k} \geq 100 \s \nu$; combining this with the fact that $E_n$ happens, 
we have that $f_n$ is $2^{-n}$-RPU with respect to $h^*$ and  $(D_{i,n})_{i=1}^k$. 

%(1)
% whenever $f_n(x) \in \cbr{-1,+1}$, $f_n(x)= h^*(x)$, as well as (2)
% \[
% \max_{i \in [k]} \Pr_{ D_{i,n} }( \hat{f}_n(x) = 0 )
% \leq 
% \varepsilon_n,
% \]
%which is equivalent 
%\joey{Why equivalent?}
%\chicheng{added a note}
%to $\max_{i \in [k]} \Pr_{ D_i }( \hat{f}_n(x) = 0 )
%\leq 
%\varepsilon_n$, 
Since $D_i$ and $D_{i,n}$ have the same marginal distribution over $\Xcal$, 
$f_n$ is also $2^{-n}$ RPU with respect to $h^*$ and $(D_i)_{i=1}^k$; this completes the induction.

We next analyze the algorithm in iteration $n_0$. Recall that when $E$ happens, we have 
a classifier $f_{n_0 - 1}$ that is $2^{-n_0-1} \leq \frac{\s \eps}{d+k}$-RPU with respect to $h^*$ and $(D_i)_{i=1}^k$. 

%such that whenever $f_{n_0-1}(x) \in \cbr{-1,+1}$, $f_{n_0-1}(x) = h^*(x)$, and
%\[
% \max_{i \in [k]} \Pr_{D_i}( f_n(x) = 0 ) \leq 2^{-n_0 - 1} \leq \frac{\s}{d+k} \varepsilon.
% \]

Define event $F$ as the success event in Lemma~\ref{lem:passive-mdl} when calling \PMDL. Since for every $i \in [k]$, 
$L(h^*, D_{i,n_0}) \leq L(h^*, D_i) \leq \nu$, we have that $\hat{h}$, the output of \PMDL, satisfies that 
\begin{equation}
\max_{i \in [k]} L( \hat{h}, D_{i,n} )
\leq 
\varepsilon,
\label{eqn:max-err-biased}
\end{equation}
and therefore, applying Lemma~\ref{lem:favbias} gives us that for every $i \in [k]$, 
\[
L( \hat{h}, D_{i} ) 
\leq 
L( h^*, D_{i} ) + L( \hat{h}, D_{i,n} )
\leq 
\nu + \varepsilon.
\]
This concludes the proof of the excess error  guarantee.

We now turn to analyzing the label complexity of Algorithm~\ref{alg:main-large-eps-df}.
\begin{itemize}
\item For the iterations $n \in [n_0-1]$, we have by Lemma~\ref{lem:passive-rpu-mdl} that the number of examples sampled from any of the $D_{i,n}$ by \PRMDL is at most $\frac{\s}{\varepsilon_n}$, and by a Chernoff bound, the number of label queries made at iteration $n$ is at most 
\[
N_n \leq O\bra{\frac{\s}{\varepsilon_n} \cdot \varepsilon_n}= O(\s).
\]

\item For the last iteration $n_0$, Lemma~\ref{lem:passive-mdl} guarantees that 
the number of samples to any of the $D_{i,n_0}$ made by \PMDL 
is at most $O(\frac{d+k}{\varepsilon})$, 
which, combined with the fact that
\[
\max_{i \in [k]} \Pr_{D_i}\sbr{f_n(x)=0} \leq \frac{\s}{d+k} \varepsilon
\]
and 
Chernoff bound, makes 
\[
N_{n_0} = O\rbr{ \s + \ln\frac1\delta }
\]
label queries.  
\end{itemize}
In summary, the total number of label queries made by Algorithm~\ref{alg:main-large-eps-df} is at most $\sum_{n=1}^{n_0} N_n = O( \s \ln\frac1\varepsilon )$.
%\qed
%\end{proof}

\section{Additional Theorems and Lemmas}

\begin{theorem}[Chernoff Bound]\label{lemma:chernoff}
Suppose $X_1, \dots, X_n$ is a collection of i.i.d. Bernoulli random variables with mean $p$. Then 
\[
\Pr\sbr{ \sum_{i=1}^n X_i \geq 2np + 2\ln\frac1\delta } \leq \delta.
\]
\end{theorem}

\begin{theorem}[Bernstein's Inequality]\label{thm:bernstein}
For independent random variables $X_1,\ldots,X_n$ with $|X_i - \mathbb{E}[X_i]| \leq M$ almost surely, and letting $S_n = \sum_{i=1}^n X_i$ and $\sigma^2 = \sum_{i=1}^n \text{Var}(X_i)$, then for any $t > 0$:

$$\mathbb{P}(|S_n - \mathbb{E}[S_n]| \geq t) \leq 2\exp\left(-\frac{t^2}{2\sigma^2 + \frac{2Mt}{3}}\right).$$
\end{theorem}

% \begin{theorem}[Freedman's Inequality]\label{theorem:freedman}
% Suppose $X_1, \dots, X_T$ is a martingale difference sequence, and $b$ is a uniform upper bound on the steps $X_i$. Let $V$ denote the sum of conditional variances,
% \[
% V = \sum_{i=1}^T \mathrm{Var} \left( X_i \mid X_1, \dots, X_{i-1} \right).
% \]
% Then, for every $a, v > 0$,
% \[
% \Pr \sbr{\sum_{i=1}^T X_i \ge a \text{ and } V \le v } \le \exp \left( \frac{-a^2}{2v + 2ab/3} \right).
% \]
% Alternatively, for all $v\ge V$,
% \[
% \Pr\sbr{\sum_{i=1}^TX_i\le\sqrt{2v\log\bra{1/\delta}}+b\log\bra{1/\delta}}\le\delta.
% \]
% \end{theorem}

\begin{theorem}[Freedman's Inequality,~\cite{bartlett2008high}'s version]
\label{thm:freedman-easytouse}
Suppose \( X_1, \ldots, X_T \) is a martingale difference sequence with \( |X_t| \leq b \). Let
\[
V = \sum_{t=1}^T \operatorname{Var}(X_t \mid X_1, \ldots, X_{t-1}) .
\]
be the sum of the conditional variances of \( X_t \)'s. 
Then we have, for any \( \delta < 1/e \) and \( T \geq 4 \),
\[
\Pr\sbr{ \sum_{t=1}^T X_t > 4 \sqrt{V \ln(1/\delta)} 
+ 2b \ln(1/\delta) } \leq \log(T)\delta .
\]
\end{theorem}

\begin{lemma}\label{lemma:quadraticinequality}
Let $A,B,C\ge0$, if $\abs{A-C}\le B\sqrt{A+C}$, then $\abs{A-C}\le2\bra{B^2+B\min\cbr{\sqrt{C},\sqrt{A}}}$.
\end{lemma}
\begin{proof}
Taking square on both sides and manipulating to get the following
\[
A^2-\bra{2C+B^2}A+C^2-B^2C\le0.
\]
By the formula of quadratic equations, we have
\begin{align*}
A&\le\frac{(2C+B^2)+\sqrt{(2C+B^2)^2-4(C^2-B^2C)}}{2}\\
&\le \frac{(2C+B^2)+\sqrt{B^4+8B^2C}}{2}\\
&\le \frac{(2C+B^2)+B^2+2\sqrt{2}B\sqrt{C}}{2}\\
&\le C+B^2+\sqrt{2}B\sqrt{C},
\end{align*}
where the third inequality comes from $\sqrt{a+b}\le\sqrt{a}+\sqrt{b}$ for $a,b\ge0$. 
% \chicheng{May I suggest a simplification here. 
% The above implies that 
% \[
% A - C \leq B^2 + \sqrt{2} B \sqrt{C}
% \]
% }

As a result,
\begin{align*}
    \bra{A-C}^2&\le B^2\bra{A+C}\\
    &\le B^2\bra{2C+B^2+B\sqrt{2C}}\\
    &\le 2B^2\bra{2C+B^2},
\end{align*}
where the last step comes from $ab\le a^2+b^2$. Taking square root on both sides and applying $\sqrt{a+b}\le\sqrt{a}+\sqrt{b}$ again and we get
\[
\abs{A-C}\le 2\bra{B\sqrt{C}+B^2}.
\]
Due to symmetry of $A$ and $C$ in the condition, we also have
\[
\abs{A-C}\le 2\bra{B\sqrt{A}+B^2}.
\]
As a result,
\[
\abs{A-C}\le 2\bra{B^2+B\min\cbr{\sqrt{A},\sqrt{C}}}.
\]
\end{proof}

% \chicheng{Can I prove it using the Lemma above to make the presentation simpler?
% \begin{lemma}\label{lemma:quadraticinequalityvariant}
% Let $A,B,C,D\ge0$, if $A-C\le B\sqrt{A+C}+D$, then $A-C\le \frac 3 2 D+ 2 B^2 + \frac 5 2 B\sqrt{C}$.
% \end{lemma}
% \begin{proof}
% The conditions of the lemma, along with the well-known algebraic fact that $\sqrt{a+b} \leq \sqrt{a} + \sqrt{b}$, implies that 
% \[
% A \le (C+B\sqrt{C}+D) + B\sqrt{A},
% \]
% which, combined with the well-known algebraic fact that $a \leq b + c\sqrt{a} \implies a \leq b + c^2 + c\sqrt{b}$~\citep{dasgupta2007general}, gives that 
% \begin{align*}
% A \leq & (C+B\sqrt{C}+D) + B^2+ B \sqrt{ C+B\sqrt{C}+D } \\
% \leq & C+B\sqrt{C}+D + B^2 + B \sqrt{C} + B \sqrt{ B\sqrt{C} } + B\sqrt{D} \tag{ $\sqrt{a+b} \leq \sqrt{a} + \sqrt{b}$ }\\
% \leq & C+ 2B\sqrt{C}+D + B^2 + \frac12 B^2 + \frac12 C + \frac12 B^2 + \frac12 D \tag{AM-GM on the last two terms } \\
% = & \frac 3 2 D+ 2 B^2 + \frac 5 2 B\sqrt{C}. \tag{algebra}
% \end{align*}
% \end{proof}
% }

\begin{lemma}\label{lemma:quadraticinequalityvariant}
Let $A,B,C,D\ge0$, if $A-C\le B\sqrt{A+C}+D$, then $A-C\le \frac{3}{2}D+\frac{3}{2}B^2+\sqrt{2}B\sqrt{C}$.
\end{lemma}
\begin{proof}
Moving $D$ to the LHS, taking square on both sides and manipulating to get the following
\[
A^2-2CA-2DA-B^2A+2DC+C^2+D^2-B^2C\le0.
\]
By the formula of quadratic equations, we have
\begin{align*}
A&\le\frac{2C+2D+B^2+\sqrt{\rbr{-2C-2D-B^2}^2-4\rbr{C^2+2DC+D^2-CB^2}}}{2}.
\end{align*}
Let $S=\rbr{-2C-2D-B^2}^2-4\rbr{C^2+2DC+D^2-CB^2}$, then let's upper bound $\sqrt{S}$. By simply expanding $S$, we have
\begin{align*}
    S&=4C^2+4D^2+B^4+8CD+4CB^2+4DB^2-4C^2-8CD-4D^2+4CB^2\\
    &=B^4+8CB^2+4DB^2\\
    &=B^2\rbr{B^2+8C+4D}.
\end{align*}
Therefore,
\begin{align*}
    \sqrt{S}\le B\sqrt{B^2+8C+4D}\le B\rbr{B+2\sqrt{2}\sqrt{C}+2\sqrt{D}}\le B^2+2\sqrt{2}B\sqrt{C}+B^2+D,
\end{align*}
where the second inequality comes from $\sqrt{a+b+c}\le \sqrt{a}+\sqrt{b}+\sqrt{c}$ for $a,b,c\ge0$ and the last inequality is AM-GM. Putting everything together, we have
\begin{align*}
    A\le C+\frac{3}{2}D+\frac{3}{2}B^2+\sqrt{2}B\sqrt{C}.
\end{align*}
Moving $C$ to the LHS and we proved the lemma.
\end{proof}

%~\cite{dasgupta2007general}
\begin{lemma}[Favorable bias, \cite{hsu2010algorithms}, Lemma 5.2]
\label{lem:favbias}
Suppose we have a distribution $D$ over $\Xcal \times \Ycal$, and a classifier $h^*$ and a region $R \subset \Xcal$. Define the ``favorably biased'' distribution $\widetilde{D}$: 
\[
\widetilde{D}(x,y) = D_i(x,y) I( x \in R )
+ 
I(y = h^*(x)) I( x \notin R ).
\]
Then for any classifier $h$, 
\[
L(h, D) - L(h^*, D) \leq L(h, \widetilde{D}) - L(h^*, \widetilde{D}).
\]
\end{lemma}

\begin{lemma}[\cite{wiener2015compression}, Lemma 8]
\label{lem:vs-disagree}
Suppose we have $S$, $n$ iid samples from distribution $D$ over $\Xcal \times \Ycal$ realizable by hypothesis $h^*$ in hypothesis class $\Hcal$. Denote by $V = \cbr{h: h \text{ is consistent with } S}$, then with probability $1-\delta$,
\[
\Pr_{D}\sbr{ x \in \DIS(V) }
\leq 
\frac{10 \s \ln\frac{en}{\s} + 4 \ln\frac{2}{\delta}}{n},
\]
where $\s$ is the star number of $\Hcal$.
\end{lemma}

\begin{lemma}\label{lemma:KLintegraltrick}
For any $p,q\in(0,1)$, the KL divergence between $\mathrm{Bern}(p)$ and $\mathrm{Bern}(q)$ admits the integral representation
\[
\DKL\bigl(\mathrm{Bern}(p)\|\mathrm{Bern}(q)\bigr)
=
\int_{p}^{q}\frac{x-p}{x(1-x)}\,dx.
\]
\end{lemma}

\begin{proof}
Define
\[
F(q)=\DKL\bigl(\mathrm{Bern}(p)\|\mathrm{Bern}(q)\bigr)
=p\ln\frac{p}{q}+(1-p)\ln\frac{1-p}{1-q}.
\]
Then
\[
F'(q)=-\frac{p}{q}+\frac{1-p}{1-q}
=\frac{q-p}{q(1-q)}.
\]
Since $F(p)=0$, integrating from $p$ to $q$ yields
\[
\DKL\bigl(\mathrm{Bern}(p)\|\mathrm{Bern}(q)\bigr)
=F(q)-F(p)
=\int_{p}^{q}F'(x)\,dx
=\int_{p}^{q}\frac{x-p}{x(1-x)}\,dx.
\]
\end{proof}

\end{document}